\documentclass[twoside]{article}

\usepackage[round]{natbib}

\newif\ifICML
\ICMLtrue

\usepackage{graphicx}
\usepackage{amsmath}
\usepackage{amssymb}
\usepackage{amsfonts}
\usepackage{amsthm}
\usepackage{algorithmic}
\usepackage{algorithm}
\usepackage{cases}
\usepackage{array}
\usepackage{tabularx}
\usepackage{enumitem}

\usepackage{caption}
\usepackage{subcaption}

\usepackage[bookmarks=false,colorlinks=true,linkcolor=blue,urlcolor=magenta,citecolor=red,pdfstartview=FitH,pagebackref]{hyperref}

\newtheorem{theorem}{Theorem}
\newtheorem{lemma}[theorem]{Lemma}
\newtheorem{definition}{Definition}

\newtheorem{corollary}[theorem]{Corollary}
\newtheorem{proposition}[theorem]{Proposition}
\theoremstyle{remark}

\newcounter{assumption}
\renewcommand{\theassumption}{A\arabic{assumption}}

\newenvironment{assumption}[1][]{\begin{trivlist}\item[] \refstepcounter{assumption}%
 \textbf{Assumption\ \theassumption\ #1} }{
 \ifvmode\smallskip\fi\end{trivlist}}

\renewcommand{\AA}{{\mathcal{A}}}

\newcommand{\XX}{{\mathcal{X}}}

\newcommand{\XA}{\XX\times\AA}

\newcommand{\beq}{\begin{equation}}
\newcommand{\eeq}{\end{equation}}
\newcommand{\beqa}{\begin{eqnarray}}
\newcommand{\eeqa}{\end{eqnarray}}
\newcommand{\beqan}{\begin{eqnarray*}}
\newcommand{\eeqan}{\end{eqnarray*}}
\newcommand{\ben}{\begin{eqnarray*}}
\newcommand{\een}{\end{eqnarray*}}

\newcommand{\norm}[1]{\left\Vert#1\right\Vert}
\newcommand{\smallnorm}[1]{\Vert#1\Vert}

\newcommand{\Real}{\mathbb R}

\newcommand{\EE}[1]{{\mathbb E}\left[#1\right]}

\newcommand{\eps}{\varepsilon}

\newcommand{\ra}{\rightarrow}

\newcommand{\FF}{{\mathcal{F}}}
\newcommand{\GG}{{\mathcal{G}}}

\newcommand{\argmin}{\mathop{\textrm{argmin}}}
\newcommand{\argmax}{\mathop{\textrm{argmax}}}

\newcommand{\Qmax}{Q_{\textrm{max}}} 
\newcommand{\Rmax}{R_{\textrm{max}}}

\newcommand{\eqdef}{\triangleq}

\newcommand{\ip}[2]{\left \langle\,#1\,,\,#2\, \right \rangle}

\newcommand{\actionnum}{{|\AA|}}

\newcommand{\MM}{\mathcal{M}}

\newcommand{\gradtheta}{\nabla_\theta}
\newcommand{\KL}{\textsf{KL} }
\newcommand{\PKernelhat}{\hat{\PKernel} }
\newcommand{\PKernelTrue}{\PKernel^*}

\newcommand{\AlgPlan}{{{\textsf{Planner}}}}

\newcommand{\Qpi}{Q^\pi}

\newcommand{\Qhat}{\hat{Q}}

\newcommand{\Dn}{\mathcal{D}_n}

\newcommand{\Id}{\mathbf{I}}

\newcommand{\EEX}[2]{{\mathbb E}_{#1}\left[#2\right]}

\newcommand{\Qopt}{Q^*}
\newcommand{\Vpi}{V^\pi}

\newcommand{\cset}[2]{\left\{\,#1\,:\,#2\,\right\}}

\newcommand{\Actions}{\mathcal{A}}

\newcommand{\PKernel}{\mathcal{P}}
\newcommand{\RKernel}{\mathcal{R}}

\newcommand{\One}[1]{{\mathbb I}{\{#1\}}}

\newcommand{\ToptP}[1]{{T^*_{#1}}}
\newcommand{\PKernelpiTrue}{{\PKernel^*}^{\pi}}

\newcommand{\Qopthat}{\hat{Q}^*}
\newcommand{\Qhatpi}{\hat{Q}^\pi}
\newcommand{\dx}{\mathrm{d}x}
\newcommand{\dy}{\mathrm{d}y}
\newcommand{\dmu}{\mathrm{d}\mu}
\newcommand{\dnu}{\mathrm{d}\nu}
\newcommand{\drho}{\mathrm{d}\rho}
\newcommand{\da}{\mathrm{d}a}
\newcommand{\Cov}[1]{{\bf Cov}\left(#1\right)}
\newcommand{\CovX}[2]{{\bf Cov}_{#1}\left(#2\right)}
\newcommand{\VarX}[2]{{\mathrm{Var}_{#1}}\left[#2\right]}
\DeclareMathOperator{\Tr}{Tr}
\newcommand{\PKernelpihat}{\hat{\PKernel}^\pi }
\newcommand{\PKernelpi}{\PKernel^\pi}

\newcommand{\DeltaPKernelpi}{\Delta \PKernel^\pi}

\newcommand{\cPG}{c_{\text{PG}}}

\newcommand{\cPGrhopinu}{c_{\text{PG}}(\rho,\nu;\pi)}

\newcommand{\pDiscounted}[1]{#1_\gamma}

\newcommand{\rhoDiscounted}{\pDiscounted{\rho}}
\newcommand{\rhohatDiscounted}{\pDiscounted{\hat{\rho}}}

\newcommand{\muDiscounted}{\pDiscounted{\mu}}
\newcommand{\muhatDiscounted}{\pDiscounted{\hat{\mu}}}

\newcommand{\ProjTheta}{\textsf{Proj}_\Theta}

\newcommand{\piBestrho}{\bar{\pi}_\rho}
\newcommand{\piBest}{ {\bar{\pi}} }

\newcommand{\LossPolicyError}{ L_\text{PAE} }

\newif\ifconsiderlater
\considerlaterfalse

\ifconsiderlater
	\newcommand{\todo}[1]{{\color{cyan} \textbf{XXX [#1] XXX}}}
\else
\newcommand{\todo}[1]{}
\fi

\newcommand{\todoNew}[1]{}

\newif\ifSupp
\Supptrue

\ifSupp \usepackage{geometry} \else \usepackage{aistats2021} \fi

\begin{document}
\ifSupp
\title{Policy-Aware Model Learning for Policy Gradient Methods}
\author{Romina Abachi\footnote{Department of Computer Science, University of Toronto \& the Vector Institute, Toronto, Canada}
\and
Mohammad Ghavamzadeh\footnote{Google Research, CA, USA}
\and
Amir-massoud Farahmand\footnote{The Vector Institute \& the Department of Computer Science and the Department of Mechanical and Industrial Engineering, University of Toronto, Toronto, Canada} }
\maketitle
\else
\twocolumn[
\aistatstitle{Policy-Aware Model Learning for Policy Gradient Methods}

\aistatsauthor{ Romina Abachi \And Mohammad Ghavamzadeh \And Amir-massoud Farahmand}

\aistatsaddress{ University of Toronto \& the Vector Institute, Toronto, Canada \And  Institution 2 \And Institution 3 } ]
\fi
\begin{abstract}
This paper considers the problem of learning a model in model-based reinforcement learning (MBRL).
We examine how the planning module of an MBRL algorithm uses the model, and propose that model learning should incorporate the way the planner is going to use the model.
This is in contrast to conventional model learning approaches, such as those based on maximum likelihood estimation, that learn a predictive model of the environment without explicitly considering the interaction of the model and the planner.
We focus on policy gradient planning algorithms and derive new loss functions for model learning that incorporate how the planner uses the model.
We call this approach Policy-Aware Model Learning (PAML).
We theoretically analyze a model-based policy gradient algorithm and provide a convergence guarantee for the optimized policy.
We also empirically evaluate PAML on some benchmark problems, showing promising results.
\todo{Revise!}
\end{abstract}
\section{Introduction}
\label{sec:PAML-Introduction}

A model-based reinforcement learning (MBRL) agent gradually learns a model of the environment as it interacts with it, and uses the learned model to plan and find a good policy.
This can be done by planning with samples coming from the model, instead of or in addition to the samples from the environment, e.g.,~\citet{Sutton1990,PengWilliams1993,SuttonSzepesvariGeramifardBowling08,DeisenrothFoxRasmussan2015,Talvite2017,HaSchmidhuber2018NIPS}.
%
%
If learning a model is easier than learning the policy or value function in a model-free manner, MBRL will lead to a reduction in the number of required interactions with the real-world and will improve the sample complexity of the agent. However, this is contingent on the ability of the agent to learn an accurate model of the real environment.
 \todo{Mention the work of Tu and Recht et al. on comparison of MBRL and model-free one, at least in LQR setting.}
Thus, the problem of learning a good model of the environment is of paramount importance in the success of MBRL. This paper addresses the question of how to approach the problem of learning a model of the environment, and proposes a method called \emph{policy-aware model learning} (PAML).

The conventional approach to model learning in MBRL is to learn a model that is a good predictor of the environment. If the learned model is accurate enough, this leads to a value function or a policy that is close to the optimal one.
Learning a good predictive model can be achieved by minimizing some form of a probabilistic loss. A common choice is to minimize the $\KL$-divergence between the empirical data and the model, which leads to the Maximum Likelihood Estimator (MLE).

The often-unnoticed fact, however, is that no model can be completely accurate, and there are always differences between the model and the real-world. An important source of inaccuracy/error is the choice of model space, i.e.,~the space of predictors, such as a particular class of deep neural networks. We suffer an error if the model space does not contain the true model of the physical system.

The \emph{decision-aware model learning} (DAML) viewpoint suggests that instead of trying to learn a model that is a good predictor of the environment, which may not be possible as just argued, one should learn only those aspects of the environment that are relevant to the decision problem. Trying to learn the complex dynamics that are irrelevant to the underlying decision problem is pointless, e.g.,~in a self-driving car, the agent does not need to model the movement of the leaves on trees when the decision problem is simply to decide whether or not to stop at a red light. The conventional model learning approach cannot distinguish between decision-relevant and irrelevant aspects of the environment, and may waste the ``capacity'' of the model on unnecessary details. In order to focus the model on the decision-relevant aspects, we shall incorporate certain aspects of the decision problem into the model learning process. 

There are several relatively recent works that can be interpreted as doing DAML, even though they do not always explicitly express their goal as such.
Some methods such as \citet{JosephGeramifardRobertsHowRoy2013, SilvervanHasseltetal2017, OhSinghLee2017,FarquharRocktaeschelIglWhiteson2018} learn a model implicitly, in an end-to-end fashion. This can be interpreted as DAML because the model is learned in service of improving policy performance. 
Other methods incorporate the value function in model-learning. 
For example, Value-Aware Model Learning (VAML) is an instantiation of DAML that incorporates the information about the value function in learning the model of the environment~\citep{FarahmandVAML2017,Farahmand2018}. Recent similar works include \citet{ayoub2020model,luo2018algorithmic}. In the latter, the loss is defined to only include the value function learned on the model, whereas ~\citet{FarahmandVAML2017, Farahmand2018, ayoub2020model} require the inclusion of the true value function in their loss as well. \todoNew{Any comment how they are relevant or different? -AMF}
%
\ifSupp
The formulation by~\citet{FarahmandVAML2017} incorporates the knowledge about the \emph{value function space} in learning the model, and the formulation by~\citet{Farahmand2018} benefits from how the value functions are generated within an approximate value iteration (AVI)-based MBRL agent. We explain the VAML framework in more details in Section~\ref{sec:PAML-Background-VAML}.
\fi
%

\todo{RA: expand on some of the related works above? some of them DO use the policy in model-learning} Designing a decision-aware model learning approach, however, is not limited to methods that benefit from the structure of the value function. Policy is another main component of RL that can be exploited for learning a model. 
Using the policy in model learning is done concurrently by \citet{DOroMetelliTirinzoniPapiniRestelli2020} and \citet{schrittwieser2019mastering}. 
We briefly compare to \citet{DOroMetelliTirinzoniPapiniRestelli2020} in Section \ref{sec:PAML-Algorithm}. 
MuZero \citep{schrittwieser2019mastering} takes into account both the value function and policy. 
The model in MuZero includes separate functions for predicting next ``internal'' states, policies and values. 
The policy prediction function learns to predict policies that would be obtained by the MCTS planner that is used to train it.  
On the other hand, in this work we consider policy gradient planners and instead of predicting policies directly, consider the interaction of the policy and value function in obtaining policy gradients.
\todoNew{Do we compare in any more detail with muZero one?} 
The high-level idea is simple: If we are using a policy gradient (PG) method to search for a good policy, we only need to learn a model that provides accurate estimates of the PG. All other details of the environment that do not contribute to estimating PG are irrelevant. Formalizing this intuition is the main algorithmic contribution of this paper (Section~\ref{sec:PAML-Algorithm}).
\todo{Maybe a bit more discussion? Especially if we talk about the compatible features and the Fisher Information Matrix.}

\ifSupp
\todo{Maybe add some vignettes of the theoretical result and the algorithm. Something like the loss is $\norm{\nabla_\theta J(\pi_\theta) - \hat{J}(\pi_\theta)}$ and alike.}
The first theoretical contribution of this work is a result that shows how the error in the model, in terms of total variation, affects the quality of the PG estimate\ifSupp~(Theorem~\ref{thm:PAML-PolicyGradientError} in Section~\ref{sec:PAML-Theory-PG-Error})\fi. 
This is reassuring as it shows that a good model leads to an accurate estimate of the PG. 
\ifSupp \else We report this result in the supplementary material. \fi It might seem natural to assume that an accurate gradient estimate leads to convergence to a good policy. 
However, we could not quantify the quality of the converged policy in a PG procedure beyond stating that it is a local optimum until recently, when~\citet{AgarwalKakadeLeeMahajar2019} provided quantitative guarantees on the convergence of PG methods beyond convergence to a local optimum. Our second theoretical contribution is Theorem~\ifSupp\ref{thm:PAML-MBPG-Convergence} \else\ref{thm:PAML-MBPG-Convergence-Compact} \fi (Section~\ref{sec:PAML-Theory}) that extends one of the results in~\citet{AgarwalKakadeLeeMahajar2019} to model-based PG and shows the effect of PG error on the quality of the converged policy, as compared to the best policy in the policy class. We also have a subtle, but perhaps important, technical contribution in the definition of the policy approximation error, which holds even in the original model-free setting.
\else
Our theoretical contribution is a global convergence guarantee for model-based policy gradient (MBPG).
The result shows the effect of policy approximation error, the model error, the number of optimization steps, and a few properties of the MDP and sampling distribution.
This is an extension of the recent work by~\citet{AgarwalKakadeLeeMahajar2019} to model-based PG.
Moreover, our result introduces a new definition of the policy approximation error, which is perhaps a more accurate way to characterize this error.
In the supplementary materials, we extend our results further and consider the error introduced by an imperfect critic. 
\todoNew{Do we mention anywhere that we also have results that allows critic error? -AMF, RA: added}
\fi 
Our empirical contributions are demonstrating that PAML can easily be formulated for \ifSupp two commonly-used PG algorithms \else a commonly-used PG algorithm \fi and showing its performance in benchmark environments (Section~\ref{sec:PAML-Empirical}), for which the code is made available at \url{https://github.com/rabachi/paml}. 
In addition, our results in a finite-state environment show that PAML outperforms conventional methods when the model capacity is limited. 


\section{Background on Decision-Aware Model Learning (DAML)}
\label{sec:PAML-Background-VAML}

A MBRL agent interacts with an environment, collects data, improves its internal model, and uses the internal model, perhaps alongside the real-data, to improve its policy.
To formalize, we consider a (discounted) Markov Decision Process (MDP) $(\XX, \AA, \RKernel^*, \PKernelTrue, \gamma)$~\citep{SzepesvariBook10}. We denote the state space by $\XX$, the action space by $\AA$,  the reward distribution by $\RKernel^*$, the transition probability kernel by $\PKernelTrue$, and the discount factor by $ 0 \leq \gamma \leq 1$.
In general, the true transition model $\PKernelTrue$ and the reward distribution $\RKernel^*$ are not known to an RL agent.
The agent instead can interact with the environment to collect samples from these distributions. The collected data is in the form of 
%
\begin{align}
\label{eq:PAML-Dataset-Generic}
	\Dn  =  \{(X_i, A_i, R_i, X'_i)\}_{i=1}^n,
\end{align}
with the current state-action being distributed according to $Z_i = (X_i, A_i) \sim \nu(\XA) \in \bar{\MM}(\XA)$, the reward $R_i \sim \RKernel^*(\cdot|X_i, A_i)$, and the next-state $X'_i \sim \PKernelTrue(\cdot|X_i, A_i)$. Note that $\bar{\MM}$ refers to the set of all probability distributions defined over $\XX$ and $\AA$.
In many cases, an RL agent might follow a trajectory $X_1, X_2, \dotsc$ in the state space (and similar for actions and rewards), that is, $X_{i+1} = X'_i$.
We denote the expected reward by $r^*(x,a) = \EE{\RKernel^*(\cdot|x, a)}$.%
\ifSupp \footnote{Given a set $\Omega$ and its $\sigma$-algebra $\sigma_\Omega$, $\bar{\MM}(\Omega)$ refers to the set of all probability distributions defined over $\sigma_\Omega$.
As we do not get involved in the measure theoretic issues in this paper, we do not explicitly define the $\sigma$-algebra, and simply use a well-defined and ``standard'' one, e.g., Borel sets defined for metric spaces.}\fi

A MBRL agent uses the interaction data to learn an estimate $\PKernelhat$ of the true model $\PKernelTrue$ and $\hat{\RKernel}$ (or simply $\hat{r}$) of the true reward distribution $\RKernel^*$ (or $r^*$). This is called \emph{model learning}. These models are then used by a planning algorithm $\AlgPlan$ to find a close-to-optimal policy. The policy may be used by the agent to collect more data and improve the estimates $\PKernelhat$ and $\hat{\RKernel}$. This generic Dyna-style~\citep{Sutton1990} MBRL algorithm is shown in Algorithm~\ref{alg:PAML}.
\newlength{\textfloatsepsave}
\setlength{\textfloatsepsave}{\textfloatsep}

\setlength{\textfloatsep}{0pt}
\begin{algorithm}[tb]
\small
\caption[Generic MBRL Algorithm]{Generic MBRL Algorithm}
\begin{algorithmic}[0]
\STATE Initialize a policy $\pi_0$
\FOR{$k=0, 1, \dotsc,K$}
\STATE Generate training set $\Dn^{(k)} = \{(X_i, A_i, R_i, X'_i)\}_{i=1}^n$ by interacting with the true environment (potentially using $\pi_k$), i.e., $(X_i, A_i) \sim \nu_k$ with $X'_i \sim \PKernelTrue(\cdot|X_i, A_i)$ and $R_i \sim \RKernel^*(\cdot|X_i, A_i)$. 
\STATE $\PKernelhat^{(k+1)} \leftarrow \argmin_{\PKernel \in \MM} \text{Loss}_\PKernel(\PKernel;\cup_{i=0}^k \Dn^{(i)})$ 
\COMMENT{{\bf PAML:} $\text{Loss}_\PKernel = \smallnorm{                \gradtheta J(\mu_\theta^k) 
        - \gradtheta \hat{J}(\mu_\theta^k)
		}_{\cup_{i=0}^k \Dn^{(i)}}^2$
}
\STATE $\hat{r} \leftarrow \argmin_{r \in \GG} \text{Loss}_\RKernel(r;\cup_{i=0}^k \Dn^{(i)})$
\STATE $\pi_\theta^{k+1} \leftarrow \AlgPlan(\PKernelhat,\hat{\RKernel})$ \COMMENT{{\bf PAML:} PG-based (e.g.,~REINFORCE or DDPG); $\theta_{k+1} \leftarrow \theta_k + \eta \gradtheta \hat{J}(\pi^k_\theta)$.}
\ENDFOR, \COMMENT {Return $\pi_K$}
\label{alg:PAML}
\end{algorithmic}
\end{algorithm}
\setlength{\textfloatsep}{\textfloatsepsave}

How should we learn a model $\PKernelhat$ that is most suitable for a particular $\AlgPlan$?
This is the fundamental question in model learning. The conventional approach in model learning ignores how $\AlgPlan$ is going to use the model and instead focuses on learning a good predictor of the environment. This can be realized by using a probabilistic loss, such as $\KL$-divergence, which leads to the maximum likelihood estimate (MLE), or similar approaches. Ignoring how the planner uses the model, however, might not be a good idea, especially if the model class $\MM$, from which we select our estimate $\PKernelhat$, does not contain the true model $\PKernelTrue$, i.e., $\PKernelTrue \notin \MM$. This is the model approximation error and its consequence is that we cannot capture all aspects of the dynamics.
%
The thesis behind DAML is that instead of being oblivious to how $\AlgPlan$ uses the model, the model learner should pay more attention to those aspects of the model that affect the decision problem the most. 
A purely probabilistic loss ignores the underlying decision problem and how $\AlgPlan$ uses the learned model, whereas a DAML method incorporates the decision problem and $\AlgPlan$.

\ifSupp

Value-Aware Model Learning (VAML) is a class of DAML methods~\citep{FarahmandVAMLEWRL2016,FarahmandVAML2017,Farahmand2018}. It is a model learning approach that is designed for a value-based type of $\AlgPlan$, i.e.,~a planner that finds a good policy by approximating the optimal value function $\Qopt$ by $\Qopthat$, and then computes the greedy policy w.r.t.~$\Qopthat$. 
In particular, the suggested formulations of VAML so far focus on value-based methods that use the Bellman optimality operator to find the optimal value function (as opposed to a Monte Carlo-based solution).
The use of the Bellman [optimality] operator, or a sample-based approximation thereof, is a key component of many value-based approaches, such as the family of (Approximate) Value Iteration~\citep{Gordon1995,SzepesvariSmart2004,Ernst05,Munos08JMLR,FarahmandACC09,FarahmandVPI2012,MnihKavukcuogluSilveretal2015,TosatooPirottaDEramoRestelli2017,ChenJiang2019} or (Approximate) Policy Iteration (API) algorithms~\citep{LagoudakisParr03,AntosSzepesvariML08,BertsekasAPI2011,LazaricGhavamzadehMunosLSPI2012,ScherrerGhavamzadehGabillonGeist2012,FarahmandGhavamzadehSzepesvariMannor2016}.

To be more concrete in the description of VAML, let us first recall that the Bellman optimality operator w.r.t.~the transition kernel $\PKernel$ is defined as
\begin{align}
\label{eq:PAML-BellmanOperator}
	\ToptP{\PKernel}: Q \mapsto r + \gamma \PKernel \max_a Q.
\end{align}
%
%
VAML attempts to find $\PKernelhat$ such that applying the Bellman operator $\ToptP{\PKernelhat}$ according to the model $\PKernelhat$ on a value function $Q$ has a similar effect as applying the true Bellman operator $\ToptP{\PKernelTrue}$ on the same function, i.e., $\ToptP{\PKernelhat} Q \approx \ToptP{\PKernelTrue} Q$.
This ensures that one can replace the true dynamics with the model without (much) affecting the internal mechanism of a Bellman operator-based $\AlgPlan$.
How this might be achieved is described in the original papers, which are summarized in \ifSupp Appendix~\ref{sec:PAML-VAML} \else an appendix of the supplementary material\fi.

The VAML framework is an instantiation of DAML when $\AlgPlan$ benefits from extra knowledge available about the value function, either in the form of the value function space (as in the original VAML formulation) or particular value functions generated by AVI (as in the IterVAML formulation), to learn a model $\PKernelhat$. The value function and our knowledge about it, however, are not the only extra information that we might have about the decision problem. Another source of information is the policy. The goal of the next section is to develop a model learning framework that benefits from the properties of the policy.

\else
Value-Aware Model Learning (VAML) is a class of DAML methods~\citep{FarahmandVAMLEWRL2016,FarahmandVAML2017,Farahmand2018}. It is a model learning approach that is designed for a value-based type of $\AlgPlan$, i.e.,~a planner that finds a good policy by approximating the optimal value function $\Qopt$ by $\Qopthat$, and then computes the greedy policy w.r.t.~$\Qopthat$. 
In particular, the suggested formulations of VAML so far focus on value-based methods that use the Bellman optimality operator to find the optimal value function. More detail about VAML can be found in the supplementary material.
The value function and our knowledge about it, however, are not the only extra information that we might have about the decision problem. Another source of information is the policy. The goal of the next section is to develop a model learning framework that benefits from the properties of the policy.
\fi

\section{Policy-Aware Model Learning}
\label{sec:PAML-Algorithm}

The policy gradient (PG) algorithm and its several variants are important tools to solve RL problems~\citep{Williams1992,SuttonMcAleesterSinghMansour2000,BaxterBartlett01,MarbachTsitsiklis2001,Kakade01,PetersVijayakumarSchaal03,Cao2005,GhavamzadehEngel07PG,PetersSchaalNAC2008,BhatnagarSuttonGhavamzadehLee2009,DeisenrothNeumannPeters2013,SchulmanLevineAbbeelJordanMoritz2015}.
These algorithms parameterize the policy and compute the gradient of the performance (cf.~\eqref{eq:PAML-PerformanceMeasure}) w.r.t.~the parameters. Model-free PG algorithms use the environment to estimate the gradient, but model-based ones use an estimated $\PKernelhat$ to generate ``virtual'' samples to estimate the gradient. 
In this section, we derive a loss function\todo{XXX Unify objective vs. loss XXX} for model learning that is designed for model-based PG estimation. We specialize the derivation to discounted MDPs, but the changes for the episodic, finite-horizon, or average reward MDPs are straightforward.


A PG method relies on accurate estimation of the gradient. Intuitively, a model-based PG method would perform well if the gradient of the performance evaluated according to the model $\PKernelhat$ is close to the one computed from the true dynamics $\PKernelTrue$. In this case, one may use the learned model instead of the true environment to compute the PG. To formalize this intuition, we first introduce some notations.


Given a transition probability kernel $\PKernelpi$, we denote by $\PKernel^\pi (\cdot|x;k)$ the future-state distribution of following policy $\pi$ from state $x$ for $k$ steps, i.e., $\PKernel^\pi (\cdot|x;k) \eqdef (\PKernel^\pi)^k(\cdot|x)$, with the understanding that $(\PKernel^\pi)^0(\cdot|x) = \Id$ is the identity map. 
For an initial probability distribution $\rho \in \bar{\MM}(\XX)$, $\int \rho(\dx) \PKernelpi(\cdot|x;k)$ is the distribution of selecting the initial distribution according to $\rho$ and following $\PKernelpi$ for $k$ steps.
We define a discounted future-state distribution of starting from $\rho$ and following $\PKernelpi$ as
\ifSupp
\begin{align}
\label{eq:PAML-DiscountedFutureStateDistribution}
	\rhoDiscounted^\pi(\cdot) = \rhoDiscounted(\cdot; \PKernelpi) \eqdef 
	(1-\gamma) \sum_{k \geq 0} \gamma^k \int \drho(x) \PKernelpi(\cdot|x;k).
\end{align}
\else
$\rhoDiscounted^\pi(\cdot) = \rhoDiscounted(\cdot; \PKernelpi) \eqdef 
	(1-\gamma) \sum_{k \geq 0} \gamma^k \int \drho(x) \PKernelpi(\cdot|x;k)$.
\fi
%
%
We may drop the dependence on $\pi$, if it is clear from the context. We use a shorthand notation $\rhohatDiscounted^\pi = \rhoDiscounted(\cdot;\PKernelpihat)$, and a similar notation for other distributions, e.g.,~$\muDiscounted^\pi$ and $\muhatDiscounted^\pi$.


For an MDP $(\XX, \AA, \RKernel^*, \PKernel, \gamma)$, we use the subscript $\PKernel$ in the definition of value function $\Vpi_\PKernel$ and $\Qpi_\PKernel$, if we want to emphasize its dependence on the transition probability kernel. We reserve the use of $\Vpi$ and $\Qpi$ for $\Vpi_{\PKernelTrue}$ and $\Qpi_{\PKernelTrue}$, the value functions of the true dynamics.

The performance of an agent starting from a user-defined initial probability distribution $\rho \in \bar{\MM}(\XX)$, following policy $\pi$ in the true MDP $(\XX, \AA, \RKernel^*, \PKernelTrue, \gamma)$ is
%
\begin{align}
\label{eq:PAML-PerformanceMeasure}
	J(\pi) = J_\rho(\pi) = \int \mathrm{d}\rho(x) \Vpi(x).
\end{align}
%

When the policy $\pi = \pi_\theta$ is parameterized by $\theta \in \Theta$, from the derivation of the PG theorem (cf. proof of Theorem 1 by~\citealt{SuttonMcAleesterSinghMansour2000}), we have that%

\ifSupp
\begin{align*}
	\frac	{\partial V^{\pi_\theta}(x)}
		{\partial \theta}
	& =
	\sum_{k \geq 0} \gamma^k 
		\int		
			{\PKernelTrue}^{\pi_\theta} (\dx'|x;k)
			\sum_{a' \in \AA} \frac{\partial \pi_\theta(a'|x')}{\partial \theta} Q^{\pi_\theta}_{\PKernelTrue} (x',a'),
\end{align*}
\else
%
\begin{small}
\begin{align*}
	\frac{\partial V^{\pi_\theta}(x)}{\partial \theta} = 
	\sum_{k \geq 0} \gamma^k \int	{\PKernelTrue}^{\pi_\theta}(\dx'|x;k)\sum_{a' \in \AA} \frac{\partial \pi_\theta(a'|x')}{\partial \theta} 	Q^{\pi_\theta}_{\PKernelTrue} (x',a').
\end{align*}
\end{small}
%
\fi

%
If the dependence of $Q$ on the transition kernel is clear, we may omit it and simply use $Q^{\pi_\theta}$. 
We also use $\PKernel^{\pi_\theta}$ instead of ${\PKernelTrue}^{\pi_\theta}$ to simplify the notation.
The gradient of the performance $J(\pi_\theta)$~\eqref{eq:PAML-PerformanceMeasure} w.r.t. $\theta$ is then
\ifSupp
\begin{align}
\label{eq:PAML-GradientOfPerformance}
\nonumber
	\nabla_\theta J(\pi_\theta) = 
	\frac{\partial{J(\pi_\theta)}}{\partial \theta}
	& = 
	\sum_{k \geq 0} \gamma^k 
		\int \drho(x)
		\int		
			\PKernel^{\pi_\theta}(\dx'|x;k)
			\sum_{a' \in \AA} \frac{\partial \pi_\theta(a'|x')}{\partial \theta} Q^{\pi_\theta} (x',a')
	\\
	& =
	\frac{1}{1 - \gamma}	
	\int \rhoDiscounted(\dx;\PKernel^{\pi_\theta})
	\sum_{a \in \AA} \pi_\theta(a|x) \frac{\partial \log \pi_\theta(a|x)}{\partial \theta} Q^{\pi_\theta} (x,a).
\end{align}
\else
%
\begin{small}
\begin{align}
\label{eq:PAML-GradientOfPerformance}
	& 
	\nabla_\theta J(\pi_\theta) = 
	\frac{\partial{J(\pi_\theta)}}{\partial \theta}
	=
	\\
	&
	\sum_{k \geq 0} \gamma^k 
		\int \drho(x)
		\int		
			\PKernel^{\pi_\theta}(\dx'|x;k)
			\sum_{a' \in \AA} \frac{\partial \pi_\theta(a'|x')}{\partial \theta} Q^{\pi_\theta} (x',a'). \nonumber
\end{align}
\end{small}
%
\fi

%

%
%

Let us expand the definition of $\nabla_\theta J$, which shall help us easily describe several ways a model-based PG method can be devised. 
For two transition probability kernels $\PKernel_1$ and $\PKernel_2$, and a policy $\pi_\theta$, we define

\ifSupp
\begin{align}
\label{eq:PAML-GradientOfPerformance-G}
	\nabla_\theta J(\pi_\theta;\PKernel_1, \PKernel_2) = 
	\sum_{k \geq 0} \gamma^k 
		\int \drho(x)
		\int		
			\PKernel_1^{\pi_\theta}(\dx'|x;k)
			\sum_{a' \in \AA} \frac{\partial \pi_\theta(a'|x')}{\partial \theta} Q^{\pi_\theta}_{\PKernel_2} (x',a').
\end{align}
\else
%
\begin{align}
\label{eq:PAML-GradientOfPerformance-G}
\nonumber
	\nabla_\theta J(\pi_\theta;\PKernel_1, \PKernel_2) = 
	&
	\sum_{k \geq 0} \gamma^k 
		\int \drho(x)
		\int		
			\PKernel_1^{\pi_\theta}(\dx'|x;k)
			\\ &
			\sum_{a' \in \AA} \frac{\partial \pi_\theta(a'|x')}{\partial \theta} Q^{\pi_\theta}_{\PKernel_2} (x',a'). 
\end{align}

\fi
This vector-valued function can be seen as the PG of following $\pi_\theta$ according to $\PKernel_1$, and using a critic that is the value function in an MDP with $\PKernel_2$ as the transition kernel. 

We have several choices to design a model learning loss function that is suitable for a PG method. 
The overall goal is to match the true PG,
i.e.,~$\frac{\partial{J(\pi_\theta)}}{\partial \theta} = \nabla_\theta J(\pi_\theta; \PKernelTrue, \PKernelTrue)$, or an empirical estimate thereof, with a PG that is somehow computed by the model $\PKernelhat$.
Let us define $\frac{\partial{\hat{J}(\pi)}}{\partial \theta} \eqdef \nabla_\theta J(\pi_\theta; \PKernelhat, \PKernelTrue)$ and set the goal of model learning to
%
%
%
\begin{align}
\label{eq:PAML-GradientsBeingEqual}
	\frac{\partial{J(\pi_\theta)}}{\partial \theta} \approx \frac{\partial{\hat{J}(\pi_\theta)}}{\partial \theta}.
\end{align}
This ensures that the gradient estimate based on following the learned model $\PKernelhat^{\pi_\theta}$ and computed using the true action-value function $\Qpi_{\PKernelTrue}$ is close to the true gradient. 
There are various ways to quantify the error between the gradient vectors. 
We choose the $\ell_2$-norm of their difference.
When the gradient w.r.t.~the model is close to the true gradient, we may use the model to perform PG.

Subtracting the gradient of the performances under two different transition probability kernels and taking the $\ell_2$-norm, we get a loss function between the true and model PGs, i.e.,

\ifSupp
\begin{align}
\label{eq:PAML-DifferenceInGradient}
\nonumber
c_\rho(\PKernel^{\pi_\theta},\PKernelhat^{\pi_\theta}) & = 
	\norm{
	\frac{\partial{J(\pi_\theta)}}{\partial \theta} - \frac{\partial{\hat{J}(\pi_\theta)}}{\partial \theta}}_2
	\\
&	=
	\norm{\sum_{k \geq 0} \gamma^k 
		\int \drho(x)
		\int		
			\left( \PKernel^{\pi_\theta}(\dx'|x;k) -  \PKernelhat^{\pi_\theta}(\dx'|x;k) \right)
			\sum_{a' \in \AA} \frac{\partial \pi_\theta(a|x')}{\partial \theta} Q^{\pi_\theta} (x',a')}_2.			
\end{align}
\else
%
\begin{align}
\label{eq:PAML-DifferenceInGradient}
\nonumber
& c_\rho(\PKernel^{\pi_\theta},\PKernelhat^{\pi_\theta}) = 
	\Big \Vert
	\frac{\partial{J(\pi_\theta)}}{\partial \theta} - \frac{\partial{\hat{J}(\pi_\theta)}}{\partial \theta} \Big \Vert_2
	=
	\\
\nonumber	
&
	\Big \Vert \sum_{k \geq 0} \gamma^k 
		\int \drho(x)
		\int		
			\left( \PKernel^{\pi_\theta}(\dx'|x;k) -  
			\PKernelhat^{\pi_\theta}(\dx'|x;k) \right) 
			\\
			&
			\qquad \qquad \qquad \qquad
			\sum_{a' \in \AA} \frac{\partial \pi_\theta(a|x')}{\partial \theta} Q^{\pi_\theta} (x',a')
	\Big \Vert_2.
\end{align}
\fi
Note that the summation (or integral) over actions can be done using any of the commonly-known Monte Carlo gradient estimators such as the score function (REINFORCE) or pathwise gradient estimators \citep{mohamed2019monte}, as we demonstrate in the Empirical Studies.

Several comments are in order.
This is a population loss function, in the sense that $\PKernel$ appears in it. To make this loss practical, we need to use its empirical version. 
Moreover, this formulation requires us to know the action-value function $Q^{\pi_\theta} = Q^{\pi_\theta}_{\PKernelTrue}$, which is w.r.t.~the true dynamics. 
This can be estimated using a model-free critic that only uses the real transition data (and not the data obtained by the model $\PKernelpihat$) \ifSupp and provides $\Qhat^{\pi_\theta} \approx Q^{\pi_\theta}_{\PKernelTrue}$. \else. \fi
\ifSupp

To be concrete, let us assume that we are given $n$ episodes with length $T$ of following policy $\PKernel^{\pi_\theta}$ starting from an initial state distribution $\rho$. That is, $X^{(i)}_1 \sim \rho$, $A^{(i)}_k \sim \pi_\theta(\cdot|X^{(i)}_k)$ and $X^{(i)}_{k+1} \sim \PKernelTrue(\cdot|X^{(i)}_{k}, A^{(i)}_k)$ for $k = 0, \dotsc, T-1$ and $i = 1, \dotsc, n$.
In order to compute the expectation w.r.t. $\rhohatDiscounted$, we generate samples from $\PKernelhat^{\pi_\theta}$ as follows:
For each $i = 1, \dotsc, n$ and $j = 1, \dotsc, m$, we set $\tilde{X}^{(i)}_{1, j} = X^{(i)}_1$ (the same initial states as in the real data).
And for the next steps, we let
$\tilde{A}^{(i)}_{k,j} \sim \pi_\theta(\cdot|X^{(i)}_k)$ and
 $\tilde{X}^{(i)}_{k+1, j} \sim \PKernelhat(\cdot|\tilde{X}^{(i)}_{k, j}, \tilde{A}^{(i)}_{k,j})$ (which should be interpreted as the same model as $\PKernelhat^{\pi_\theta}$)
 for  $i = 1, \dotsc, n$, $j = 1, \dotsc, m$, and $k = 0, \dotsc, T - 1$.
The empirical loss can then be defined as
\begin{align}\label{eq:empiricalloss}
\nonumber
c_n(\PKernel^{\pi_\theta},\PKernelhat^{\pi_\theta}) = 
	\Bigg\Vert
	\frac{1}{n}
	\sum_{i=1}^n
		\sum_{k=1}^T \gamma^k
			\Bigg[
			&
			\nabla_\theta \log \pi_\theta (A_k^{(i)} |X_k^{(i)} ) \Qhat^{\pi_\theta} (X_k^{(i)},A_k^{(i)})
			-
			\\
			&
			\frac{1}{m} \sum_{j=1}^m \nabla_\theta \log \pi_\theta (\tilde{A}_{k,j}^{(i)} |X_k^{(i)} ) \Qhat^{\pi_\theta} (X_{k,j}^{(i)},\tilde{A}_{k,j}^{(i)})			
			\Bigg]
	\Bigg\Vert_2.
\end{align}
\else
We provide the empirical version of this loss function in the supplementary material.
\fi

\todo{Add the compatible feature and FIM formulation.}


\todo{To add discussion: Finite $T$ introduces approximation. Finite $m$ instead of computing expectation. Deterministic models requires $m = 1$.}

Also note the loss function $c_\rho(\PKernel^{\pi_\theta},\PKernelhat^{\pi_\theta})$~\eqref{eq:PAML-DifferenceInGradient} is defined for a particular policy $\pi_\theta$. 
However, policy $\pi_\theta$ gradually changes during the run of a PG algorithm.
To deal with this change, we should regularly update $\PKernelhat^{\pi_\theta}$ based on data collected by the most recent $\pi_\theta$.
\ifSupp
In Section~\ref{sec:PAML-Theory-PolicyChange},
\else
In the supplementary material, 
\fi
we show that under certain conditions, the error in the PG estimation of a new policy $\pi_{\theta_\text{new}}$ using a model that was learned for an old policy $\pi_{\theta_\text{old}}$ is $O(\norm{\theta_\text{new} - \theta_\text{old}})$.
Our empirical studies (including in the Supp) show that the model does not expire quickly with policy changes.
%
%

Setting $\nabla_\theta J(\pi_\theta; \PKernelTrue, \PKernelTrue) \approx \nabla_\theta J(\pi_\theta; \PKernelhat, \PKernelTrue)$ is only one way to define a model learning objective for a PG method. 
We can opt instead to find a $\PKernelhat$ such that
%
{\small
\begin{subnumcases}
{
\label{eq:gradients-cases}
\nabla_\theta J(\pi_\theta; \PKernelTrue, \PKernelTrue) \approx
}
		\nabla_\theta J(\pi_\theta; \PKernelhat, \PKernelTrue), \label{eq:gradients-cases-a} \\
		\nabla_\theta J(\pi_\theta; \PKernelTrue, \PKernelhat),  \label{eq:gradients-cases-b} \\
		\nabla_\theta J(\pi_\theta; \PKernelhat, \PKernelhat).  \label{eq:gradients-cases-c}
\end{subnumcases}}
The difference between these cases is in whether the discounted future-state distribution is computed according to the true dynamics $\PKernelTrue$ or the learned dynamics $\PKernelhat$, and whether the critic $Q^{\pi_\theta}_{\PKernel}$ is computed according to the true dynamics or the learned dynamics.
Case~\eqref{eq:gradients-cases-a} is the same as~\eqref{eq:PAML-GradientsBeingEqual}.
Case~\eqref{eq:gradients-cases-b} uses the model only to train the critic, but not to compute the future-state distribution $\rhohatDiscounted$. 
Having $\PKernelhat$, the critic can be estimated using Monte Carlo estimates or any other method for estimating the value function given a model.
This is similar to how~\citet{DOroMetelliTirinzoniPapiniRestelli2020} use their model, though their loss function is a weighted log-likelihood, and the model-learning step does not take the action-value function into account.\todo{Verify! And explain how different.}
Case~\eqref{eq:gradients-cases-c} corresponds to calculating the whole PG according to the model $\PKernelpihat$. This requires us to estimate both the future-state distribution and the critic according to the model.
In this paper, we theoretically analyze~\eqref{eq:gradients-cases-a} and provide empirical results for approximations of~\eqref{eq:gradients-cases-a} and~\eqref{eq:gradients-cases-b}.
\todo{Doesn't the new finite MDP case with exact value and policy correspond to~\eqref{eq:gradients-cases-c}? -AMF}

\ifSupp If the policy is deterministic~\citep{silver2014deterministic}, the formulation would be almost the same with the difference that instead of terms in the form of $\sum_{a \in \AA} \nabla_\theta \pi_\theta(a|x) Q^{\pi_\theta}(x,a)$, we would have $\nabla_\theta \pi_\theta(x) \frac{\partial Q^{\pi_\theta}(x,a)}{\partial a}  \vert_{a = \pi_\theta(x)}$. \fi
%


\section{Theoretical Analysis of PAML}
\label{sec:PAML-Theory}
\ifSupp

\ifSupp
We theoretically study some aspects of a generic model-based PG (MBPG) method.
Theorem~\ref{thm:PAML-PolicyGradientError} in Section~\ref{sec:PAML-Theory-PG-Error} quantifies the error between the PG $\nabla_\theta J(\pi_\theta)$ obtained by following the true model $\PKernelTrue$ and the PG $\nabla_\theta \hat{J}(\pi_\theta)$ obtained by following $\PKernelhat$, and relates it to the error between the models.
Even though having a small PG error might intuitively suggest that using $\nabla_\theta \hat{J}(\pi_\theta)$ instead of $\nabla_\theta J(\pi_\theta)$ should lead to a good policy, it does not show the quality of the converged policy.
Theorem~\ref{thm:PAML-MBPG-Convergence} in Section~\ref{sec:PAML-Theory-PG-Convergence} (for exact critic)
and Theorem~\ref{thm:PAML-MBPG-Convergence-InexactCritic} in Section~\ref{sec:PAML-Theory-PG-Convergence-Inexact-Critic} (for inexact critic) provide such convergence guarantees for a MBPG and shows that having a small PG error indeed leads to a better solution.
\todo{Make sure to add a comparison after Theorem~\ref{thm:PAML-PolicyGradientError}.} 

\else
We theoretically study some aspects of a generic model-based PG (MBPG) method, and provide justification for using our proposed loss as an alternative to MLE.
One of our contributions is a PG error result that quantifies the error between $\nabla_\theta J(\pi_\theta)$ obtained following the true model $\PKernelTrue$ and $\nabla_\theta \hat{J}(\pi_\theta)$ obtained following $\PKernelhat$, and relates it to the error between the models.
We only briefly report this result here and defer its detailed description to the supplementary material.

Suppose that the policy $\pi_\theta: \XX \ra \bar{\MM}(\AA)$ is from the exponential family with features $\phi = \phi(a|x): \XA \ra \Real^d$, parameters $\theta \in \Theta \subset \Real^d$, and has the density 
\begin{equation}
\label{eq:PAML-PolicyParametrization}
	\pi_\theta(a | x) = 
		\frac	{\exp \left( \phi^\top(a|x) \theta \right)}
			{\int \exp \left( \phi^\top(a'|x) \theta \right) \da'}.
\end{equation}
The result states that the PG error $\smallnorm{
		\nabla_\theta J(\pi_\theta) - 
		\nabla_\theta \hat{J}(\pi_\theta)
		}_2$
can be upper-bounded by
\begin{align}\label{eq:PAML-Exp-UpperBound}
	\frac{\gamma \Qmax B_2}{(1-\gamma)^2}
	\times
	\begin{cases}
		\cPG(\rho, \nu; \pi_\theta) \norm{\Delta \PKernel^{\pi_\theta}}_{1,1(\nu)}, \\
		2 \norm{\Delta \PKernel^{\pi_\theta}}_{1,\infty},
	\end{cases}
\end{align}
where $B_2$ is the $\ell_2$-norm of features used in the definition of the policy, 
$\norm{\Delta \PKernel^{\pi_\theta}}_{1,1(\nu)}$ and 
$\norm{\Delta \PKernel^{\pi_\theta}}_{1,\infty}$ are total variation-based norms of the model error
$\DeltaPKernelpi = \PKernelpi - \PKernelpihat$,
and the concentrability coefficient
$
	\cPG(\rho, \nu; \pi_\theta) 
	\eqdef
	\smallnorm{
		\frac	{\mathrm{d} \rhoDiscounted^{\pi_\theta} }
			{\dnu	}
		}_\infty
$ is the supremum of the Radon-Nikodym derivative of distribution $\rhoDiscounted^{\pi_\theta}$ w.r.t.~$\nu$. Moreover, through an application of Pinsker's inequality, we show in the supplementary material that the upper-bound in \eqref{eq:PAML-Exp-UpperBound} can be further upper-bounded by the $\KL$-divergence between $\PKernelpi$ and $\PKernelpihat$. Recall that the MLE is the minimizer of the $\KL$-divergence between the empirical distribution of samples generated from $\PKernelpi$ and $\PKernelpihat$. This upper-bound suggests why PAML might be a more suitable approach in learning a model. An MLE-based approach tries to minimize an upper-bound of an upper-bound for the quantity that we care about, i.e.,~PG error. This consecutive upper-bounding might be quite loose. On the other hand, the population version of PAML's loss \eqref{eq:PAML-DifferenceInGradient} is exactly the error in the PG estimates that we care about. A question that may arise is that although these two losses are different, are their minimizers the same? In Figures~\ref{fig:PAML-GMM-visualization} and~\ref{fig:PAML-GMM-contour} we show through a simple visualization that the minimizers of PAML and $\KL$ could indeed be different. We then report Theorem~\ref{thm:PAML-MBPG-Convergence-Compact} that provides a convergence guarantee for a MBPG method.

\ifSupp
\begin{figure}[t]
    \centering
    \includegraphics[width=1.0\textwidth]{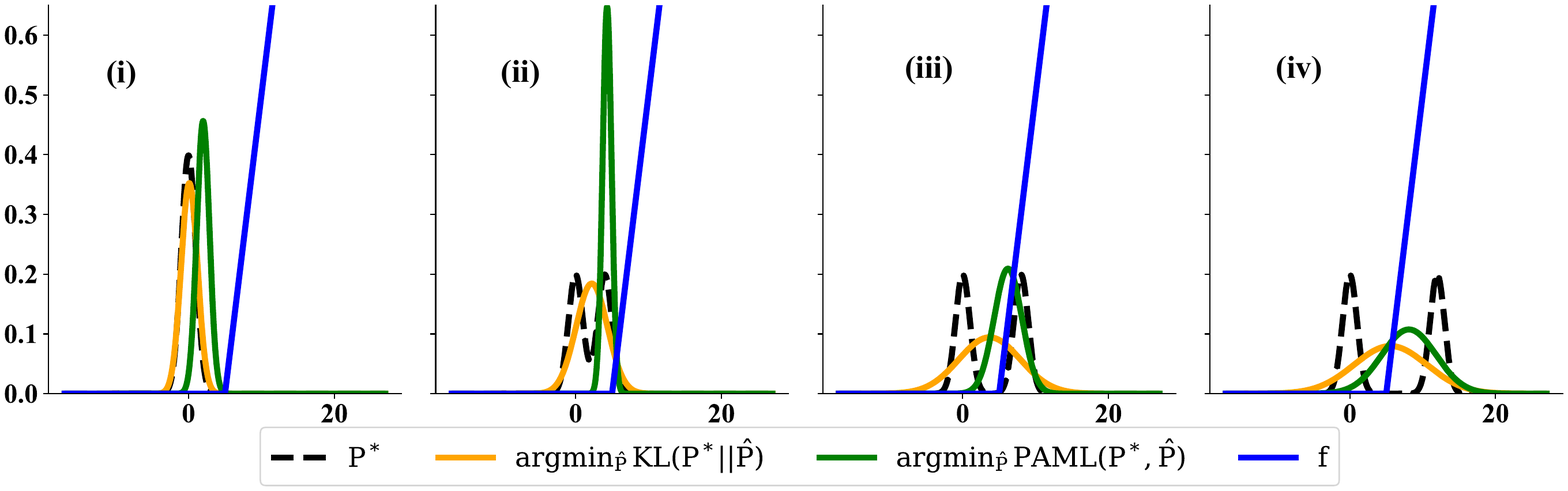}
    \includegraphics[width=1.0\textwidth]{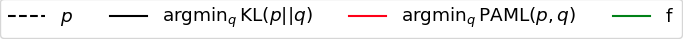}
\end{figure}
\else
\begin{figure}[t]
    \centering
    \begin{subfigure}[b]{0.55\linewidth}
     \centering
        \includegraphics[width = \linewidth]{plots/annotated_reverse_forward_kl.pdf}
        \caption{Visualization of minimizers for PAML and MLE. $\PKernelTrue$, denoted by $p$, is a Gaussian mixture model with 2 modes and the learned model $\PKernelhat$, denoted by $q$, is a single Gaussian. The loss minimized by PAML for this simple case is: $|\sum_x (\PKernelTrue - \PKernelhat)f(x)|^2$, where f is an arbitrary function.}
        \label{fig:PAML-GMM-visualization}%
    \end{subfigure} 
    \hspace{1em}
    \begin{subfigure}[b]{0.4\textwidth}
        \centering
        \includegraphics[width = \linewidth]{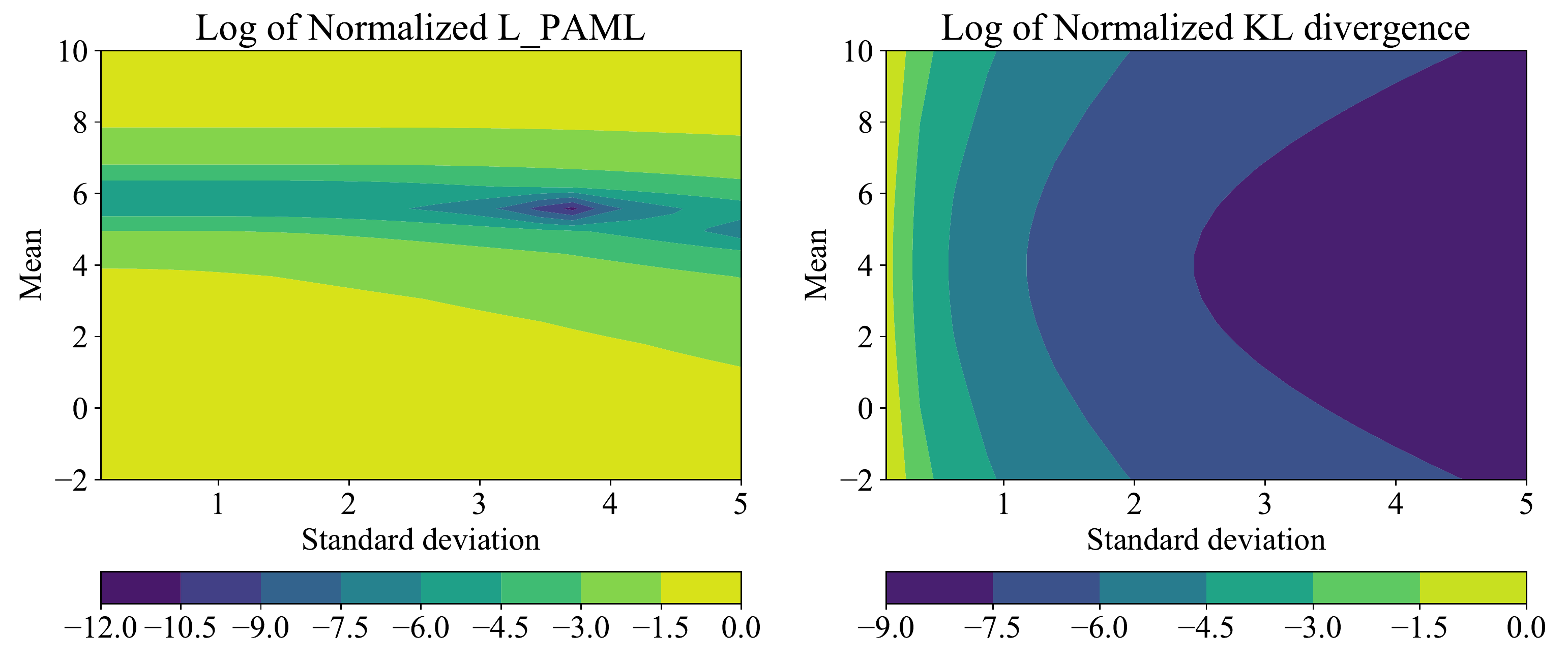}
        \caption{Contours of the two loss surfaces for Figure \ref{fig:PAML-GMM-visualization}(iii). Note the locations of the minimizers for each. (The losses were log-normalized for better visual contrast in this figure.)}
        \label{fig:PAML-GMM-contour}%
    \end{subfigure}
\end{figure}%
\fi
%

\fi

\ifSupp

\subsection{Policy Gradient Error}
\label{sec:PAML-Theory-PG-Error}

\todo{Mention that $\PKernel^\pi$ is really ${\PKernelTrue}^\pi$.}

Given the true $\PKernel = \PKernelTrue$ and estimated $\PKernelhat$ transition probability kernel, a policy $\pi_\theta$, and their induced discounted future-state distributions $\rhoDiscounted^{\pi_\theta} = \rhoDiscounted(\cdot;\PKernel^{\pi_\theta})$ and $\rhohatDiscounted^{\pi_\theta} =  \rhoDiscounted(\cdot;\PKernelhat^{\pi_\theta})$, the performance gradients $\frac{\partial J(\pi_\theta)}{\partial \theta}$ (according to $\PKernel^{\pi_\theta}$) and
$\frac{\partial \hat{J}(\pi_\theta)}{\partial \theta}$  (according to $\PKernelhat^{\pi_\theta}$) are
\begin{align}
\label{eq:PAML-J-and-Jhat}
\nonumber
	& \frac{\partial{J(\pi_\theta)}}{\partial \theta} 
	=
	\frac{1}{1 - \gamma}	
	\EEX{X \sim \rhoDiscounted(\cdot;\PKernel^{\pi_\theta})}
		{ \EEX{A \sim \pi_\theta(\cdot|X)} {\nabla_\theta \log \pi_\theta(A|X) Q^{\pi_\theta}(X,A)}
	},
	\\
	& \frac{\partial{\hat{J}(\pi_\theta)}}{\partial \theta} 
	=
	\frac{1}{1 - \gamma}	
	\EEX{X \sim \rhoDiscounted(\cdot;\PKernelhat^{\pi_\theta})}
		{ \EEX{A \sim \pi_\theta(\cdot|X)} {\nabla_\theta \log \pi_\theta(A|X) Q^{\pi_\theta}(X,A)}
	}.
\end{align}

%

We want to compare the difference of these two PG.
Recall that this is the case when the same critic $Q^{\pi_\theta} = Q^{\pi_\theta}_\PKernel$ is used for both PG calculation, e.g., a critic is learned in a model-free way, and not based on $\PKernelhat$.
At first we assume that the critic is exact, and we do not consider that one should learn it based on data, which brings in considerations on the approximation and estimation errors for a policy evaluation method. Theoretical analyses on how well we might learn a critic have been studied before.\todo{Cite relevant papers!}
We study the effect of having an inexact critic in Section~\ref{sec:PAML-Theory-PG-Convergence-Inexact-Critic}.

We first introduce some notations and definitions.
Let us denote $\DeltaPKernelpi = \PKernelpi - \PKernelpihat$.
For the error in transition kernel $\DeltaPKernelpi(\cdot|x)$, which is a signed measure, we use 
$\norm{\DeltaPKernelpi(\cdot|x)}_1$ to denote its total variation (TV) distance, i.e.,
$\norm{\DeltaPKernelpi(\cdot|x)}_1 = 2 \sup_{A \in \XX} |\int \DeltaPKernelpi(\dy|x) \One{y \in A}|$, where the supremum is over the measurable subsets of $\XX$.
We have
\begin{align}
\label{eq:PAML-TV-Distance}
	\norm{\DeltaPKernelpi(\cdot|x)}_1 = \sup_{\norm{f}_\infty \leq 1} \left| \int \DeltaPKernelpi(\dy|x) f(y) \right|,	
\end{align}
where the supremum is over $1$-bounded measurable functions on $\XX$.
When $\DeltaPKernelpi(\cdot|x)$ has a density w.r.t. to some countably additive non-negative measure, it holds that
$\norm{\DeltaPKernelpi(\cdot|x)}_1 = \int |\DeltaPKernelpi(\dy|x)|$.

We define the following two norms on $\DeltaPKernelpi$: 
Consider a probability distribution $\nu \in \bar{\MM}(\XX)$.
We define
\begin{align*}
	& \norm{\DeltaPKernelpi}_{1, \infty} = \sup_{x \in \XX} \norm{\DeltaPKernelpi(\cdot|x)}_{1},
	& \norm{\DeltaPKernelpi}_{1, 1(\nu)} = \int \dnu(x) \norm{\DeltaPKernelpi(\cdot|x)}_{1}.
\end{align*}
%
%
\todo{Come back to this and verify, especially whether I need a 2 factor in the definition of TV or not.}
\todo{Add intuition. -AMF}


We also define
\begin{align*}
	& \KL_{\infty}(\PKernelpi_1 || \PKernelpi_2) = \sup_{x \in \XX} \KL(\PKernelpi_1(\cdot|x) || \PKernelpi_2 (\cdot|x) ),
	& \KL_{1(\nu)}(\PKernelpi_1 || \PKernelpi_2) = \int \dnu(x) \KL(\PKernelpi_1(\cdot|x) || \PKernelpi_2 (\cdot|x) ).
\end{align*}

\todo{Add intuition. -AMF}

For an action-value function $Q: \XA \ra \Real$, a probability distribution $\nu \in \bar{\MM}(\XX)$ and a policy $\pi: \XX \ra \bar{\MM}(\AA)$, and $p \in [1, \infty)$, we define the following norm:
\begin{align}
\label{eq:PAML-Norm-Q}
	\norm{Q}_{p(\nu;\pi)} = \sqrt[p]{\int \dnu(x) \pi(\da|x) |Q(x,a)|^p}.
\end{align}

When the action-space is finite $\actionnum< \infty$, we also define the following norm:
\begin{align}
\label{eq:PAML-Norm-Q-UniformOverActions}
	\norm{ Q(x,\cdot) }_2
	=
	\sqrt{
		\frac{1}{\actionnum}
		\sum_{a \in \Actions}
		\left| Q(x,a) \right|^2
		}.
\end{align}
This can be seen as similar to the norm~\eqref{eq:PAML-Norm-Q} with the choice of a policy $\pi$ that assigns a uniform probability $\frac{1}{\actionnum}$ to all actions in the finite action space (but it is not exactly the same, because here we are only concerned of a specific state $x$, instead of a distribution $\nu$ over the states).
These norms will be used later when we analyze the effect of the critic error.

\ifSupp
For a vector-valued function $f: \XX \ra \Real^d$ (for some $d \geq 1$), we define its mixed $p,\infty$-norm as
\begin{align}
\label{eq:PAML-Norm-f}
	\norm{f}_{p,\infty} = \sup_{x \in \XX} \norm{f(x)}_p.
\end{align}
\fi

We shall see that the error in gradient depends on the discounted future-state distribution~\eqref{eq:PAML-DiscountedFutureStateDistribution}. Sometimes we may want to express the errors w.r.t. another distribution $\nu \in \bar{\MM}(\XX)$, which is, for example but not necessarily, the distribution used to collect data to train the model.
This requires a change of measure argument.
Recall that for a measurable function $f: \XX \ra \Real$ and two probability measures $\mu_1, \mu_2 \in \bar{\MM}(\XX)$, if $\mu_1$ is absolutely continuous w.r.t. $\mu_2$ ($\mu_1 \ll \mu_2$), the Radon-Nikydom (R-N) derivative $\frac{\dmu_1}{\dmu_2}$ exists, and we have
\begin{align}
\label{eq:PAML-Change-of-Measure-Argument}
	\int f(x) \dmu_1 = \int f(x) \frac{\dmu_1}{\dmu_2} \dmu_2
	\leq
	\norm{\frac{\dmu_1}{\dmu_2}}_\infty \int |f(x)| \dmu_2,
\end{align}
where %
\[
\norm{\frac{\dmu_1}{\dmu_2}}_\infty = \sup_{x} \left| \frac{\dmu_1}{\dmu_2}(x) \right|.
\]

The supremum of the R-N derivative of $\rhoDiscounted^\pi$ w.r.t. $\nu$ plays an important role in our results. It is called the Discounted Concentrability Coefficient. We formally define it next.

\begin{definition}[Discounted Concentrability Coefficient]
Given two distributions $\rho, \nu \in \bar{\MM}(\XX)$ and a policy $\pi$, define
\begin{align*}
	\cPGrhopinu 
	\eqdef
	\norm{
		\frac	{\mathrm{d} \rhoDiscounted^\pi}
			{\dnu	}
		}_\infty.
\end{align*}
If the discounted future-state distribution is not absolutely continuous w.r.t. $\nu$, we set $\cPGrhopinu = \infty$.
\end{definition}
\todo{Some discussion!}

\ifSupp
The models $\PKernelpi$ and $\PKernelpihat$ induce discounted future-state distributions $\rhoDiscounted^\pi$ and $\rhohatDiscounted^\pi$. We can compare the expectation of a given (vector-valued) function under these two distributions. The next lemma upper bounds the difference in the $\ell_p$-norm of these expectations, and relates it to the error in the models, and some MDP-related quantities.
\todo{Maybe mention that this is for the TV type of error. Not what matters for PAML. -AMF}

\begin{lemma}
\label{lem:PAML-ExpectedDiscountedError}
Consider a vector-valued function $f: \XX \ra \Real^d$, two distributions $\rho, \nu \in \bar{\MM}(\XX)$, and
two transition probability kernels $\PKernelpi$ and $\PKernelpihat$.
For any $0 \leq \gamma < 1$, and $1 \leq p \leq \infty$, we have
\begin{align*}
	\norm{
			\EEX{X \sim \rhoDiscounted(\cdot;\PKernelpi)}{f(X)} - 
		\EEX{ X \sim \rhoDiscounted(\cdot;\PKernelpihat)}{f(X)}			
		}_p
	\leq
	\frac{\gamma}{1-\gamma}
	\norm{f}_{p, \infty} \times
	\begin{cases}
		\cPGrhopinu \norm{\Delta \PKernelpi}_{1,1(\nu)},
		 \\
		\norm{\Delta \PKernelpi}_{1,\infty}.
	\end{cases}
\end{align*}
\end{lemma}
\begin{proof}
For $k = 0, 1, \dotsc$, define the vector-valued function $E_k: \XX \ra \Real^d$ by
\begin{align*}
	E_k(x) = 
	\int \left( \PKernelpi(\dy|x;k) - \PKernelpihat(\dy|x;k) \right) f(y),
\end{align*}
for any $x \in \XX$. Note that $E_0(x) = 0$.
So we can write
\begin{align}
\label{eq:PAML-ExpectedDiscountedError-Proof-MainTerm}
\nonumber
	\frac{1}{1 - \gamma}
	\int \left( \rhoDiscounted(\dx;\PKernelpi) - \rhoDiscounted(\dx;\PKernelpihat) \right) f(x)
	& =
	\sum_{k \geq 0} \gamma^k \int \drho(x) \left( \PKernelpi(\dy|x;k) - \PKernelpihat(\dy|x;k) \right) f(y)
	\\ 
	& =
	\sum_{k \geq 0} \gamma^k \int \drho(x) E_k(x).
\end{align}

In order to upper bound the norm of~\eqref{eq:PAML-ExpectedDiscountedError-Proof-MainTerm}, we provide an upper bound on the norm of each $E_k$. We start by inductively expressing $E_{k}$ as a function of $E_{k-1}$, $\DeltaPKernelpi$, and other quantities.
For $k \geq 1$, and for any $x \in \XX$, we write
\begin{align*}
	E_{k}(x) & = 
	\int \left( \PKernelpi(\dy|x;k) - \PKernelpihat(\dy|x;k) \right) f(y)
	\\
	& =
	\int \left( \PKernelpi(\dx'|x) \PKernelpi(\dy|x';k-1) - \PKernelpihat(\dx'|x) \PKernelpihat(\dy|x';k-1) \right) f(y)
	\\
	& =
	\int \left( \PKernelpi(\dx'|x) \PKernelpi(\dy|x';k-1) -
		\left[ \PKernelpi(\dx'|x) - \DeltaPKernelpi(\dx'|x) \right] \PKernelpihat(\dy|x';k-1) \right) f(y)
	\\
	& =
	\int \PKernelpi(\dx'|x) \left[ \PKernelpi(\dy|x';k-1) - \PKernelpihat(\dy|x';k-1) \right] f(y) +
		\DeltaPKernelpi(\dx'|x) \PKernelpihat(\dy|x';k-1) f(y)
	\\
	& =
	\int \PKernelpi(\dx'|x) E_{k-1}(x') +
		\DeltaPKernelpi(\dx'|x) \PKernelpihat(\dy|x';k-1) f(y).
\end{align*}

To simplify the notation, we denote the $\ell_p$-norm (for any $1 \leq p \leq \infty$) of $E_k(x)$ by $e_k(x)$, i.e., 
$
	e_k(x) = \norm{E_k(x)}_p
$.
Moreover, we denote $\eps(x) = \int |\DeltaPKernelpi(\dx'|x)|$.
By the convexity of the $\ell_p$-norm for $1 \leq p \leq \infty$ and the application of the Jensen's inequality, we get
\begin{align*}
	e_k(x) = \norm{E_k(x)}_p & = 
	\norm{
		\int \PKernelpi(\dx'|x) E_{k-1}(x') +
		\DeltaPKernelpi(\dx'|x) \PKernelpihat(\dy|x';k-1) f(y)
		}_p
	\\
	& \leq
	\int \PKernelpi(\dx'|x) \norm{E_{k-1}(x')}_p +
	\int \norm{ \DeltaPKernelpi(\dx'|x) \PKernelpihat(\dy|x';k-1) f(y) }_p
	\\
	& \leq
	\int \PKernelpi(\dx'|x) e_{k-1}(x') +
	\sup_{y \in \XX} \norm{f(y)}_p \int \left| \DeltaPKernelpi(\dx'|x) \right|
	\\
	& = 
	\int \PKernelpi(\dx'|x) e_{k-1}(x') + \norm{f}_{p,\infty} \eps(x).
\end{align*}
This relates $e_k$ to $e_{k-1}$ and $\eps$ and $\PKernelpi$. By unrolling $e_{k-1}, e_{k-2}, \dotsc$, we get
\begin{align*}
	e_k(x) & \leq
	\norm{f}_{p,\infty} \eps(x)
	+
	\int \PKernelpi(\dx'|x) \left[ \norm{f}_{p,\infty} \eps(x') + \int \PKernelpi(\dx'|x) e_{k-2}(x') \right]
	\leq \cdots
	\\
	&
	\leq
	\norm{f}_{p,\infty}
	\sum_{i=0}^{k-1} \int \PKernelpi(\dx'|x;i) \eps(x').
\end{align*}

As a result, the norm of~\eqref{eq:PAML-ExpectedDiscountedError-Proof-MainTerm} can be upper bounded as
\begin{align}
\label{eq:PAML-ExpectedDiscountedError-Proof-NormOfSumofEkTerm}
\nonumber
	\norm{
		\sum_{k \geq 0} \gamma^k \int \drho(x) E_k(x)
		}_p
	& \leq
	\sum_{k\geq0} \gamma^k \int \drho(x) e_k(x)
	\\
\nonumber	
	& \leq
	\norm{f}_{p,\infty}
	\int \drho(x) \sum_{k \geq 0} \gamma^k \sum_{i=0}^{k-1} \int \PKernelpi(\dx'|x;i) \eps(x')
	\\
\nonumber	
	& \leq
	\frac{\gamma}{1 - \gamma} \norm{f}_{p,\infty}
	\sum_{k \geq 0} \gamma^k \int \drho(x) \PKernelpi(\dx'|x;k) \eps(x')
	\\
	& =
	\frac{\gamma}{(1 - \gamma)^2} \norm{f}_{p,\infty}	
	\int \rhoDiscounted^\pi(\dx) \eps(x),
\end{align}
where we used the definition of $\rhoDiscounted^\pi$~\eqref{eq:PAML-DiscountedFutureStateDistribution} in the last equality.
By the change of measure argument, we have
\[
	\int \rhoDiscounted^\pi(\dx) \eps(x) = 
	\int \frac{ \mathrm{d} \rhoDiscounted^\pi }{\dnu }(x) \dnu(x) \eps(x)
	\leq
	\sup_{x \in \XX} \frac{ \mathrm{d} \rhoDiscounted^\pi }{\dnu }(x) \int \dnu(x) \eps(x)
	=
	\cPGrhopinu \norm{\eps}_{1(\nu)}.
\]

By~\eqref{eq:PAML-ExpectedDiscountedError-Proof-MainTerm} and~\eqref{eq:PAML-ExpectedDiscountedError-Proof-NormOfSumofEkTerm}, we get that
\begin{align*}
	\norm{
	\int \left( \rhoDiscounted(\dx;\PKernelpi) - \rhoDiscounted(\dx;\PKernelpihat) \right) f(x)
	}_p
	\leq
	\frac{\gamma}{1 - \gamma} \cPGrhopinu \norm{f}_{p,\infty} \norm{\eps}_{1(\nu)}.
\end{align*}
This leads to the first statement of the lemma.

Alternatively, as $\int \rhoDiscounted^\pi(\dx') \eps(x') \leq \sup_{x \in \XX} \eps(x)$, we can also upper bound~\eqref{eq:PAML-ExpectedDiscountedError-Proof-NormOfSumofEkTerm} by
\begin{align*}
	\frac{\gamma}{(1 - \gamma)^2} \norm{f}_{p,\infty}
	\norm{\eps}_{\infty},
\end{align*}
which leads to the second part of the result.
\end{proof}

Lemma~\ref{lem:PAML-ExpectedDiscountedError} is for a general function $f$.
By choosing 
$f(x) = f(x;\theta) = \EEX{A \sim \pi_\theta(\cdot|x)}{\nabla_\theta \log \pi_\theta(A|x) Q^{\pi_\theta}(x,A)}$,
one can provide an upper bound on the error in the PG~\eqref{eq:PAML-DifferenceInGradient}.
To be concrete, we consider a policy from the exponential family.
\todo{Any comment on how general the exponential family is? I think they can be dense in the space of probability distributions, but I should find the paper. -AMF}

\fi

Suppose that the policy $\pi_\theta: \XX \ra \bar{\MM}(\AA)$ is from the exponential family with features $\phi = \phi(a|x): \XA \ra \Real^d$ and parameterized by $\theta \in \Theta \subset \Real^d$, and has the probability (or density) of
\begin{align}
\label{eq:PAML-PolicyParametrization}
	\pi_\theta(a | x) = 
		\frac	{\exp \left( \phi^\top(a|x) \theta \right)}
			{\int \exp \left( \phi^\top(a'|x) \theta \right) \da'}.
\end{align}
If the dependence of $\pi_\theta$ on $\theta$ is clear from the context, we may simply refer to it as $\pi$.


\begin{theorem}
\label{thm:PAML-PolicyGradientError}
Consider the policy parametrization~\eqref{eq:PAML-PolicyParametrization}, the initial state distribution $\rho \in \bar{\MM}(\XX)$, and the discount factor $0 \leq \gamma < 1$.
The policy gradients w.r.t. the true model $\PKernel^{\pi_\theta}$ and the learned model $\PKernelhat^{\pi_\theta}$ are denoted by $\frac{\partial{J(\pi_\theta)}}{\partial \theta}$ and $\frac{\partial{\hat{J}(\pi_\theta)}}{\partial \theta}$, respectively~\eqref{eq:PAML-J-and-Jhat}.
Consider an arbitrary distribution $\nu \in \bar{\MM}(\XX)$.
Assume that $\norm{Q^{\pi_\theta}}_\infty \leq \Qmax$.
For $p \in \{2, \infty\}$, let $B_p = \sup_{(x,a) \in \XA} \norm{\phi(a|x)}_p$, and assume that $B_p < \infty$.
We have
\begin{align*}
	\norm{
		\frac{\partial{J(\pi_\theta)}}{\partial \theta} - \frac{\partial{\hat{J}(\pi_\theta)}}{\partial \theta}
		}_p
	\leq
	\frac{\gamma}{(1-\gamma)^2} \Qmax B_p
	\times
	\begin{cases}
		\cPG(\rho, \nu; \pi_\theta) \norm{\Delta \PKernel^{\pi_\theta}}_{1,1(\nu)}, \\
		2 \norm{\Delta \PKernel^{\pi_\theta}}_{1,\infty}.
	\end{cases}
\end{align*}
\end{theorem}
\ifSupp
\begin{proof}
\todo{Re-read it as it has recently been changed significantly. -AMF}
The policy gradient is 
\[
	\frac{\partial{J(\pi_\theta)}}{\partial \theta} = \frac{1}{1 - \gamma} \EEX{X \sim \rhoDiscounted(\cdot;\PKernelpi)}{f(X)}
\]
with the choice of $f(x) = f(x;\theta) = \EEX{A \sim \pi_\theta(\cdot|x)}{\nabla_\theta \log \pi_\theta(A|x) Q^{\pi_\theta}(x,A)}$ (and similar for $\frac{\partial{\hat{J}(\pi_\theta)}}{\partial \theta}$).
We can use Lemma~\ref{lem:PAML-ExpectedDiscountedError} to upper bound the difference between 
$\frac{\partial{J(\pi_\theta)}}{\partial \theta}$ and $\frac{\partial{\hat{J}(\pi_\theta)}}{\partial \theta}$.
To apply that lemma, we require to have upper bounds on the $\ell_2$ and $\ell_\infty$ norms of $f(x)$.
As $\norm{Q^{\pi_\theta}}_\infty \leq \Qmax$, for any $1 \leq p \leq \infty$ we have
\begin{align}
\label{eq:PAML-PolicyGradientError-Proof-Generic-Norm-of-f-Bound}
	\norm{f(x;\theta)}_p
	\leq
	\EEX{A \sim \pi_\theta(\cdot|x)} { \norm{ \nabla_\theta \log \pi_\theta(A|x) Q^{\pi_\theta}(x,A)  }_p }
	\leq
	\Qmax 
	\EEX{A \sim \pi_\theta(\cdot|x)} { \norm{ \nabla_\theta \log \pi_\theta(A|x)  }_p }.
\end{align}

By Lemma~\ref{lem:PAML-ExponentialFamilyPolicy-Boundedness}, we have
\begin{align*}
	\EEX{A \sim \pi_\theta(\cdot|x) }
	{\norm{\nabla_\theta \log \pi_\theta(A|x)}_p} \leq
	\begin{cases}
		B_2,	& p = 2 \\
		2 B_\infty. & p = \infty
	\end{cases}
\end{align*}%
Therefore by~\eqref{eq:PAML-PolicyGradientError-Proof-Generic-Norm-of-f-Bound},
we have the upper bounds
$\norm{f(x;\theta)}_2 \leq \Qmax B_2$
and
$\norm{f(x;\theta)}_\infty \leq 2 \Qmax B_\infty$.
This finishes the proof.
\end{proof}
\fi

This theorem shows the effect of the model error, quantified in the total variation-based norms 
$\norm{\Delta \PKernel^{\pi_\theta}}_{1,\infty}$ or 
$\norm{\Delta \PKernel^{\pi_\theta}}_{1,1(\nu)}$, on the PG estimate. 
The norms measure how different the distribution of the true dynamics $\PKernel^{\pi_\theta} = {\PKernelTrue}^{\pi_\theta}$ is from the distribution of the estimate $\PKernelhat^{\pi_\theta}$, according to the difference in the total variation distance between their next-state distributions, i.e., $\smallnorm{\PKernel^{\pi_\theta}(\cdot|x) - \PKernelhat^{\pi_\theta}(\cdot|x)}_1$.
The difference between them is on whether we take the supremum over the state space $\XX$ or average (according to $\nu$) over it.
Clearly, $\norm{\Delta \PKernel^{\pi_\theta}}_{1,\infty}$ is a more strict norm compared to $\norm{\Delta \PKernel^{\pi_\theta}}_{1,1(\nu)}$.

For the average norm, a concentrability coefficient $\cPG(\rho, \nu; \pi_\theta)$ appears in the bound. This coefficient measures how different the discounted future-state distribution $\rhoDiscounted^{\pi_\theta}$ is from the distribution $\nu$, used for taking average over the total variation errors.
If $\nu$ is selected to be $\rhoDiscounted^{\pi_\theta}$,  the coefficient $\cPG(\rho, \nu; \pi_\theta)$ would be equal to $1$.
Moreover, if we choose $\nu$ to be equal to the initial state distribution $\rho$, one can show that the coefficient $\cPG(\rho, \rho; \pi_\theta) \leq \frac{1}{1 - \gamma}$ (we show this in the proof of Theorem~\ref{thm:PAML-PolicyError-at-Stationarity}).

We can use this upper bound to relate the quality of an MLE to the quality of the PGs.
By Pinsker's inequality, the TV distance of two distributions can be upper bounded by their $\KL$-divergence:
\[
	\norm{\Delta \PKernel^{\pi_\theta}(\cdot|x)}_1
	\leq
	\sqrt{ 2 \KL \left( \PKernel^{\pi_\theta} (\cdot|x) || \PKernelhat_{\pi_\theta}(\cdot|x) \right)
	}.
\]
Therefore, we also have
$\norm{\Delta \PKernel^{\pi_\theta}}_{1,\infty} \leq \sqrt{ 2 \KL_{\infty}(\PKernel^{\pi_\theta} || \PKernelhat_{\pi_\theta})}$ too.
Moreover, as
\begin{align*}
	\EEX{\nu}{ \norm{ \Delta \PKernel^{\pi_\theta}(\cdot|X)}_1}^2
	\leq
	\EEX{\nu}{ \norm{ \Delta \PKernel^{\pi_\theta}(\cdot|X)}_1^2}
	& \leq
	\EEX{\nu}{ 2 \KL \left( \PKernel^{\pi_\theta} (\cdot|X) || \PKernelhat_{\pi_\theta}(\cdot|X) \right) }
	\\
	& =
	2 
	\KL_{1(\nu)} \left( \PKernel^{\pi_\theta}  || \PKernelhat_{\pi_\theta}  \right),
\end{align*}
we get
$
\norm{\Delta \PKernel^{\pi_\theta}}_{1,1(\nu)} \leq \sqrt{ 2 \KL_{1(\nu)}(\PKernel^{\pi_\theta} || \PKernelhat_{\pi_\theta})}
$.
Combined with the upper bound of Theorem~\ref{thm:PAML-PolicyGradientError}, we get that
\begin{align}
\label{eq:PAML-GradientError-KL}
	\norm{
		\nabla_\theta J(\pi_\theta)- \nabla_\theta \hat{J}(\pi_\theta)
		}_p
	\leq
	\frac{\gamma}{(1-\gamma)^2} \Qmax B_p
	\times
	\begin{cases}
		\cPG(\rho, \nu; \pi_\theta) \sqrt{ 2 \KL_{1(\nu)}(\PKernel^{\pi_\theta} || \PKernelhat_{\pi_\theta})}, \\
		2 \sqrt{ 2 \KL_{\infty}(\PKernel^{\pi_\theta} || \PKernelhat_{\pi_\theta})}.
	\end{cases}
\end{align}

This is an upper bound on the PG error for conventional model learning procedures.
Recall that the MLE is the minimizer of the $\KL$-divergence between the empirical distribution of samples generated from $\PKernel^{\pi_\theta}$ and $\PKernelhat^{\pi_\theta}$.
There would be some statistical deviation between its minimizer and the minimizer of 
$
\min_{\PKernel \in \MM} \KL(\PKernel^{\pi_\theta} || \PKernel )
$,
but if the model space is chosen properly, the difference between the minimizer decreases as the number of samples increases.

%


\todo{Revise!} 
This upper bound suggests why PAML might be a more suitable approach in learning a model. An MLE-based approach tries to minimize an upper bound of an upper bound for the quantity that we care about (PG error). This consecutive upper bounding might be quite loose. On the other hand, the population version of PAML's loss \eqref{eq:PAML-DifferenceInGradient} is exactly the error in the PG estimates that we care about.
A question that may arise is that although these two losses are different, are their minimizers the same? In Figures~\ref{fig:PAML-GMM-visualization} and~\ref{fig:PAML-GMM-contour} we show through a simple visualization that the minimizers of PAML and $\KL$ could indeed be different. 
\todo{More discussions. -AMF}

\begin{figure}[t]
    \centering
    \begin{subfigure}{1.0\linewidth}
     \centering
        \includegraphics[width=1.0\textwidth]{plots/annotated_reverse_forward_kl.pdf}
        \caption{Visualization of minimizing models for PAML and MLE. $\PKernelTrue$ is a Gaussian mixture model and the learned model is a single Gaussian. The loss minimized by PAML for this simple case is: $|\sum_x (\PKernelTrue - \PKernelhat)(x) f(x)|^2$.}
        \label{fig:PAML-GMM-visualization}%
    \end{subfigure} 
    \begin{subfigure}{1.0\linewidth}
        \centering
        \includegraphics[width=0.8\textwidth]{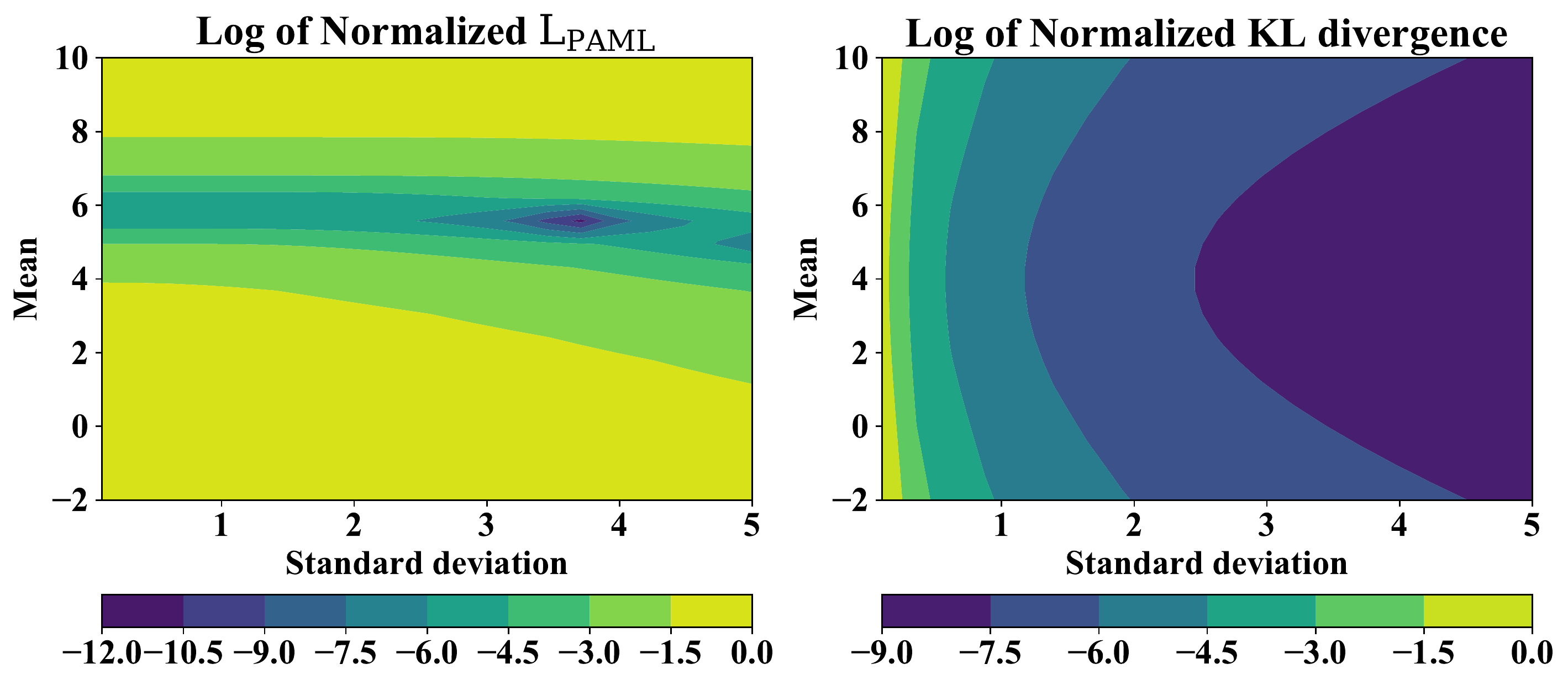}
        \caption{Contours of the two loss surfaces for case (c) above, demonstrating the locations of the minimizers for each. Note that the losses were log-normalized for better visual contrast in this figure.}
        \label{fig:PAML-GMM-contour}%
    \end{subfigure}
    \caption{Visualization of minimizers of PAML and KL for fitting a simple distribution.}
\end{figure}

\todo{We need to compare this result with Restelli's. -AMF}

\fi


\newcommand{\pitheta}{{\pi_\theta}}

\subsection{Convergence of Model-Based PG}
\label{sec:PAML-Theory-PG-Convergence}

\ifconsiderlater
\todo{Some possibly relevant papers. I haven't read them carefully.}

Fazel, Kakade, et al., ``Global Convergence of Policy Gradient Methods for the Linear Quadratic Regulator'', ICML, 2018

Jalaj Bhandari, and Daniel Russo, ``Global Optimality Guarantees For Policy Gradient Methods," 2019 June.
 
Lingxiao Wang, Qi Cai, Zhuoran Yang,  Zhaoran Wang, ``Neural Policy Gradient Methods: Global Optimality and Rates of Convergence,'', 2019 August 29.

Boyi Liu, Zhaoran Wang, et al., ``Neural Proximal/Trust Region Policy Optimization Attains Globally Optimal Policy,'' 2019 September.

Kaiqing Zhang, Alec Koppel, Hao Zhu Tamer Basar, ``Global Convergence of Policy Gradient Methods to (Almost) Locally Optimal Policies,'' 2019 September.

Shie Mannor and others: ``Adaptive Trust Region Policy Optimization: Global Convergence and Faster Rates for Regularized MDPs

XXX
\fi

We provide a convergence guarantee for a MBPG method.
The guarantee applies for the restricted policy space, and shows that the obtained policy is not much worse than the best policy in the class.
The error depends on the number of PG iterations, the error in the PG computation, and some properties of the MDP and the sampling distributions.
One factor that determines the PG error is the error in the model $\PKernelhat$.
Another factor is the error in the critic $\Qhat$.
We first consider the case when there is no error in the critic (Section~\ref{sec:PAML-Theory-PG-Convergence-Exact-Critic}).
We then let the critic have some errors too and analyze its effect (Section~\ref{sec:PAML-Theory-PG-Convergence-Inexact-Critic}).
Our focus here is the analysis of the PG as an optimization procedure. Even though we consider model and critic errors, we do not relate those errors to the number of interactions with the environment, the capacity and expressiveness of the model space and the value function space, i.e., the learning aspects of analyzing a complete model-based actor-critic algorithm.
Moreover, we suppose that given the model, the PG is calculated exactly, so there is no error in estimation of the gradient.

%

Even though our promised analysis seems somewhat restrictive, we would like to note that until very recently there had not been much theoretical work on the convergence of the PG algorithm, including the actor-critic variants, beyond proving its convergence to a local optimum~\citep{KondaTsitsiklis01,BaxterBartlett01,MarbachTsitsiklis2001,SuttonMcAleesterSinghMansour2000,BhatnagarSuttonGhavamzadehLee2009,TadicDoucet2017}.
Nevertheless, there has been a recent surge of interest in providing global convergence guarantees for PG methods and variants~\citep{AgarwalKakadeLeeMahajar2019,BhandariRusso2019,LiuCaiYangWang2019,WangCaiYangWang2019,ShaniEfroniMannor2020,XuWangLiang2020}.
This section is based on the recent work by~\citet{AgarwalKakadeLeeMahajar2019}, who have provided convergence results for several variations of the PG method. Their result is for a model-free setting, where the gradients are computed according to the true dynamics $\PKernelpi$ of the policy. We modify their result to show the convergence of MBPG. In addition to this difference, we introduce a new notion of policy approximation error, which is perhaps a better characterization of the approximation error of the policy space. We also explicitly consider the critic error in Section~\ref{sec:PAML-Theory-PG-Convergence-Inexact-Critic}.

Instead of extending~\citeauthor{AgarwalKakadeLeeMahajar2019}'s result to be suitable for the model-based setting, we provide a slightly, but crucially, different result for the convergence of a PG algorithm. In particular, we consider the same setting as in Section 6.2 (Projected Policy Gradient for Constrained Policy Classes) of~\citet{AgarwalKakadeLeeMahajar2019} and prove a result similar to their Theorem 6.11. We briefly mention that the main difference with their result is that our new notation of \emph{policy approximation error}, to be defined shortly, considers {\bf 1)} how well one can approximate the best policy in the policy class $\Pi$, instead of how well one can approximate the greedy policy w.r.t.~the action-value function of the current policy, in their result, and {\bf 2)} the interaction of the value function and the policy, as opposed to the error in only approximating the policy in their result.
We explain this in more detail \ifSupp \else in the same section of the supplementary material \fi after we describe all the relevant quantities.
This result, in turn, can be used to prove a convergence guarantee, as in their Corollary 6.14.
Before continuing, we mention that we liberally use the groundwork provided by~\citet{AgarwalKakadeLeeMahajar2019}.

We analyze a projected PG with the assumption that the PGs are computed exactly. We consider a setup where the performance is evaluated according to a distribution $\rho \in \bar{\MM}(\XX)$, but the PG is computed according to a possibly different distribution $\mu \in \bar{\MM}(\XX)$. To be concrete, let us consider a policy space $\Pi = \cset{\pi_\theta}{\theta \in \Theta}$ with $\Theta$ being a convex subset of $\Real^d$ and $\ProjTheta$ be the projection operator onto $\Theta$.
%
Consider the projected policy gradient procedure
\begin{align*}
	\theta_{t+1} \leftarrow \ProjTheta \left [ \theta_t + \eta \nabla_\theta J_\mu(\pi_{\theta_t}) \right ],
\end{align*}
with a learning rate $\eta > 0$, to be specified.
%

A policy $\pi_\theta$ is called $\eps$-stationary if for all $\theta + \delta \in \Theta$ and with the constraint that $\norm{\delta}_2 \leq 1$, we have
\begin{align}
\label{eq:PAML-Stationarity}
	\delta^\top \nabla_\theta J_\mu(\pi_\theta) \leq \eps.
\end{align}
Let us denote the best policy in the policy class $\Pi$ according to the initial distribution $\rho$ by $\piBestrho$ (or simply $\piBest$, if it is clear from the context), i.e.,
\begin{align}
\label{eq:PAML-piBest}
	\piBest \leftarrow \argmax_{\pi \in \Pi} J_\rho(\pi).
\end{align}
%
%
We define a function called \emph{Policy Approximation Error (PAE)}. Given a policy parameter $\theta$ and $w \in \Real^d$, and for a probability distribution $\nu \in \bar{\MM}(\XX)$, it is defined as
\begin{align*}
	\LossPolicyError(\theta, w; \nu) \eqdef 
	\EEX{X \sim \nu}
		{\left|
			\sum_{a \in \Actions} \left( \piBest(a|X) - \pi_\theta(a|X) - w^\top \nabla_\theta \pi_\theta(a|X) \right) Q^{\pi_\theta}(X,a) 
		\right|}.
\end{align*}

This can be roughly interpreted as the error in approximating the improvement in the value from the current policy $\pi_\theta$ to the best policy in the class, $\piBest$, i.e.,~$\sum_{a \in \Actions} ( \piBest(a|X) - \pi_\theta(a|X) ) Q^{\pi_\theta}(X,a)$, by a linear model 
$\sum_{a \in \Actions} w^\top \nabla_\theta \pi_\theta(a|X) Q^{\pi_\theta}(X,a) = 
w^\top \EEX{A \sim \pi_\theta(\cdot|X)}{\nabla_\theta \log \pi_\theta(a|X) Q^{\pi_\theta}(X,a) }$.
\todo{Maybe revise this interpretation. Is this clear enough? And this is ``roughly'' true because the improvement interpretation only holds when $\nu = \rhoDiscounted^\piBest$. -AMF}

For any $\theta \in \Theta$, we can define the best $w^*(\theta) = w^*(\theta;\nu)$ that minimizes $\LossPolicyError(\theta, w; \nu)$ as
\begin{align}
\label{eq:PAML-Loss-PAE-wstar}
	w^*(\theta;\nu) \leftarrow \argmin_{w + \theta \in \Theta} \LossPolicyError (\theta, w; \nu).
\end{align}
We use $\LossPolicyError(\theta;\nu)$ to represent $\LossPolicyError(\theta, w^*(\theta); \nu)$. We may drop the distribution $\nu$ whenever it is clear from the context.


The following result relates the performance loss of a policy compared to the best policy in the class (i.e., $J_\rho(\piBest) - J_\rho(\pi_\theta)$) to its $\eps$-stationarity, the policy approximation error, and some other quantities.
As we shall see, one can show the $\eps$-stationarity of projected PG using tools from the optimization literature (e.g., Theorem 10.15 of~\citealt{Beck2017}, quoted with slight modification as Lemma~\ref{lem:PAML-Beck-Expectation} in Appendix~\ref{sec:PAML-Theory-Background}), hence providing a performance guarantee.

\begin{theorem}
\label{thm:PAML-PolicyError-at-Stationarity}
Consider any initial distributions $\rho, \mu \in \bar{\MM}(\XX)$ and a policy $\pi_\theta$ with $\theta \in \Theta$, a convex set.
Suppose that $\pi_\theta$ is an $\eps$-stationary w.r.t. distribution $\mu$~\eqref{eq:PAML-Stationarity}.
Let $\piBest$ be defined as~\eqref{eq:PAML-piBest}
and
$w^*(\theta;\rhoDiscounted^\piBest)$
as~\eqref{eq:PAML-Loss-PAE-wstar}.
Assume that $\rhoDiscounted^\piBest$ is absolutely continuous w.r.t. $\mu$, and $0 \leq \gamma < 1$.
We then have
\begin{align*}
	J_\rho(\piBest) - J_\rho(\pi_\theta)
	\leq
	\frac{1}{1 - \gamma}
	\left[
		\LossPolicyError(\theta; \rhoDiscounted^\piBest) +
		\norm{	\frac	{\mathrm{d}\rhoDiscounted^\piBest}
					{\mathrm{d} \mu }
			}_\infty
	(1 \vee \norm{w^*(\theta;\rhoDiscounted^\piBest)} ) \eps			
	\right].
\end{align*}
\end{theorem}
\begin{proof}
By the performance difference lemma (Lemma 6.1 of~\citet{KakadeLangfordCPI2002} or Lemma 3.2 of~\citealt{AgarwalKakadeLeeMahajar2019}) for any policy $\pi_\theta$ and the best policy in class $\piBest = \piBestrho$, we have that \todo{Maybe quote the performance difference lemma in an appendix?}
\begin{align*}
	J_\rho(\piBest) - J_\rho(\pi_\theta)
	& =
	\frac{1}{1 - \gamma}
	\EEX{X \sim \rhoDiscounted^\piBest}
		{\sum_{a \in \Actions} \piBest(a|X) A^{\pi_\theta}(X,a) }
	\\
	& \stackrel{(a)}{=}
	\frac{1}{1 - \gamma}
	\EEX{X \sim \rhoDiscounted^\piBest}
		{\sum_{a \in \Actions} \left( \piBest(a|X) - \pi_\theta(a|X) \right) A^{\pi_\theta}(X,a) }
	\\
	& =
	\frac{1}{1 - \gamma}
	\EEX{X \sim \rhoDiscounted^\piBest}
		{\sum_{a \in \Actions} \left( \piBest(a|X) - \pi_\theta(a|X) \right) \left(Q^{\pi_\theta}(X,a) - V^{\pi_\theta}(X) \right) }	\\
	& \stackrel{(b)}{=}
	\frac{1}{1 - \gamma}
	\EEX{X \sim \rhoDiscounted^\piBest}
		{\sum_{a \in \Actions} \left( \piBest(a|X) - \pi_\theta(a|X) \right) Q^{\pi_\theta}(X,a) },
\end{align*}
where (a) is because $\sum_{a} \pi_\theta(a|x) A^{\pi_\theta}(x,a) = \sum_{a} \pi_\theta(a|x) ( Q^{\pi_\theta}(x,a) - V^{\pi_\theta}(x) ) = 0$ by the definition of the state-value function,
and (b) is because $\sum_a ( \piBest(a|x) - \pi_\theta(a|x) ) V^{\pi_\theta}(x) = V^{\pi_\theta}(x) \sum_a \piBest(a|x) - \pi_\theta(a|x)   = V^{\pi_\theta}(x) (1 - 1)  = 0$.

Let $w \in \Real^d$ be an arbitrary vector. By adding and subtracting the scalar $w^\top \nabla_\theta \pi_\theta Q^{\pi_\theta}(X,a)$, we obtain
\begin{align*}
	J_\rho(\piBest) - J_\rho(\pi_\theta)
	= &
	\frac{1}{1 - \gamma}
	\EEX{X \sim \rhoDiscounted^\piBest}
		{\sum_{a \in \Actions} \left( \piBest(a|X) - \pi_\theta(a|X) - w^\top \nabla_\theta \pi_\theta \right) Q^{\pi_\theta}(X,a) }
	+
	\\
	&
	\frac{1}{1 - \gamma}
	\EEX{X \sim \rhoDiscounted^\piBest}
		{\sum_{a \in \Actions} w^\top \nabla_\theta \pi_\theta(a|X) Q^{\pi_\theta}(X,a) }
	\\
	\leq &
	\frac{1}{1 - \gamma} \LossPolicyError(\theta, w; \rhoDiscounted^\piBest)
	+
	w^\top 
	\frac{1}{1 - \gamma}
	\EEX{X \sim \rhoDiscounted^\piBest}
		{\sum_{a \in \Actions} \pi_\theta(a|X) \nabla_\theta \log \pi_\theta(a|X) Q^{\pi_\theta}(X,a) }.
\end{align*}

We make two observations.
The first is that as this inequality holds for any $w$, it holds for $w^*(\theta)$ too, so we can substitute $\LossPolicyError(\theta, w; \rhoDiscounted^\piBest)$ with $\LossPolicyError(\theta; \rhoDiscounted^\piBest) = \LossPolicyError(\theta, w^*(\theta); \rhoDiscounted^\piBest)$.
The second is that the expectation 
$
	\EEX{X \sim \rhoDiscounted^\piBest}
		{\sum_{a \in \Actions} \pi_\theta(a|X) \nabla_\theta \log \pi_\theta(a|X) Q^{\pi_\theta}(X,a) }	
$
is of the same general form of a policy gradient $\nabla_\theta J(\pi_\theta)$, with the difference that the state distribution is w.r.t. the discounted future-state distribution of starting from $\rho$ and following the best policy in class $\piBest$, 
as opposed to the discounted future-state distribution of starting from $\mu$ and following policy $\pi_\theta$, cf.~\eqref{eq:PAML-J-and-Jhat}.
We use a change of measure argument, similar to~\eqref{eq:PAML-Change-of-Measure-Argument}, to convert the expectation to the desired form.
Based on these two observations, we obtain
\begin{align}
\label{eq:PAML-PolicyError-at-Stationarity-Theorem-Proof-PerfDiffUpperBound-1}
	J_\rho(\piBest) - J_\rho(\pi_\theta)
	\leq
	\frac{1}{1 - \gamma} \LossPolicyError(\theta; \rhoDiscounted^\piBest) +
	\norm{	\frac	{\mathrm{d}\rhoDiscounted^\piBest}
				{\mathrm{d} \muDiscounted^{\pi_\theta} }
		}_\infty
	w^*(\theta)^\top \nabla_\theta J_{\mu } (\pi_\theta).
\end{align}

We would like to use the $\eps$-stationary of the policy in order to upper bound the right-hand side (RHS).
Define $\delta = \frac{w^*(\theta)}{1 \vee \norm{w^*(\theta)} }$. It is clear that $\norm{\delta} \leq 1$.
As both $\theta$ and $\theta + w^*(\theta)$ belong to the set $\Theta$ and $\Theta$ is convex, the line segment connecting them is within $\Theta$ too. The point $\theta + \delta$ is on that line segment, so it is within $\Theta$.
By the $\eps$-stationarity, we obtain that
\[
	w^*(\theta) \nabla_\theta J_\mu(\pi_\theta)
	=
	(1 \vee \norm{w^*(\theta)} ) \delta^\top \nabla_\theta J_\mu(\pi_\theta)
	\leq
	(1 \vee \norm{w^*(\theta)} ) \eps.
\]
We plug-in this result in~\eqref{eq:PAML-PolicyError-at-Stationarity-Theorem-Proof-PerfDiffUpperBound-1} to get
\begin{align*}
	J_\rho(\piBest) - J_\rho(\pi_\theta)
	& \leq
	\frac{1}{1 - \gamma} \LossPolicyError(\theta; \rhoDiscounted^\piBest) +
	\norm{	\frac	{\mathrm{d}\rhoDiscounted^\piBest}
				{\mathrm{d} \muDiscounted^{\pi_\theta} }
		}_\infty
	(1 \vee \norm{w^*(\theta)} ) \eps
	\\
	&
	\leq
	\frac{1}{1 - \gamma}
	\left[
		\LossPolicyError(\theta; \rhoDiscounted^\piBest) +
		\norm{	\frac	{\mathrm{d}\rhoDiscounted^\piBest}
					{\mathrm{d} \mu }
			}_\infty
	(1 \vee \norm{w^*(\theta)} ) \eps			
	\right].
\end{align*}
The second inequality is because of the property of the Radon-Nikodym derivative that states that
if $\rhoDiscounted^\piBest \ll \mu \ll \muDiscounted^{\pi_\theta}$, we have
\[
	\frac{\mathrm{d} \rhoDiscounted^\piBest }{\mathrm{d}  \muDiscounted^{\pi_\theta} }
	= 
	\frac{\mathrm{d} \rhoDiscounted^\piBest }{\dmu }
	\frac{\dmu }{\mathrm{d}  \muDiscounted^{\pi_\theta} },
\]
and the fact that 
$\smallnorm{\frac{\dmu }{\mathrm{d}  \muDiscounted^{\pi_\theta} } }_\infty \leq \frac{1}{1 - \gamma}$.
To see the truth of the latter claim, consider any measurable set $\XX_0 \subset \XX$ and any policy $\pi$.
The probability of $\XX_0$ according to $\muDiscounted^\pi$ is greater or equal to $(1-\gamma)$ times of its probability according to $\mu$, that is,  
$\muDiscounted^\pi(\XX_0) = (1-\gamma) [ \mu(\XX_0) + (\mu \PKernelpi)(\XX_0) + (\mu (\PKernelpi)^2 (\XX_0) + \cdots] \geq (1 - \gamma)\mu(\XX_0)$.

To verify the conditions $\rhoDiscounted^\piBest \ll \mu \ll \muDiscounted^{\pi_\theta}$, notice that 
the condition
$\rhoDiscounted^\piBest \ll \mu$ is satisfied by assumption;
the condition 
$\mu \ll \muDiscounted^{\pi_\theta}$ is satisfied as we just show that $\muDiscounted^\pi(\XX_0) \geq (1 - \gamma)\mu(\XX_0)$ for any policy;
and 
if $\rhoDiscounted^\piBest \ll \mu$ is satisfied, $\rhoDiscounted^\piBest \ll \muDiscounted^{\pi_\theta}$ is satisfied too.
\end{proof}


This result is similar to Theorem 6.11 of~\citet{AgarwalKakadeLeeMahajar2019} with one small, but perhaps important difference.
The difference is in the policy approximation error term. Instead of $\LossPolicyError(\theta, w; \nu)$, they have a term called \emph{Bellman Policy Error}, which is defined as
\begin{align*}
	L_\text{BPE}(\theta, w; \nu) \eqdef 
	\EEX{X \sim \nu}
		{
			\sum_{a \in \Actions} 
			\left|
			\argmax_{a \in \AA} Q^{\pi_\theta}(X,a) - \pi_\theta(a|X) - w^\top \nabla_\theta \pi_\theta(a|X)
			\right| 
		}.
\end{align*}
The minimizer of this function over $w$, that is $L_\text{BPE}(\theta, \rhoDiscounted^\piBest) = \min_{w + \theta \in \Theta} L_\text{BPE}(\theta, w; \rhoDiscounted^\piBest)$, appears instead of $\LossPolicyError(\theta; \rhoDiscounted^\piBest)$ in the upper bound of Theorem~\ref{thm:PAML-PolicyError-at-Stationarity}.
The Bellman Policy Error measures the error in approximating the 1-step greedy policy improvement relative to a policy in the class.

Both BPE and PAE are equal to zero for a finite state and action space with a direct parametrization of the policy, i.e., $\pi_\theta(a|x) = \theta_{x,a}$ for $\theta \in \Real^{\XA}$ with appropriate constraints on $\theta$ to make $\pi_\theta(\cdot|x)$ a valid probability distribution. So both definitions pass the sanity check that they are not showing a non-zero value for policy approximation error when it should be zero (for a particular class of MDPs).
To see this concretely, note that for the direct parametrization, $\frac{\partial \pi_\theta(a|x)}{\partial \theta_{x',a'}} = 1$ when $(x,a) = (x',a')$, and $0$ otherwise. So if we choose ${w^*_\text{BPE}}(x,a) = \argmax_{a' \in \AA} Q^{\pi_\theta}(x,a) - \pi_\theta(a|x)$, the BPE loss $L_\text{BPE}(\theta, w^*_\text{BPE}; \nu) = 0$ (Section 6.2 of~\citet{AgarwalKakadeLeeMahajar2019}).
Likewise, if we choose 
${w^*_\text{PAE}}(x,a) = \piBest(a|x) - \pi_\theta(a|x)$, we get that the PAE loss
$L_\text{PAE}(\theta, w^*_\text{PAE}; \nu) = 0$.
\todo{This is a new paragraph. Come back to this and re-read. -AMF}

It is curious to know which of $L_\text{BPE}$ and $\LossPolicyError$ is a better characterizer of the policy approximation error. We do not have a definite answer to this question so far, as the properties of neither of them are well-understood yet, but we make two observations that show that $\LossPolicyError$ is better (smaller) at least in some circumstances.

The first observation is that $L_\text{BPE}$ ignores the value function $\Qpi$ and its interaction with the policy error, whereas 
$\LossPolicyError$ does not. 
As an example, if the reward function is constant everywhere, the action-value function  $\Qpi$ for any policy would be constant too. In this case, $\LossPolicyError(\theta;\nu)$ is zero (simply choose $w = 0$ as the minimizer), but $L_\text{BPE}$ may not be.

Weighting the error in policies with a value function is reminiscent of the loss function appearing in some classification-based approximate policy iteration methods such as the work by
\citet{LazaricGhavamzadehMunosDPI2010,FarahmandCAPI2015,LazaricGhavamzadehMunos2016} (and different from
the original formulation by \citet{LagoudakisParrICML2003} and more recent instantiation by \citet{SilverSchrittwieserSimonyanetal2017} whose policy loss does not incorporate the value functions), 
Policy Search by Dynamic Programming \citep{BagnellKakadeNgSchneider2003}, and Conservative Policy Iteration~\citep{KakadeLangfordCPI2002}.
\todo{Any paper that is missed? -AMF}

The other observation is that if the policy is parameterized such that there is only one policy in the policy class  $\Pi = \cset{\pi_\theta}{\theta \in \Theta}$ (but still $\Theta$ is a subset of $\Real^d$, so we can define PG), any policy $\pi_\theta \in \Pi$ is the same as the best policy $\piBest$, i.e., $\pi_\theta = \piBest$.
In that case,
\[
	\LossPolicyError(\theta, 0;\nu) = \EEX{X \sim \nu}
		{\left|
			\sum_{a \in \Actions} \left( \piBest(a|X) - \piBest(a|X) - 0^\top \nabla_\theta \pi_\theta(a|X) \right) Q^{\pi_\theta}(X,a) \right|} = 0.
\]
On the other hand, it may not be possible to make 
\[
L_\text{BPE}(\theta, w; \nu) = 
	\EEX{X \sim \nu}
		{
			\sum_{a \in \Actions} 
			\left|
			\argmax_{a \in \AA} Q^{\piBest}(X,a) - \piBest(a|X) - w^\top \nabla_\theta \pi_\theta(a|X)
			\right|
		}
\]
equal to zero for any choice of $w$, as it requires the policy space to approximate the greedy policy, which is possibly outside the policy space.
We leave further study of these two policy approximation errors to a future work.


To provide a convergence rate, we require some extra assumptions on the smoothness of the policy.

\begin{assumption}[(Assumption 6.12 of~\citet{AgarwalKakadeLeeMahajar2019})]
\label{asm:PolicySmoothness}
Assume that there exist finite constants $\beta_1, \beta_2 \geq 0$ such that for all $\theta_1, \theta_2 \in \Theta$, and for all $(x,a) \in \XA$, we have
\begin{align*}
	& | \pi_{\theta_1}(a|x) - \pi_{\theta_2}(a|x) | \leq \beta_1 \norm{\theta_1 - \theta_2}_2,
	\\
	&
	\norm{ \nabla_\theta \pi_{\theta_1}(a|x) - \nabla_\theta \pi_{\theta_1}(a|x)  }_2
	\leq \beta_2 \norm{\theta_1 - \theta_2}_2.
\end{align*}
\end{assumption}

As an example, this assumption holds for the exponential family~\eqref{eq:PAML-PolicyParametrization}
with bounded features $\norm{\phi(a|x)}_2 \leq B$.
In that case,  $\beta_1 = 2B$ and $\beta_2 = 6B^2$ (Lemma~\ref{lem:PAML-ExponentialFamilyPolicy-Smoothness} in Appendix~\ref{sec:PAML-Theory-Auxiliary}).


\citet{AgarwalKakadeLeeMahajar2019} assume that the reward function is in $[0,1]$. Here we consider the reward to be $\Rmax$-bounded, which leads to having the value function being $\Qmax$-bounded with $\Qmax = \frac{\Rmax}{1-\gamma}$.
Some results should be slightly modified (particularly, Lemma E.2 and E.5 of that paper). We report the modifications in Appendix~\ref{sec:PAML-Theory-Auxiliary}. Here we just mention that the difference is that the upper bounds in those result should be multiplied by $(1 -  \gamma) \Qmax$.

\subsection{Exact Critic}
\label{sec:PAML-Theory-PG-Convergence-Exact-Critic}
We are ready to analyze the convergence behaviour of a model-based PG algorithm with exact critic.
We consider a projected PG algorithm that uses the model $\PKernelhat^{\pi_{\theta_k}}$ to compute the gradient, i.e., 
\begin{align}
\label{eq:PAML-ProjectedPGUpdate-WithModel}
	\theta_{t+1} \leftarrow \ProjTheta \left [ \theta_t + \eta \nabla_\theta \hat{J}_\mu (\theta_k) \right ],
\end{align}
with a learning rate $\eta > 0$ to be specified.
The following theorem is the main result of this section.

%

\begin{theorem}
\label{thm:PAML-MBPG-Convergence}
Consider any initial distributions $\rho, \mu \in \bar{\MM}(\XX)$ and a policy space $\Pi$ parameterized by $\theta \in \Theta$ with $\Theta$ being a convex subset of $\Real^d$.
Assume that all policies $\pi_\theta \in \Pi$ satisfy Assumption~\ref{asm:PolicySmoothness}.
Furthermore, suppose that the value function is bounded by $\Qmax$, the MDP has a finite number of actions $\actionnum$, and $0 \leq \gamma < 1$.
Let
\begin{align}
\label{eq:PAML-MBPG-Convergence-Theorem-beta}
	\beta =
	\Qmax
	\left[
		\frac{2 \gamma \beta_1^2 \actionnum^2 }{(1-\gamma)^2} + 
		\frac{\beta_2 \actionnum }{1 - \gamma}
	\right].
\end{align}
Let $T$ be an integer number. Starting from a $\pi_{\theta_0} \in \Pi$, consider the sequence of policies $\pi_{\theta_1}, \dotsc, \pi_{\theta_T}$ generated by the projected model-based PG algorithm~\eqref{eq:PAML-ProjectedPGUpdate-WithModel} with step-size $\eta = \frac{1}{\beta}$.
Let $W = \sup_{\theta \in \Theta} \norm{w^*(\theta;\rhoDiscounted^\piBest)}_2$, and assume that $W < \infty$.
Assume that for any policy $\pi_\theta \in \{ \pi_{\theta_0}, \dotsc, \pi_{\theta_{T-1}} \}$, there exist constants 
$\eps_{\text{PAE}}$ and $\eps_{\text{model}}$ 
such that
\begin{align*}
	& \LossPolicyError(\theta; \rhoDiscounted^\piBest) \leq \eps_{\text{PAE}},
		& \text{(policy approximation error)}
	\\
	&
	\norm{ \nabla_\theta J_\mu(\pi_\theta) - \nabla_\theta \hat{J}_\mu(\pi_\theta) }_2 \leq \eps_{\text{model}}.
		& \text{(model error)}
\end{align*}
We then have
\begin{align*}
	\EEX{t \sim \text{Unif}(1,\dotsc, T)}{
	J_\rho(\piBest) - J_\rho(\pi_{\theta_t})}
	\leq
	\frac{1}{1 - \gamma}
	\left[
		\eps_{\text{PAE}} +
		\norm{	\frac	{\mathrm{d}\rhoDiscounted^\piBest}
					{\mathrm{d} \mu }
			}_\infty
	(1 \vee W )
	\left(
		4 \sqrt{ \frac{\Qmax \beta}{T} } +
		\eps_\text{model}
	\right)
	\right].	
\end{align*}
\end{theorem}
\begin{proof}
By Lemma~\ref{lem:PAML-PolicySmoothness} in Appendix~\ref{sec:PAML-Theory-Auxiliary}, $V^{\pi_\theta}$ is $\beta$-smooth for all states $x$ with $\beta$ specified in~\eqref{eq:PAML-MBPG-Convergence-Theorem-beta}.
Hence $\hat{J}_\mu(\pi_\theta)$ is also $\beta$-smooth.

Let the gradient mapping for $\theta$ be defined as
\[
	G^\eta(\theta) = \frac{1}{\eta} \left( \ProjTheta \left [ \theta + \eta \nabla_\theta \hat{J}_\mu \right ] - \theta \right).
\]
For a projected gradient ascent on a $\beta$-smooth function over a convex set with a step-size of $\eta = \frac{1}{\beta}$, 
Lemma~\ref{lem:PAML-Beck-Expectation}, which is a slight modification of Theorem 10.15 of~\citet{Beck2017}, shows that
\begin{align}
\label{eq:PAML-MBPGConvergence-Theorem-Proof-GradientMappingUpperBound}
	\EEX{t \sim \text{Unif}(0, \dotsc, T-1)}{
	\norm{G^\eta(\theta_t)}_2}
	\leq
	\sqrt{
		\frac	{2 \beta \left( \max_{\pi \in \Pi} \hat{J}_\mu(\pi) - \hat{J}_\mu(\pi_{\theta_0}) \right) }
			{T}
	}
	\leq
	\sqrt{
		\frac	{4 \Qmax \beta  }
		{T}.
	}
\end{align}

By Proposition~\ref{prop:PAML-SmallGradientMapping-to-epsStationarity} in Appendix~\ref{sec:PAML-Theory-Auxiliary} (originally Proposition D.1 of~\citealt{AgarwalKakadeLeeMahajar2019}), if we let $\theta' = \theta + \eta G^\eta$, we have that
\begin{align*}
	\max_{\theta + \delta \in \Theta, \norm{\delta}_2 \leq 1}
	\delta^\top \nabla_\theta \hat{J}_\mu(\pi_{\theta'})
	\leq
	(\eta \beta + 1) \norm{G^\eta(\theta)}_2.
\end{align*}
This upper bound along~\eqref{eq:PAML-MBPGConvergence-Theorem-Proof-GradientMappingUpperBound} and $\beta \eta + 1 = 2$ show that the sequence $\theta_0, \theta_1, \dotsc, \theta_T$ generated by~\eqref{eq:PAML-ProjectedPGUpdate-WithModel} satisfies
\begin{align}
\label{eq:PAML-MBPGConvergence-Theorem-Proof-epsStationarity-forModel-UpperBound}
	\frac{1}{T} \sum_{t = 1}^T
	\max_{\theta_t + \delta \in \Theta, \norm{\delta}_2 \leq 1}
	\delta^\top \nabla_\theta \hat{J}_\mu(\pi_{\theta_t})
	\leq
	(\eta \beta + 1)
	\frac{1}{T}
	\sum_{t=0}^{T-1}
	\norm{G^\eta(\theta_t)}_2
	\leq
	4 \sqrt{ \frac{\Qmax \beta}{T} }.
\end{align}


Note that this is a guarantee on $\delta^\top \nabla_\theta \hat{J}_\mu(\pi_{\theta_t})$, the inner product of a direction $\delta$ with the PG according to $\PKernelhat^{\pi_{\theta_t}}$, and not on
$\delta^\top \nabla_\theta J_\mu(\pi_{\theta_t})$, which has the PG according to $\PKernel^{\pi_{\theta_t}}$ and is what we need in order to compare the performance.
We can relate them, however.
For any $\theta$, including $\theta_0, \theta_1, \dotsc, \theta_{T}$, we have
\begin{align}
\label{eq:PAML-MBPGConvergence-Theorem-Proof-ModelBasedDirectionalGradient-To-ModelFree}
\nonumber
	\max_{ \norm{\delta}_2 \leq 1}
	\left|
		\delta^\top \nabla_\theta J_\mu(\pi_\theta)
	\right|
	& =
	\max_{ \norm{\delta}_2 \leq 1}
	\left|
		\delta^\top \left( \nabla_\theta J_\mu(\pi_\theta) 
		- \nabla_\theta \hat{J}_\mu(\pi_\theta) + \nabla_\theta \hat{J}_\mu(\pi_\theta) \right)
	\right|
	\\
\nonumber	
	& \leq
	\max_{ \norm{\delta}_2 \leq 1}
	\left|
		\delta^\top \nabla_\theta \hat{J}_\mu(\pi_\theta)
	\right|
	+
	\max_{ \norm{\delta}_2 \leq 1}
	\left|
		\delta^\top \left( \nabla_\theta J_\mu(\pi_\theta) 
		- \nabla_\theta \hat{J}_\mu(\pi_\theta) \right)
	\right|
	\\
	& =
	\max_{ \norm{\delta}_2 \leq 1}
	\left|
		\delta^\top \nabla_\theta \hat{J}_\mu(\pi_\theta)
	\right|
	+
	\norm{ \nabla_\theta J_\mu(\pi_\theta) - \nabla_\theta \hat{J}_\mu(\pi_\theta) }_2.
\end{align}

This inequality together with~\eqref{eq:PAML-MBPGConvergence-Theorem-Proof-epsStationarity-forModel-UpperBound} and the assumption on the model error provide an upper bound on the average of $\eps$-stationarities:
\begin{align}
\label{eq:PAML-MBPGConvergence-Theorem-Proof-AverageStationarity}
	\frac{1}{T} \sum_{t=1}^T
	\max_{ \norm{\delta}_2 \leq 1}
		\left|
			\delta^\top \nabla_\theta J_\mu(\pi_{\theta_t})
		\right|
		\leq \eps_\text{model} + 
		4 \sqrt{ \frac{\Qmax \beta}{T} }
\end{align}

%

We can now evoke Theorem~\ref{thm:PAML-PolicyError-at-Stationarity} for each $\pi_{\theta_t}$ and take a summation over both sides of the inequality.
Suppose that $\pi_t$ is $\eps_t$-stationary, i.e., $\delta^\top \nabla_\theta J_\mu(\pi_{\theta_t}) \leq \eps_t$  for any valid $\delta$ (cf.~\eqref{eq:PAML-Stationarity}).
Also recall that $W = \sup_{\theta \in \Theta} \norm{w^*(\theta;\rhoDiscounted^\piBest)}_2$.
So we get
%
%
%
%
\begin{align*}
	\frac{1}{T} \sum_{t=1}^T J_\rho(\piBest) - J_\rho(\pi_{\theta_t} )
	& \leq
	\frac{1}{1 - \gamma}
	\frac{1}{T} \sum_{t=1}^T
	\left[
		\LossPolicyError(\theta_t; \rhoDiscounted^\piBest) +
		\norm{	\frac	{\mathrm{d}\rhoDiscounted^\piBest}
					{\mathrm{d} \mu }
			}_\infty
	(1 \vee \norm{w^*(\theta_t;\rhoDiscounted^\piBest)} ) \eps_t
	\right]
	\\
	& \leq
	\frac{1}{1 - \gamma}
	\left[
		\eps_{\text{PAE}} +
		\norm{	\frac	{\mathrm{d}\rhoDiscounted^\piBest}
					{\mathrm{d} \mu }
			}_\infty
	(1 \vee W )
	\frac{1}{T} \sum_{t=1}^T \eps_t
	\right],
	\\
	& \leq
	\frac{1}{1 - \gamma}
	\left[
		\eps_{\text{PAE}} +
		\norm{	\frac	{\mathrm{d}\rhoDiscounted^\piBest}
					{\mathrm{d} \mu }
			}_\infty
	(1 \vee W )
	\left(
		4 \sqrt{ \frac{\Qmax \beta}{T} } +
		\eps_\text{model}
	\right)
	\right],
\end{align*}
where we used~\eqref{eq:PAML-MBPGConvergence-Theorem-Proof-AverageStationarity} in the last inequality.
This is the desired result.
\end{proof}

This result shows the effect of the policy approximation error $\eps_\text{PAE}$, the model error $\eps_\text{model}$, and the number of iterations $T$.
We observe that the error due to optimization decreases as $O(\frac{1}{\sqrt{T}})$.

The policy approximation error is similar to the function approximation term (or bias) in supervised learning, and depends on how expressive the policy space is. This term may not go to zero, which means that the projected PG method may not find the best policy in the class, even if $T \ra \infty$. This means that the convergence would not be to the global optimum within the policy class.
What we know about the properties of Policy Approximation Error of this work or Bellman Policy Error of~\citet{AgarwalKakadeLeeMahajar2019} are rather limited at the moment, and studying them is an interesting future research direction.
What we know so far, however, is that for finite state and action spaces with direct parameterization of the policy, both PAE and BPE are zero, as expected. This is discussed after Theorem~\ref{thm:PAML-PolicyError-at-Stationarity}.
We also know that there are certain situations where PAE is zero, but PBE is not, suggesting that PAE might be a better way to quantify the policy approximation error.\footnote{Though it might be possible to find examples where PBE is zero, but PAE is not; we are not aware of such an example.}

As mentioned after Theorem~\ref{thm:PAML-PolicyError-at-Stationarity}, 
The model error $\eps_\text{model}$ captures how well one can replace the PG computed according to the true dynamics $\PKernelpi$ with the learned dynamics $\PKernelpihat$, i.e., 
\[
	\norm{ \nabla_\theta J_\mu(\pi_\theta) - \nabla_\theta \hat{J}_\mu(\pi_\theta) }_2 \leq \eps_\text{model}.
\]
This is the error in the PG estimation between following the true model and the estimated model.
If this error is small, the effect of using the model on the policies obtained from this MBPG procedure is small too.
This norm is exactly what PAML tries to minimize (through its empirical version).
Similar to the discussion after Theorem~\ref{thm:PAML-PolicyGradientError}, this suggests that PAML's objective is more relevant for having a good MBPG method than a conventional model learning method that is based on MLE or similar criteria.
The magnitude of this error depends on how expressive the model class is, the number of samples used in minimizing the loss, etc.

The distribution mismatch between the discounted future-state distribution of following $\piBest$ with an initial state distribution of $\rho$ and the initial distribution $\mu$, used for the computation of PG shows itself in 
the Radon-Nikodym derivative 
$\smallnorm{	\frac	{\mathrm{d}\rhoDiscounted^\piBest}
					{\mathrm{d} \mu }
			}_\infty$.
\todo{Maybe more discussion for this? -AMF}

\todo{minor: this paragraph seems redundant after the discussion above -RA}
This result can be compared to Corollary 6.14 of~\citet{AgarwalKakadeLeeMahajar2019}, whose proof we followed closely.
There are several differences that are noteworthy.
The major difference is that this result provides a guarantee for MBPG, whereas \citeauthor{AgarwalKakadeLeeMahajar2019}'s result is for model-free PG.
The other difference is that we have the policy approximation error $\eps_{\text{PAE}}$, instead of the Bellman Policy Error of~\citet{AgarwalKakadeLeeMahajar2019}. This is due to using Theorem~\ref{thm:PAML-PolicyError-at-Stationarity}, which we already have discussed.
The other difference is that the guarantee of this theorem is for the average over iterations of the performance loss 
$\EEX{t \sim \text{Unif}(1,\dotsc, T)}{J_\rho(\piBest) - J_\rho(\pi_{\theta_t})}$ instead of the minimum over iterations of the performance loss 
$\min_{t < T} J_\rho(\piBest) - J_\rho(\pi_{\theta_t})$, as in~\citet{AgarwalKakadeLeeMahajar2019}.
This is a minor difference, and the current result holds for the minimum over $t$ as well.

We use this theorem along Theorem~\ref{thm:PAML-PolicyGradientError} on PG error estimate in order to provide the following convergence rate for the class of exponentially parameterized policies.

\begin{corollary}
\label{cor:PAML-MBPG-Convergence-ExponentialFamilyPolicy}
Consider any distributions $\rho, \mu, \nu \in \bar{\MM}(\XX)$.
Let the policy $\pi_\theta$ be in the exponential family~\eqref{eq:PAML-PolicyParametrization} with $\theta \in \Theta$ and $\Theta$ being a convex subset of $\Real^d $.
Assume that $B = \sup_{x,a \in \XA} \norm{\phi(a|x)}_2 < \infty$.
Suppose that the value function is bounded by $\Qmax$, the MDP has a finite number of actions $\actionnum$, and $\frac{1}{2} \leq \gamma < 1$.
Let
$
	\beta =
	B^2 \Qmax
	\left[
		\frac{8 \gamma \actionnum^2 }{(1-\gamma)^2} + 
		\frac{6 \actionnum }{1 - \gamma}
	\right]
$.
%
%
Let $T$ be an integer number. Starting from a $\pi_{\theta_0} \in \Pi$, consider the sequence of policies $\pi_{\theta_1}, \dotsc, \pi_{\theta_T}$ generated by the projected model-based PG algorithm~\eqref{eq:PAML-ProjectedPGUpdate-WithModel} with step-size $\eta = \frac{1}{\beta}$.
Let $W = \sup_{\theta \in \Theta} \norm{w^*(\theta;\rhoDiscounted^\piBest)}_2$, and assume that $W < \infty$.
Assume that for any policy $\pi_\theta \in \{ \pi_{\theta_0}, \dotsc, \pi_{\theta_{T-1}} \}$, there exists constant $\eps_\text{PAE}$ such that $\LossPolicyError(\theta; \rhoDiscounted^\piBest) \leq \eps_{\text{PAE}}$.
We then have
\begin{align*}
	&
	\EEX{t \sim \text{Unif}(1,\dotsc, T)}{
	J_\rho(\piBest) - J_\rho(\pi_{\theta_t})}
	\leq
	\\
	&
	\frac{1}{1 - \gamma}
	\Bigg[
		\eps_\text{PAE} + 
		\frac{B \Qmax}{1 - \gamma}
		\cPG(\rho,\mu;\piBest)
	(1 \vee W ) 
	\left(
		4 \actionnum
		\sqrt{ \frac{14 \gamma }{T} } +
		\frac{\gamma }{1-\gamma} e_\text{model}
	\right)
	\Bigg].	
\end{align*}
%
%
with
\begin{align*}
	e_\text{model}
	=
	\begin{cases}
		\sup_{\theta \in \Theta} \cPG(\mu,\nu;\pi_\theta) \norm{\Delta \PKernel^{\pi_\theta}}_{1,1(\nu)}, \\
		2 \sup_{\theta \in \Theta} \norm{\Delta \PKernel^{\pi_\theta}}_{1,\infty}.
	\end{cases}
\end{align*}
\end{corollary}
\begin{proof}
For the exponential family policy parameterization~\eqref{eq:PAML-PolicyParametrization}, Lemma~\ref{lem:PAML-ExponentialFamilyPolicy-Smoothness} shows that Assumption~\ref{asm:PolicySmoothness} is satisfied with the choice of $\beta_1 = 2B$ and $\beta_2 = 6B^2$.
We may now apply Theorem~\ref{thm:PAML-MBPG-Convergence}.
To provide an upper bound for $\eps_{\text{model}}$ in that theorem, we apply Theorem~\ref{thm:PAML-PolicyGradientError}.
After some simplifications applicable for $\gamma \geq 1/2$, we obtain the desired result.
\end{proof}

Notice that the choice of $\gamma \geq 1/2$ is only to simply the upper bound, and a similar result holds for any $\gamma \in [0,1)$.

This result relates the performance of the best policy obtained as a result of $T$ iterations of the PG algorithm to the number of iterations $T$, the distribution mismatch between $\cPG(\rho,\mu;\piBest)$, some quantities related to the MDP and policy space, and in particular to the model error $e_\text{model}$.
The model error is expressed in terms of the TV error $\norm{\Delta \PKernel^{\pi_\theta}}_{1,\infty}$, and not in terms of PAML-like objective, as in Theorem~\ref{thm:PAML-MBPG-Convergence}.
This result shows how a conventional model learning approach, which can provide a guarantee on the TV error (or the $\KL$-divergence), leads to a reasonable MBPG method.
As discussed after Theorem~\ref{thm:PAML-PolicyGradientError}, however, TV and $\KL$ might provide loose upper bounds.

\subsection{Inexact Critic}
\label{sec:PAML-Theory-PG-Convergence-Inexact-Critic}

We now focus on the case that the true value function $\Qpi$ is unknown, and instead we have a critic $\Qhatpi \approx \Qpi$.
Rather than analyze the problem of how accurate we can estimate the critic given a number of samples, we only suppose that we have an upper bound on the error of the critic and analyze the effect of this error on the performance of the resulting policy. 
We assume that
\[
	\norm{\Qhatpi - \Qpi}_{2(\muhatDiscounted^\pi;\pi)} \leq \eps_\text{critic}.
\]
where the norm is defined as~\eqref{eq:PAML-Norm-Q}, and 
$\muhatDiscounted^\pi = \muDiscounted(\cdot;\PKernelpihat)$ is the discounted future-state probability of following model $\PKernelpihat$ (cf.~\eqref{eq:PAML-DiscountedFutureStateDistribution}).
We also need to assume that $\Qhatpi$ is $\Qmax$-bounded, which is easy to enforce by truncating the output of any estimator at the known threshold of $\Qmax$.


To make our discussion more clear, we introduce a new notation.
For a transition probability kernel $\PKernel^\pitheta$ and a value function estimate $\Qhat^\pitheta$, we define
\begin{align}
\label{eq:PAML-GradientOfPerformance-InexactCritic}
	\nabla_\theta J_\mu(\pi_\theta, \PKernel^\pitheta, \Qhat^\pitheta) =
	\frac{1}{1-\gamma}
			\EEX{X \sim \muDiscounted(.;\PKernel^\pitheta)}
			    	{\EEX{A \sim \pitheta(.|X)} {\nabla_\theta \log \pi_\theta (A|X) \Qhat^\pitheta(X,A) } }.
\end{align}
This is the PG computed when the distribution is generated according to $\PKernel^\pitheta$ and the critic is $\Qhat^\pitheta$.
Other combinations such as $\nabla_\theta J_\mu(\pi_\theta, \PKernelhat^\pitheta, \Qhat^\pitheta)$ and $\nabla_\theta J_\mu(\pi_\theta, \PKernel^\pitheta, Q^\pitheta)$ follow the similar definition.

We need to make another assumption about the critic.
\begin{assumption}
\label{asm:Critic-Smoothness}
The critic $\Qhat^{\pi_\theta}$ is such that for any $\theta_1, \theta_2 \in \Theta$ and for any $x \in \XX$, there exists a constant $L \geq 0$ such that
\begin{align*}
	\norm{ \Qhat^{\pi_{\theta_1}}(x,\cdot) - \Qhat^{\pi_{\theta_2}}(x,\cdot)
	}_2
	\leq
	L \norm{\theta_1 - \theta_2}_2.
\end{align*}
%
\end{assumption}
This is an assumption on how much the critic changes as the policy $\pi_\theta$ changes.
We require a Lipschitzness as a function of the policy parameter $\theta$.
The intuition of why we need this assumption is that if the critic changes too much as we change the policy, the performance according to this critic would not be smooth enough, hence making the optimization difficult.
We did not require this assumption in the exact critic case in Section~\ref{sec:PAML-Theory-PG-Convergence-Exact-Critic}, because the exact value function $Q^\pitheta$ is smooth, under the smoothness assumption on the policy space (Assumption~\ref{asm:PolicySmoothness}), as stated in Lemma~\ref{lem:PAML-PolicySmoothness} in Appendix~\ref{sec:PAML-Theory-Background}.

We consider a projected PG that uses the model $\PKernelhat^{\pi_{\theta_t}}$, similar to~\eqref{eq:PAML-ProjectedPGUpdate-WithModel} of Section~\ref{sec:PAML-Theory-PG-Convergence-Exact-Critic}, but with a value function $\hat{Q}^{\pi_{\theta_t} }$ that might have some errors, i.e., 
\begin{align}
\label{eq:PAML-ProjectedPGUpdate-WithModel-InexactCritic}
	\theta_{t+1} \leftarrow \ProjTheta \left [ \theta_t + \eta \nabla_\theta J_\mu (\pi_{\theta_t}, \PKernelhat^{\pi_{\theta_t}}, \hat{Q}^{\pi_{\theta_t} }) \right ],
\end{align}
with a learning rate $\eta > 0$, to be specified.
We particularly focus on the exponential policy parameterization~\eqref{eq:PAML-PolicyParametrization} for finite action $\actionnum < \infty$, as opposed to the general policy class of Theorem~\ref{thm:PAML-MBPG-Convergence}. This is mainly to simplify some derivations, and potentially can be relaxed.
The following theorem is the main result of this section.

%

\begin{theorem}
\label{thm:PAML-MBPG-Convergence-InexactCritic}
Consider any initial distributions $\rho, \mu \in \bar{\MM}(\XX)$. Let the policy space $\Pi$ consists of policies $\pi_\theta$ in the exponential family~\eqref{eq:PAML-PolicyParametrization} with $\theta \in \Theta$ and $\Theta$ being a convex subset of $\Real^d $. We assume that $B = \sup_{x,a \in \XA} \norm{\phi(a|x)}_2 < \infty$,
the MDP has a finite number of actions $\actionnum$, and the discount factor $0 \leq \gamma < 1$.
Furthermore, suppose that the critic satisfies Assumption~\ref{asm:Critic-Smoothness} and is $\Qmax$-bounded.
Let
\begin{align}
\label{eq:PAML-MBPG-Convergence-Theorem-InexactCritic-beta}
	\beta =
	\frac{B}{1 - \gamma}
	\left[
		\sqrt{2 \actionnum} L + 
		\frac{\gamma B \Qmax}{1 - \gamma}
	\right].
\end{align}
Let $T$ be an integer number. Starting from a $\pi_{\theta_0} \in \Pi$, consider the sequence of policies $\pi_{\theta_1}, \dotsc, \pi_{\theta_T}$ generated by the projected model-based PG algorithm~\eqref{eq:PAML-ProjectedPGUpdate-WithModel-InexactCritic} with step-size $\eta = \frac{1}{\beta}$.
Let $W = \sup_{\theta \in \Theta} \norm{w^*(\theta;\rhoDiscounted^\piBest)}_2$, and assume that $W < \infty$.
Assume that for any policy $\pi_\theta \in \{ \pi_{\theta_0}, \dotsc, \pi_{\theta_{T-1}} \}$, there exist constants $\eps_{\text{PAE}}$, $\eps_{\text{model}}$, and $\eps_\text{critic}$ such that 
\begin{align*}
	&
	\LossPolicyError(\theta; \rhoDiscounted^\piBest) \leq \eps_{\text{PAE}},
	& \text{(policy approximation error)}
	\\
	&
	\norm{ \nabla_\theta J_\mu(\pi_{\theta}, \PKernel^\pitheta, \Qhat^\pitheta) - \nabla_\theta J_\mu(\pi_{\theta}, \PKernelhat^\pitheta, \Qhat^\pitheta) }_2 \leq \eps_{\text{model}},
	& \text{(model error)}
	\\
	&
	\norm{Q^\pitheta - \Qhat^\pitheta}_{2(\muDiscounted^\pitheta;\pitheta)} \leq \eps_\text{critic}.
	& \text{(critic error)}
\end{align*}
We then have
\begin{align*}
	\EEX{t \sim \text{Unif}(1,\dotsc, T)}{
	J_\rho(\piBest) - J_\rho(\pi_{\theta_t})}
	\leq
	\frac{1}{1 - \gamma}
	\left[
		\eps_{\text{PAE}} +
		\norm{	\frac	{\mathrm{d}\rhoDiscounted^\piBest}
					{\mathrm{d} \mu }
			}_\infty
	(1 \vee W )
	\left(
		4 \sqrt{ \frac{\Qmax \beta}{T} } +
		\frac{B \, \eps_\text{critic} }{1-\gamma}
		+
		\eps_\text{model}
	\right)
	\right].	
\end{align*}
\end{theorem}
\begin{proof}
By Proposition~\ref{prop:PAML-Smoothness-of-Performance-with-Inexact-Critic} in Appendix~\ref{sec:PAML-Theory-Auxiliary}, the performance according to $\PKernelhat^{\pitheta}$ and with an inexact critic $\Qhat^\pitheta$, that is 
$J_\mu(\pi_{\theta}, \PKernelhat^{\pi_\theta}, \Qhat^{\pi_\theta})$, is $\beta$-smooth w.r.t. $\theta$.
%
%
%
Let the gradient mapping for $\theta$ be defined as
\[
	G^\eta(\theta) = \frac{1}{\eta} \left( \ProjTheta \left [ \theta + \eta \nabla_\theta J_\mu(\pi_\theta, \PKernelhat^{\pi_\theta}, \Qhat^{\pi_\theta}) \right ] - \theta \right).
\]
For a projected gradient ascent on a $\beta$-smooth function over a convex set with a step-size of $\eta = \frac{1}{\beta}$, Lemma~\ref{lem:PAML-Beck-Expectation}, which is a slight modification of
Theorem 10.15 of~\citet{Beck2017}, shows that
\begin{align}
\label{eq:PAML-MBPGConvergence-InexactCritic-Theorem-Proof-GradientMappingUpperBound}
	\EEX{t \sim \text{Unif}(0, \dotsc, T-1)}{
	\norm{G^\eta(\theta_t)}_2}
	\leq
	\sqrt{
		\frac	{2 \beta \left( \max_{\pi \in \Pi} J_\mu(\pi, \PKernelhat^{\pi_\theta}, \Qhat^{\pi_\theta}) - J_\mu(\pi_{\theta_0}, \PKernelhat^{\pi_\theta}, \Qhat^{\pi_\theta}) \right) }
			{T}
	}
	\leq
	\sqrt{
		\frac	{4 \Qmax \beta  }
		{T}
	}.
\end{align}
By Proposition~\ref{prop:PAML-SmallGradientMapping-to-epsStationarity} in Appendix~\ref{sec:PAML-Theory-Auxiliary} (originally Proposition D.1 of~\citealt{AgarwalKakadeLeeMahajar2019}), if we let $\theta' = \theta + \eta G^\eta$, we have that
\begin{align*}
	\max_{\theta + \delta \in \Theta, \norm{\delta}_2 \leq 1}
	\delta^\top \nabla_\theta J_\mu(\pi_{\theta'}, \hat{\PKernel}^{\pi_{\theta'}}, \Qhat^{\pi_{\theta'} } )
	\leq
	(\eta \beta + 1) \norm{G^\eta(\theta)}_2.
\end{align*}
This upper bound along~\eqref{eq:PAML-MBPGConvergence-InexactCritic-Theorem-Proof-GradientMappingUpperBound} and $\beta \eta + 1 = 2$ show that the sequence $\theta_0, \theta_1, \dotsc, \theta_T$ generated by~\eqref{eq:PAML-ProjectedPGUpdate-WithModel-InexactCritic} satisfies
%
%
%
\begin{align}
\label{eq:PAML-MBPGConvergence-InexactCritic-Theorem-Proof-epsStationarity-forModel-UpperBound}
	\frac{1}{T} \sum_{t = 1}^T
	\max_{\theta_t + \delta \in \Theta, \norm{\delta}_2 \leq 1}
	\delta^\top \nabla_\theta J_\mu(\pi_{\theta_t}, \hat{\PKernel}^{\pi_{\theta_t}}, \Qhat^{\pi_{\theta_t} })
	\leq
	(\eta \beta + 1)
	\frac{1}{T}
	\sum_{t=0}^{T-1}
	\norm{G^\eta(\theta_t)}_2
	\leq
	4 \sqrt{ \frac{\Qmax \beta}{T} }.
\end{align}

Note that this is a guarantee on $\delta^\top \nabla_\theta J_\mu(\pi_{\theta_t}, \PKernelhat^{\pi_{\theta_t}}, \Qhat^{\pi_{\theta_t}})$, the inner product of a direction $\delta$ with the PG according to the model $\PKernelhat^{\pi_{\theta_t}}$ and the inexact critic $\Qhat^{\pi_{\theta_t}}$, and not on
$\delta^\top \nabla_\theta J_\mu(\pi_{\theta_t}, \PKernel^{\pi_{\theta_t}}, Q^{\pi_{\theta_t}})$, which has the PG according to $\PKernel^{\pi_{\theta_t}}$ and the exact critic $Q^{\pi_{\theta_t}}$. 
The latter is what we need in order to compare the performance.
We can relate them, however, using a series of inequalities.
For any $\theta$, including $\theta_0, \theta_1, \dotsc, \theta_{T}$, we have
\begin{align}
\label{eq:PAML-MBPGConvergence-InexactCritic-Theorem-Proof-ModelBasedDirectionalGradient-To-ModelFree}
\nonumber
	\begin{split}
	\max_{ \norm{\delta}_2 \leq 1}
	\left|
		\delta^\top \nabla_\theta J_\mu(\pi_\theta, \PKernel^\pitheta, Q^\pitheta)
	\right|
    = & \max_{ \norm{\delta}_2 \leq 1}
	\Bigg|
		\delta^\top 
		\Big[ 
		\nabla_\theta J_\mu(\pi_\theta, \PKernelhat^\pitheta,  \Qhat^\pitheta) + 
		\\
		&
		\qquad \qquad \quad
		\left( 
		\nabla_\theta J_\mu(\pi_\theta, \PKernel^\pitheta, Q^\pitheta) - \nabla_\theta J_\mu(\pi_\theta, \PKernel^\pitheta, \Qhat^\pitheta) 
		\right) + {}
		\\ 
		& 
		\qquad \qquad \quad
		\left( \nabla_\theta J_\mu(\pi_\theta, \PKernel^\pitheta, \Qhat^\pitheta) - \nabla_\theta J_\mu(\pi_\theta, \PKernelhat^\pitheta, \Qhat^\pitheta) 
		\right)
		\Big]
	\Bigg|
    \end{split}
	\\ \nonumber
	\begin{split}
	\leq &
	\max_{ \norm{\delta}_2 \leq 1}
	\left|
		\delta^\top \nabla_\theta J_\mu(\pi_\theta, \PKernelhat^\pitheta,  \Qhat^\pitheta)
	\right| + \\
	&
	\max_{ \norm{\delta}_2 \leq 1}
	\left|
		\delta^\top \left( \nabla_\theta J_\mu(\pi_\theta, \PKernel^\pitheta, Q^\pitheta) - \nabla_\theta J_\mu(\pi_\theta, \PKernel^\pitheta, \Qhat^\pitheta) \right)
	\right| + 
	\\ \nonumber
	&
	\max_{ \norm{\delta}_2 \leq 1}
	\left|
		\delta^\top \left( \nabla_\theta J_\mu(\pi_\theta, \PKernel^\pitheta, \Qhat^\pitheta) - \nabla_\theta J_\mu(\pi_\theta, \PKernelhat^\pitheta, \Qhat^\pitheta) \right)
	\right|
	\end{split}
	\\ \nonumber
	= &
	\max_{ \norm{\delta}_2 \leq 1}
	\left|
		\delta^\top \nabla_\theta J_\mu(\pi_\theta, \PKernelhat^\pitheta,  \Qhat^\pitheta)
	\right| +
	\\
\nonumber	
	&
	\norm{ \nabla_\theta J_\mu(\pi_\theta, \PKernel^\pitheta, Q^\pitheta) - \nabla_\theta J_\mu(\pi_\theta, \PKernel^\pitheta, \Qhat^\pitheta) }_2 
	+ \\
	&
	\norm{ \nabla_\theta J_\mu(\pi_\theta, \PKernel^\pitheta, \Qhat^\pitheta) - \nabla_\theta J_\mu(\pi_\theta, \PKernelhat^\pitheta, \Qhat^\pitheta) }_2.
\end{align}

These terms represent error due to the optimization process ($\max_{ \norm{\delta}_2 \leq 1}
	|\delta^\top \nabla_\theta J_\mu(\pi_\theta, \PKernelhat^\pitheta,  \Qhat^\pitheta)|$),
	error due to having an inexact critic ($\smallnorm{ \nabla_\theta J_\mu(\pi_\theta, \PKernel^\pitheta, Q^\pitheta) - \nabla_\theta J_\mu(\pi_\theta, \PKernel^\pitheta, \Qhat^\pitheta) }_2$),
	and model error ($\smallnorm{ \nabla_\theta J_\mu(\pi_\theta, \PKernel^\pitheta, \Qhat^\pitheta) - \nabla_\theta J_\mu(\pi_\theta, \PKernelhat^\pitheta, \Qhat^\pitheta) }_2$).
We provide an upper bound for each of them.

The model error is upper bounded by assumption:
\begin{align}
\label{eq:PAML-MBPGConvergence-InexactCritic-Theorem-Proof-ModelError-Term}
	\norm{ \nabla_\theta J_\mu(\pi_\theta, \PKernel^\pitheta, \Qhat^\pitheta) - \nabla_\theta J_\mu(\pi_\theta, \PKernelhat^\pitheta, \Qhat^\pitheta) }_2
	\leq
	\eps_{\text{model}}.
\end{align}

The error due to the inexact critic can be upper bounded after an application of the Cauchy-Schwarz inequality, using
Lemma~\ref{lem:PAML-ExponentialFamilyPolicy-Boundedness}, and our assumption on the critic error as follows
\begin{align}
\label{eq:PAML-MBPGConvergence-InexactCritic-Theorem-Proof-CriticError-Term}
    \nonumber
    &
    \norm{ \nabla_\theta J_\mu(\pi_\theta, \PKernel^\pitheta, Q^\pitheta) - \nabla_\theta J_\mu(\pi_\theta, \PKernel^\pitheta, \Qhat^\pitheta) }_2
    \\
    \nonumber
    &= \norm{ \frac{1}{1-\gamma} \EEX{X \sim \muDiscounted(.;\PKernel^\pitheta)}
    {\EEX{A \sim \pitheta(.|X)} {\nabla_\theta \log \pi_\theta (A|X) \left( Q^\pitheta(X,A) - \Qhat^\pitheta(X,A) \right)
    }}
    }_2
    \\ 
    \nonumber
    &\leq
    \frac{1}{1-\gamma} \EEX{X \sim \muDiscounted(.;\PKernel^\pitheta)}
    {\EEX{A \sim \pitheta(.|X)} {\norm{\nabla_\theta \log \pi_\theta (A|X) \left( Q^\pitheta(X,A) - \Qhat^\pitheta(X,A) \right)
    }_2}}
    \\ \nonumber
    &
    \leq
       \frac{1}{1-\gamma} 
	\sqrt{\EEX{X \sim \muDiscounted(.;\PKernel^\pitheta)}{\EEX{A \sim \pitheta(.|X)}{\norm{\nabla_\theta \log \pitheta (A|X)}^2_2}}}
    \times
    \\
    \nonumber
    & \qquad \qquad
	\sqrt{\EEX{X \sim \muDiscounted(.;\PKernel^\pitheta)}{\EEX{A \sim \pitheta(.|X)}{\left|Q^\pitheta(X,A) - \Qhat^\pitheta(X,A) \right|^2}}}    
     \\
    & \leq 
    \frac{B}{1 - \gamma} \norm{Q^\pitheta - \Qhat^\pitheta}_{2(\muDiscounted^\pitheta; \pitheta)}
    \leq 
    \frac{B}{1 - \gamma} \eps_{\text{critic}}.
\end{align}


Plugging the upper bounds~\eqref{eq:PAML-MBPGConvergence-InexactCritic-Theorem-Proof-epsStationarity-forModel-UpperBound},~\eqref{eq:PAML-MBPGConvergence-InexactCritic-Theorem-Proof-ModelError-Term}, and~\eqref{eq:PAML-MBPGConvergence-InexactCritic-Theorem-Proof-CriticError-Term} in~\eqref{eq:PAML-MBPGConvergence-InexactCritic-Theorem-Proof-ModelBasedDirectionalGradient-To-ModelFree} show that
\begin{align}
\label{eq:PAML-MBPGConvergence-Inexact-Theorem-Proof-AverageStationarity}
	\frac{1}{T} \sum_{t=1}^T
	\max_{ \norm{\delta}_2 \leq 1}
		\left|
			\delta^\top \nabla_\theta J_\mu(\pi_{\theta_t}, \PKernel^{\pi_{\theta_t}}, Q^{\pi_{\theta_t}})
		\right|
		\leq \eps_\text{model} + 
		4 \sqrt{ \frac{\Qmax \beta}{T} } +
		\frac{B}{1 - \gamma} \eps_\text{critic}.
\end{align}

We can now evoke Theorem~\ref{thm:PAML-PolicyError-at-Stationarity} for each $\pi_{\theta_t}$ and take a summation over both sides of the inequality.
Suppose that $\pi_{\theta_t}$ is $\eps_t$-stationary, i.e., 
$\delta^\top \nabla_\theta J_\mu(\pi_{\theta_t};\PKernel^{\pitheta_t}, Q^{\pitheta_t}) \leq \eps_t$ for any valid $\delta$, see~\eqref{eq:PAML-Stationarity}.
Also recall that $W = \sup_{\theta \in \Theta} \norm{w^*(\theta;\rhoDiscounted^\piBest)}_2$.
So we have
\begin{align*}
	\frac{1}{T} \sum_{t=1}^T J_\rho(\piBest) - J_\rho(\pi_{\theta_t} )
	& \leq
	\frac{1}{1 - \gamma}
	\frac{1}{T} \sum_{t=1}^T
	\left[
		\LossPolicyError(\theta_t; \rhoDiscounted^\piBest) +
		\norm{	\frac	{\mathrm{d}\rhoDiscounted^\piBest}
					{\mathrm{d} \mu }
			}_\infty
	(1 \vee \norm{w^*(\theta_t;\rhoDiscounted^\piBest)} ) \eps_t
	\right]
	\\
	& \leq
	\frac{1}{1 - \gamma}
	\left[
		\eps_{\text{PAE}} +
		\norm{	\frac	{\mathrm{d}\rhoDiscounted^\piBest}
					{\mathrm{d} \mu }
			}_\infty
	(1 \vee W )
	\frac{1}{T} \sum_{t=1}^T \eps_t
	\right],
	\\
	& \leq
	\frac{1}{1 - \gamma}
	\left[
		\eps_{\text{PAE}} +
		\norm{	\frac	{\mathrm{d}\rhoDiscounted^\piBest}
					{\mathrm{d} \mu }
			}_\infty
	(1 \vee W )
	\left(
		4 \sqrt{ \frac{\Qmax \beta}{T} } +
		\eps_\text{model} +
		\frac{B}{1 - \gamma} \eps_\text{critic}
	\right)
	\right],
\end{align*}
where we used~\eqref{eq:PAML-MBPGConvergence-Inexact-Theorem-Proof-AverageStationarity} in the last inequality.
This is the desired result.
\end{proof}

\todo{Add some discussion on the smoothness of critic. -AMF}
\todo{Some discussion about the proof? -AMF}

\subsection{Effect of Policy Change on the Loss Function}
\label{sec:PAML-Theory-PolicyChange}

\todo{This is a new section! Re-read and revise it. -AMF}

Recall from Section~\ref{sec:PAML-Algorithm} that since the loss function 
$c_\rho(\PKernel^{\pi_\theta},\PKernelhat^{\pi_\theta})$~\eqref{eq:PAML-DifferenceInGradient} is defined for a particular policy $\pi_\theta$, model $\PKernelhat^{\pi_\theta}$ should be updated based on the most recent $\pi_\theta$, because the policy $\pi_\theta$ gradually changes during the run of a PG algorithm.
An important practical question is how quickly the model expires after a policy update. Should the model be updated very frequently, or can we update it only occasionally?
We empirically study this question in Section~\ref{sec:PAML-Empirical} (see Figure~\ref{fig:delta-loss}).
In this section, we theoretically study this question in some detail.

Suppose that we start from policy a $\pi_{\theta}$, learn a model $\PKernelhat^{\pi_\theta}$ that minimizes the PAML's loss, and then update the policy to $\pi_{\theta'}$.
We would like to know whether we might use the distribution induced by following $\PKernelhat^{\pi_\theta}$ in order to compute the PG w.r.t. the new policy $\theta'$.

Let us introduce a new notation. 
Given two policies $\pi_{\theta_1}$ and $\pi_{\theta_2}$, we define
the PG of starting from initial state distribution $\rho$, following $\PKernel^{\pi_{\theta_1} }$ and evaluating the pointwise gradient \linebreak $\EEX{A \sim \pi_{\theta_2} }{ \nabla_\theta \log \pi_{\theta_2}(A|x) Q^{\pi_{\theta_2}}(x,A)}$ by 
\begin{align}
\label{eq:PAML-GradientOfPerformance-DifferentPolicyandModel}
\nonumber
	\nabla_\theta J(\pi_{\theta_2 }; \PKernel^{\pi_{\theta_1}} ) 
	& = 
	\sum_{k \geq 0} \gamma^k 
		\int \drho(x)
		\int		
			\PKernel^{\pi_{\theta_1}}(\dx'|x;k)
			\sum_{a' \in \AA} \frac{\partial \pi_{\theta_2} (a'|x')}{\partial \theta} Q^{\pi_{\theta_2} } (x',a')
	\\
	& =
	\frac{1}{1 - \gamma}	
	\int \rhoDiscounted(\dx;\PKernel^{\pi_{\theta_1} })
	\sum_{a \in \AA} \pi_\theta(a|x) \frac{\partial \log \pi_{\theta_2} (a|x)}{\partial \theta} Q^{\pi_{\theta_2} } (x,a).
\end{align}

Note that the PG of $\pi_\theta$ w.r.t. the true model $\PKernel^{\pi}$
is $\nabla_\theta J(\pi_\theta) = \nabla_\theta J(\pi_{\theta}; \PKernel^{\pi_{\theta}} )$ (see ~\eqref{eq:PAML-GradientOfPerformance}) and 
the PG w.r.t. the learned model $\PKernelhat^{\pi_\theta}$ is
$\nabla_\theta \hat{J}(\pi_\theta) = \nabla_\theta J(\pi_{\theta}; \PKernelhat^{\pi_{\theta}} )$.
For both of these, the policy of the model and the policy of the integrand of the PG are the same.

\todo{minor: I found it confusing that the paragraph just above uses 1 and 2 and this one uses ' to distinguish two different policies. Notationally it could be simpler to use ' everywhere -RA}
On the other hand, the PG of $\pi_{\theta'}$ using the model learned at $\pi_\theta$, that is $\PKernelhat^{\pi_\theta}$, is 
$\nabla_\theta J(\pi_{\theta'}; \PKernel^{\pi_{\theta}} )$.
We would like to know how different $\nabla_\theta J(\pi_{\theta'}; \PKernel^{\pi_{\theta}} )$ is compared to the true PG of $\pi_{\theta'}$, which is $\nabla_\theta J(\pi_{\theta'}) = \nabla_\theta J(\pi_{\theta'}; \PKernel^{\pi_{\theta'}} )$.
If the difference is small, it entails that the model is still valid.

Before stating the result, recall that $\PKernelhat^{\pi_\theta}$ is the minimizer of the empirical version of the loss function~\eqref{eq:PAML-DifferenceInGradient}. Depending on how close we get to the minimizer, which is a function of the number of samples, the expressivity of the model space, the optimizer, etc., we might have some error. We assume that the error is $\eps_\text{model}$, i.e., 
\begin{align}
\label{eq:PAML-Theory-PolicyChange-ModelErrorAssumption}
	\norm{
	\nabla_\theta J(\pi_\theta; \PKernel^{\pi_\theta} ) - 
	\nabla_\theta J(\pi_\theta; \PKernelhat^{\pi_\theta} )
		}_2
	\leq
	\eps_\text{model}.
\end{align}

We are ready to state the main result of this section.

\begin{proposition}[Loss Change]
\label{prop:PAML-PolicyChange-to-LossObsoletion}
Consider a policy $\pi_\theta$ with the policy parameterization~\eqref{eq:PAML-PolicyParametrization}. Assume that $B = \sup_{(x,a) \in \XA} \norm{\phi(a|x)}_2 < \infty$, the action space is finite with $\actionnum$ elements, and
 the action-value functions are all $\Qmax$-bounded. Moreover, assume that the model error is bounded by $\eps_\text{model}$~\eqref{eq:PAML-Theory-PolicyChange-ModelErrorAssumption}.
We then have
\begin{align*}
	\norm{ 
		\nabla_\theta J(\pi_{\theta'}; \PKernel^{\pi_{\theta'}} ) -
		\nabla_\theta J(\pi_{\theta'}; \PKernelhat^{\pi_{\theta}} )
		}_2
	\leq
	\eps_\text{model} +
	c_1 \norm{\theta - \theta'}_2,
\end{align*}
with
\[
	c_1 = 	\frac{\Qmax B^2 \actionnum}{1 - \gamma}
	\left[
		12 + \frac{4 \gamma (1 + 2 \actionnum)}{1 - \gamma}
	\right].
\]
\end{proposition}
\begin{proof}
Consider two policies $\pi_{\theta}$ and $\pi_{\theta'}$.
We have
\begin{align*}
	\norm{ 
		\nabla_\theta J(\pi_{\theta'}; \PKernel^{\pi_{\theta'}} ) -
		\nabla_\theta J(\pi_{\theta'}; \PKernelhat^{\pi_{\theta}} )
		}_2
	\leq
	&
	\norm{ 
		\nabla_\theta J(\pi_{\theta'}; \PKernel^{\pi_{\theta'}} ) -
		\nabla_\theta J(\pi_{\theta}; \PKernel^{\pi_{\theta}} )
		}_2
	+ {}
	\\	
	&
	\norm{ 
		\nabla_\theta J(\pi_{\theta}; \PKernel^{\pi_{\theta}} ) -
		\nabla_\theta J(\pi_{\theta}; \PKernelhat^{\pi_{\theta}} )
		}_2	
	+ {}
	\\	
	&
	\norm{ 
		\nabla_\theta J(\pi_{\theta}; \PKernelhat^{\pi_{\theta}} ) -
		\nabla_\theta J(\pi_{\theta'}; \PKernelhat^{\pi_{\theta}} )
		}_2.
\end{align*}

We consider each term in the RHS separately, and provide an upper bound for them.

The term 
$\smallnorm{ 
		\nabla_\theta J(\pi_{\theta'}; \PKernel^{\pi_{\theta'}} ) -
		\nabla_\theta J(\pi_{\theta}; \PKernel^{\pi_{\theta}} )
		}_2$
is the change in the true PG from $\pi_{\theta}$ to $\pi_{\theta'}$.
It can be written as
\[
	\norm{ \EEX{X \sim \rho}{ \nabla_\theta V^{\pi_{\theta'}}(X) - \nabla_\theta V^{\pi_{\theta}}(X) } }_2.
\]
By Lemma~\ref{lem:PAML-PolicySmoothness}, for any policy that satisfies Assumption~\ref{asm:PolicySmoothness}, we have that 
\[
	\norm{ \nabla_\theta V^{\pi_{\theta'}}(x) - \nabla_\theta V^{\pi_{\theta}}(x) }_2
	\leq
	\beta \norm{\theta' - \theta}_2,
\]
with
$
	\beta =
	\Qmax
	\left[
		\frac{2 \gamma \beta_1^2 \actionnum^2 }{(1-\gamma)^2} + 
		\frac{\beta_2 \actionnum }{1 - \gamma}
	\right]
$.
Lemma~\ref{lem:PAML-ExponentialFamilyPolicy-Smoothness} shows that for the exponential family, 
$	\beta_1 = 2B$ and $\beta_2 = 6B^2$.
Therefore, 
\begin{align}
\label{eq:PAML-PolicyChange-to-LossObsoletion-Proof-Term1}
	\norm{ 
		\nabla_\theta J(\pi_{\theta'}; \PKernel^{\pi_{\theta'}} ) -
		\nabla_\theta J(\pi_{\theta}; \PKernel^{\pi_{\theta}} )
		}_2
	\leq
	\Qmax B^2
	\left( \frac{8 \gamma \actionnum^2}{(1-\gamma)^2} +
		\frac{6\actionnum}{1 - \gamma}
	\right)
	\norm{\theta' - \theta}_2.
\end{align}

The term $	\smallnorm{ 
		\nabla_\theta J(\pi_{\theta}; \PKernel^{\pi_{\theta}} ) -
		\nabla_\theta J(\pi_{\theta}; \PKernelhat^{\pi_{\theta}} )
		}_2$ is the model error at $\pi_{\theta}$ and is upper bounded by $\eps_\text{model}$ by assumption.

To provide an upper bound for
$	\smallnorm{ 
		\nabla_\theta J(\pi_{\theta}; \PKernelhat^{\pi_{\theta}} ) -
		\nabla_\theta J(\pi_{\theta'}; \PKernelhat^{\pi_{\theta}} )
		}_2$,
let us first denote
\[
	f(x;\theta) = \EEX{A \sim \pi_{\theta} }{ \nabla_\theta \log \pi_{\theta}(A|x) Q^{\pi_\theta}(x,A)}.
\]
We then have
\begin{align}
\label{eq:PAML-PolicyChange-to-LossObsoletion-Proof-Term2-Intermediate}
\nonumber
	\norm{ 
		\nabla_\theta J(\pi_{\theta}; \PKernelhat^{\pi_{\theta}} ) -
		\nabla_\theta J(\pi_{\theta'}; \PKernelhat^{\pi_{\theta}} )
		}_2
	& = 
	\frac{1}{1 - \gamma}
	\norm{
		\int \rhoDiscounted(\dx; \PKernelhat^{\pi_\theta}) \left( f(x;\theta) - f(x;\theta') \right)
	}_2
	\\
	&
	\leq
	\frac{1}{1 - \gamma}
		\int \rhoDiscounted(\dx; \PKernelhat^{\pi_\theta}) \norm{ f(x;\theta) - f(x;\theta') }_2.
\end{align}

Lemma~\ref{lem:PAML-MultivariateMeanValueTheorem} shows that
\begin{align}
\label{eq:PAML-PolicyChange-to-LossObsoletion-Proof-Term2-Intermediate-2}
	\norm{ f(x;\theta) - f(x;\theta') }_2
	\leq
	\sup_{\theta} \norm{ \nabla_\theta f(x;\theta) }_2 \norm{\theta - \theta'}_2.
\end{align}
Note that $\nabla_\theta f(x;\theta)$ is a matrix.
To compute $\nabla_\theta f(x;\theta)$, we first note that
$f(x;\theta) = \sum_{a \in \AA} \pi_\theta(a) \nabla_\theta \log \pi_\theta(a|x) Q^{\pi_\theta}(x,a) = \sum_{a \in \AA} \nabla_\theta \pi_\theta( a|x) Q^{\pi_\theta}(x,a)$.
Therefore,
\begin{align}
\label{eq:PAML-PolicyChange-to-LossObsoletion-Proof-gradient-of-f}
\nonumber
	\nabla_\theta f(x;\theta) & =
	\nabla_\theta \sum_{a \in \Actions} \nabla_\theta \pi_\theta( a|x) Q^{\pi_\theta}(x,a)
	\\
	& =
	\sum_{a \in \Actions} \frac{\partial^2 \pi(a|x) }{\partial \theta^2} Q^{\pi_\theta}(x,a) +
	\nabla_\theta \pi(a|x) \nabla_\theta Q^{\pi_\theta}(x,a).
\end{align}

We need to upper bound the norm of each of these terms.

Lemma~\ref{lem:PAML-ExponentialFamilyPolicy-Smoothness} shows that
\begin{align}
\label{eq:PAML-PolicyChange-to-LossObsoletion-Proof-PolicyGradientBound}
	\norm{\nabla_\theta \pi_\theta}_2
	\leq 2B,
\end{align}
\begin{align}
\label{eq:PAML-PolicyChange-to-LossObsoletion-Proof-PolicyHessianBound}
	&
	\norm{ \frac{\partial^2 \pi(a|x) }{\partial \theta} }_2 \leq 6B^2,
\end{align}
in which the matrix norm is the $\ell_2$-induced norm. 
%

%
%

We use an argument similar to the proof of PG theorem (Theorem 1 by~\citealt{SuttonMcAleesterSinghMansour2000}) to get
\begin{align*}
	\nabla_\theta Q^{\pi_\theta}(x,a) & =
	\frac{\partial}{\partial \theta}
	\left[
		r(x,a) + \gamma \int \PKernel(\dy | x,a) V^{\pi_\theta}(y)
	\right]
	\\ 
	& 	=
	\gamma \int \PKernel(\dy | x,a) \frac{\partial V^{\pi_\theta}(y)}{\partial \theta}
	\\
	& =
	\gamma \int \PKernel(\dy | x,a) \int \sum_{k \geq 0} \gamma^k \PKernel(\mathrm{d} z | y; k) f(z;\theta),
\end{align*}
where we recursively expanded $\frac{\partial V^{\pi_\theta}(y)}{\partial \theta}$.
Therefore,
\begin{align*}
	\norm{ \nabla_\theta Q^{\pi_\theta}(x,a)}_2 
	\leq
	\frac{\gamma}{1 - \gamma} \sup_{x \in \XX} \norm{f(x;\theta)}_2.
\end{align*}
For the exponential family, we get
\begin{align*}
	\norm{f(x;\theta)}_2 & =
	\norm{  \sum_{a \in \Actions} \nabla_\theta \pi_\theta(a|x) Q^{\pi_\theta}(x,a) \da } 
	\\
	& =
	\norm{  \sum_{a \in \Actions} \pi_\theta(a|x) \left( \phi(a|x) - \EEX{\pi_\theta(\cdot|x)}{\phi(A|x)} \right) Q^{\pi_\theta}(x,a) }_2
	\leq
	2B \Qmax.
\end{align*}
As a result,
\begin{align}
\label{eq:PAML-PolicyChange-to-LossObsoletion-Proof-GradientofQ}
	\norm{ \nabla_\theta Q^{\pi_\theta}(x,a)}_2 
	\leq
	\frac{2 \gamma B \Qmax}{1 - \gamma}.
\end{align}

After plug-in~\eqref{eq:PAML-PolicyChange-to-LossObsoletion-Proof-PolicyHessianBound},~\eqref{eq:PAML-PolicyChange-to-LossObsoletion-Proof-PolicyGradientBound}, and~\eqref{eq:PAML-PolicyChange-to-LossObsoletion-Proof-GradientofQ}
in~\eqref{eq:PAML-PolicyChange-to-LossObsoletion-Proof-gradient-of-f},
we get 
\begin{align}
\nonumber
	\norm{\nabla_\theta f(x;\theta)}_2 & \leq
	\sum_{a \in \Actions} \norm{ \frac{\partial^2 \pi(a|x) }{\partial \theta^2} Q^{\pi_\theta}(x,a) }_2 +
	\norm{ \nabla_\theta \pi(a|x) \nabla_\theta Q^{\pi_\theta}(x,a)}_2
	\\
\nonumber	
	&
	\leq
	\sum_{a \in \Actions} \Qmax \norm{ \frac{\partial^2 \pi(a|x) }{\partial \theta^2} }_2 +
	\norm{ \nabla_\theta \pi(a|x)}_2 \norm{\nabla_\theta Q^{\pi_\theta}(x,a)}_2	
	\\
	&
	\leq
	\sum_{a \in \Actions} \Qmax (6B^2)  +
	(2B) \left(\frac{2 \gamma B \Qmax}{1 - \gamma} \right)
	=
	\Qmax B^2 \left( 6 + \frac{4 \gamma}{1 - \gamma}\right) \actionnum.
\end{align}
This result together with~\eqref{eq:PAML-PolicyChange-to-LossObsoletion-Proof-Term2-Intermediate} and~\eqref{eq:PAML-PolicyChange-to-LossObsoletion-Proof-Term2-Intermediate-2} upper bound the third term as follows:
\begin{align}
\label{eq:PAML-PolicyChange-to-LossObsoletion-Proof-Term2}
	\norm{ 
		\nabla_\theta J(\pi_{\theta}; \PKernelhat^{\pi_{\theta}} ) -
		\nabla_\theta J(\pi_{\theta'}; \PKernelhat^{\pi_{\theta}} )
		}_2
	\leq
	\frac{\Qmax B^2}{1 - \gamma}
	\left( 6 + \frac{4 \gamma}{1 - \gamma}\right) \actionnum \norm{\theta - \theta'}_2.
\end{align}

The upper bounds~\eqref{eq:PAML-PolicyChange-to-LossObsoletion-Proof-Term1},~\eqref{eq:PAML-PolicyChange-to-LossObsoletion-Proof-Term2}, and the upper bound on the model error lead to
\begin{align*}
	\norm{ 
		\nabla_\theta J(\pi_{\theta'}; \PKernel^{\pi_{\theta'}} ) -
		\nabla_\theta J(\pi_{\theta'}; \PKernelhat^{\pi_{\theta}} )
		}_2
	\leq
	\eps_\text{model} + 
	\frac{\Qmax B^2 \actionnum}{1 - \gamma}
	\left[
		12 + \frac{4 \gamma (1 + 2 \actionnum)}{1 - \gamma}
	\right]
	\norm{\theta - \theta'}_2.
\end{align*} 
\end{proof}

\todo{Some explanation should go here. -AMF}

\else
    
\todo{Needs a careful reading. -AMF}

We theoretically study the convergence properties of a model-based PG (MBPG) method.
We study how the model error affects the convergence behaviour.
The result is generic for any MBPG method, but it further enlightens why PAML might be a good alternative to MLE.


Our result extends~\citet{AgarwalKakadeLeeMahajar2019} from a model-free setting, where the gradients are computed according to the true dynamics $\PKernelpi$ of the policy, to the model-based setting where the PG is computed according to a learned model, incorporating the model error in the convergence result. 
We also remove the assumption of access to a perfect critic in our analysis, which we defer to the supplementary material due to space constraints.
Additionally, we introduce a new notion of policy approximation error, which is perhaps a better characterization of the approximation error of the policy space.
We would like to note that providing global convergence guarantees for PG methods and variants has only recently attracted attention~\citep{AgarwalKakadeLeeMahajar2019,ShaniEfroniMannor2020,BhandariRusso2019,LiuCaiYangWang2019}.
This section is a very brief summary of the theoretical results in the supplementary material, which not only include proofs, but also more discussion and intuition.

We analyze a projected PG with the assumption that the PGs are computed exactly. We consider a setup where the performance is evaluated according to a distribution $\rho \in \bar{\MM}(\XX)$, but the PG is computed according to a possibly different distribution $\mu \in \bar{\MM}(\XX)$. To be concrete, let us consider a policy space $\Pi = \cset{\pi_\theta}{\theta \in \Theta}$ with $\Theta$ being a convex subset of $\Real^d$ and $\ProjTheta$ be the projection operator onto $\Theta$. 
Let us denote the best policy in the policy class $\Pi$ according to the initial distribution $\rho$ by $\piBestrho$ (or simply $\piBest$, if it is clear from the context), i.e., $\piBest \leftarrow \argmax_{\pi \in \Pi} J_\rho(\pi).$
%
%
We define a function that we call \emph{Policy Approximation Error (PAE)}. Given a policy parameter $\theta$ and $w \in \Real^d$, and for a probability distribution $\nu \in \bar{\MM}(\XX)$, define
%
\begin{align*}
	\LossPolicyError(\theta, w; \nu) \eqdef 
	\mathbb E_{X \sim \nu}
	\Big[
		\Big|
			\sum_{a \in \Actions} & \big( \piBest(a|X) - \pi_\theta(a|X) - 
			\\
			&
			w^\top \nabla_\theta \pi_\theta(a|X) \big) Q^{\pi_\theta}(X,a) 
		\Big|
	\Big].
\end{align*}
This can be roughly interpreted as the error in approximating the improvement in the value from the current policy $\pi_\theta$ to the best policy, $\piBest$, in the class, i.e.,~$\sum_{a \in \Actions} ( \piBest(a|X) - \pi_\theta(a|X) ) Q^{\pi_\theta}(X,a)$, by a linear model 
$\sum_{a \in \Actions} w^\top \nabla_\theta \pi_\theta(a|X) Q^{\pi_\theta}(X,a) = 
w^\top \EEX{A \sim \pi_\theta(\cdot|X)}{\nabla_\theta \log \pi_\theta(a|X) Q^{\pi_\theta}(X,a) }$.
\todo{Maybe revise this interpretation. Is this clear enough? And this is ``roughly'' true because the improvement interpretation only holds when $\nu = \rhoDiscounted^\piBest$. -AMF}

This quantity can be compared to the Bellman Policy Error of~\citet{AgarwalKakadeLeeMahajar2019}, which is 
\begin{align*}
    L_\text{BPE}(\theta, w; \nu) \eqdef 
	\mathbb{E}_{X \sim \nu}
	\Big[
		\Big|
			&
			\sum_{a \in \Actions} 
			\big|
			\argmax_{a \in \AA} Q^{\pi_\theta}(X,a) - 
			\\ 
			&
			\pi_\theta(a|X) - w^\top \nabla_\theta \pi_\theta(a|X)
			\big| 
		\Big|
	\Big].
\end{align*}
There are two differences between these two notions of policy approximation errors.
The first is that $\LossPolicyError$ considers how well one can approximate the best policy in the policy class $\Pi$, instead of the greedy policy w.r.t.~the action-value function of the current policy as in $L_\text{BPE}$. Moreover, while $L_{\text{BPE}}$ ignores the value function and its interaction with the policy error, $\LossPolicyError$ explicitly considers it. For example, if the reward function is constant everywhere, the action-value function  $\Qpi$ for any policy would be constant too. In this case, $\LossPolicyError(\theta;\nu)$ is zero (simply choose $w = 0$ as the minimizer), but $L_\text{BPE}$ may not be.
We leave further study of these two policy approximation errors to a future work.
%

For any $\theta \in \Theta$, we can define the best $w^*(\theta) = w^*(\theta;\nu)$ that minimizes $\LossPolicyError(\theta, w; \nu)$ as
\begin{align}
\label{eq:PAML-Loss-PAE-wstar}
	w^*(\theta;\nu) \leftarrow \argmin_{w + \theta \in \Theta} \LossPolicyError (\theta, w; \nu).
\end{align}
We use $\LossPolicyError(\theta;\nu)$ to represent $\LossPolicyError(\theta, w^*(\theta); \nu)$. 
We may drop the distribution $\nu$ whenever it is clear from the context.
We consider a projected PG algorithm that uses the model $\PKernelhat^{\pi_{\theta_k}}$ and the exact value function $Q^{\pi_k}$ to compute the gradient, i.e.,
%
{\small
\begin{align}
\label{eq:PAML-ProjectedPGUpdate-WithModel}
	\theta_{t+1} \leftarrow \ProjTheta \left [ \theta_t + \eta \nabla_\theta J_\mu (\pi_{\theta_k}, \PKernelhat^{\pi_{\theta_k}}, Q^{\pi_k}) \right ],
\end{align}}
with a learning rate $\eta > 0$ to be specified. 
Refer to the supplementary materials for results on when an estimated, and thus inexact, model-free critic is used.

We require the following smoothness assumption on the policy space.
%
\begin{assumption}[(Assumption 6.12 of~\citet{AgarwalKakadeLeeMahajar2019})]
\label{asm:PolicySmoothness}
Assume that there exist finite constants $\beta_1, \beta_2 \geq 0$ such that for all $\theta_1, \theta_2 \in \Theta$, and for all $(x,a) \in \XA$, we have
$| \pi_{\theta_1}(a|x) - \pi_{\theta_2}(a|x) | \leq \beta_1 \norm{\theta_1 - \theta_2}_2$ and
$\smallnorm{ \nabla_\theta \pi_{\theta_1}(a|x) - \nabla_\theta \pi_{\theta_1}(a|x)  }_2
	\leq \beta_2 \norm{\theta_1 - \theta_2}_2$.
\end{assumption}
%
As an example, this assumption holds for an exponential family 
$	\pi_\theta(a | x) \propto 
\exp \left( \phi^\top(a|x) \theta \right)
$
with bounded features $\norm{\phi(a|x)}_2 \leq B$.
In that case,  $\beta_1 = 2B$ and $\beta_2 = 6B^2$.

We are now ready to state the convergence guarantee for the projected PG method for a general policy class.

\begin{theorem}
\label{thm:PAML-MBPG-Convergence-Compact-New}
Consider any initial distributions $\rho, \mu \in \bar{\MM}(\XX)$ and a policy space $\Pi$ parameterized by $\theta \in \Theta$ with $\Theta$ being a convex subset of $\Real^d$.
Assume that all policies $\pi_\theta \in \Pi$ satisfy Assumption~\ref{asm:PolicySmoothness} and $0 \leq \gamma < 1$.
Let $T$ be an integer number. Starting from a $\pi_{\theta_0} \in \Pi$, consider the sequence of policies $\pi_{\theta_1}, \dotsc, \pi_{\theta_T}$ generated by the projected model-based PG algorithm~\eqref{eq:PAML-ProjectedPGUpdate-WithModel} with step-size as defined in the Supplementary. 
%
Assume that for any policy $\pi_\theta \in \{ \pi_{\theta_0}, \dotsc, \pi_{\theta_{T-1}} \}$, there exist constants 
$\eps_{\text{PAE}}$ and $\eps_{\text{model}}$ 
such that
$\LossPolicyError(\theta; \rhoDiscounted^\piBest) \leq \eps_{\text{PAE}}$ (policy approximation error)
and
$\smallnorm{ \nabla_\theta J_\mu(\pi_\theta) - \nabla_\theta \hat{J}_\mu(\pi_\theta) }_2 \leq \eps_{\text{model}}$ (model error).
%
%
%
%
We then have
%
%
\begin{align*}
	& \EEX{t \sim \text{Unif}(1,\dotsc, T)}{
	J_\rho(\piBest) - J_\rho(\pi_{\theta_t})}
	\leq
	\\ &
    \mathcal{O}\left(
    \frac{\eps_{\text{PAE}}}{1 - \gamma} + \frac{\eps_\text{model}}{1-\gamma} + \frac{1}{(1 - \gamma)\sqrt{T}}
    \right).
\end{align*}
%
%
\end{theorem}
%
%

This result shows the effect of the policy approximation error, the model error, and the number of iterations, and order notation is used to drop constants.
We observe that the error due to optimization decreases as $\mathcal{O}(\frac{1}{\sqrt{T}})$.
The policy approximation error, $\eps_{\text{PAE}}$, is similar to the function approximation term (or bias) in supervised learning, and depends on how expressive the policy space is. 
This term may not go to zero, which means that the projected PG method may not find the best policy in the class.

The model error, $\eps_\text{model}$, captures how well one can replace the PG computed according to the true dynamics $\PKernelpi$ with the learned dynamics $\PKernelpihat$. 
Its magnitude depends on how expressive the model class is, the number of samples used in minimizing the loss, etc.
Note that this error is the same as the population loss of PAML~\eqref{eq:PAML-DifferenceInGradient}. 
The PAML objective appears naturally as a factor in the convergence result of a generic MBPG method.

%


How does PAML compare to MLE as an objective for model learning? 
Recall that the MLE is the minimizer of the $\KL$-divergence between the empirical distribution of samples generated from $\PKernelpi$ and $\PKernelpihat$. 
On the other hand, the population version of PAML's loss \eqref{eq:PAML-DifferenceInGradient} is exactly the error in the PG estimates that we care about. 
We theoretically analyze the error between $\nabla_\theta J(\pi_\theta)$ obtained following the true model $\PKernelTrue$ and $\nabla_\theta \hat{J}(\pi_\theta)$ obtained following $\PKernelhat$, and relate it to the error between the models.
We only briefly report this result here, and defer its detailed description to the supplementary material.
The result states that for an exponential family parametrization of a policy (see the supplementary for a general policy parameterization),
the PG error $\smallnorm{
		\nabla_\theta J(\pi_\theta) - 
		\nabla_\theta \hat{J}(\pi_\theta)
		}_2$
can be upper bounded by
\begin{align}
	\frac{\gamma \Qmax B_2}{(1-\gamma)^2}
	\times
	\begin{cases}
		\cPG(\rho, \nu; \pi_\theta) \norm{\Delta \PKernel^{\pi_\theta}}_{1,1(\nu)}, \\
		2 \norm{\Delta \PKernel^{\pi_\theta}}_{1,\infty},
	\end{cases}
	\label{PG-upperbound-concise}
\end{align}
where $B_2$ is the $\ell_2$-norm of features used in the definition of the policy, the value function is bounded by $Q_\text{max}$, 
$\norm{\Delta \PKernel^{\pi_\theta}}_{1,1(\nu)}$ and 
$\norm{\Delta \PKernel^{\pi_\theta}}_{1,\infty}$ are total variation-based norms of the model error
$\DeltaPKernelpi = \PKernelpi - \PKernelpihat$,
and
$
	\cPG(\rho, \nu; \pi_\theta) 
	\eqdef
	\smallnorm{
		\frac	{\mathrm{d} \rhoDiscounted^{\pi_\theta} }
			{\dnu	}
		}_\infty
$ is the supremum of the Radon-Nikodym derivative of distribution $\rhoDiscounted^{\pi_\theta}$ w.r.t. $\nu$.

As we show in the supplementary, the R.H.S. ~of \eqref{PG-upperbound-concise} can be further upper-bound by a function of the $\KL$-divergence between the distributions through an application of Pinsker's inequality.
This upper bound suggests why PAML might be a more suitable approach in learning a model. An MLE-based approach tries to minimize an upper bound of an upper bound for the quantity that we care about (PG error). This consecutive upper bounding might be quite loose. On the other hand, the population version of PAML's loss \eqref{eq:PAML-DifferenceInGradient} is exactly the error in the PG estimates that we care about.

A question that may arise is that although these two losses are different, are their minimizers the same? 
In Figures~\ref{fig:PAML-GMM-visualization} and~\ref{fig:PAML-GMM-contour} we show through a simple visualization that the minimizers of PAML and $\KL$ could indeed be different. 
We illustrate that the minimizers of PAML and $\KL$ are different when we seek to estimate the expectation of a function $f$, whose weight lies mostly on a specific part of the underlying distribution. 
\begin{figure}[t]
    \centering
    \begin{subfigure}[b]{1.0\linewidth}
     \centering
        \includegraphics[width = \linewidth]{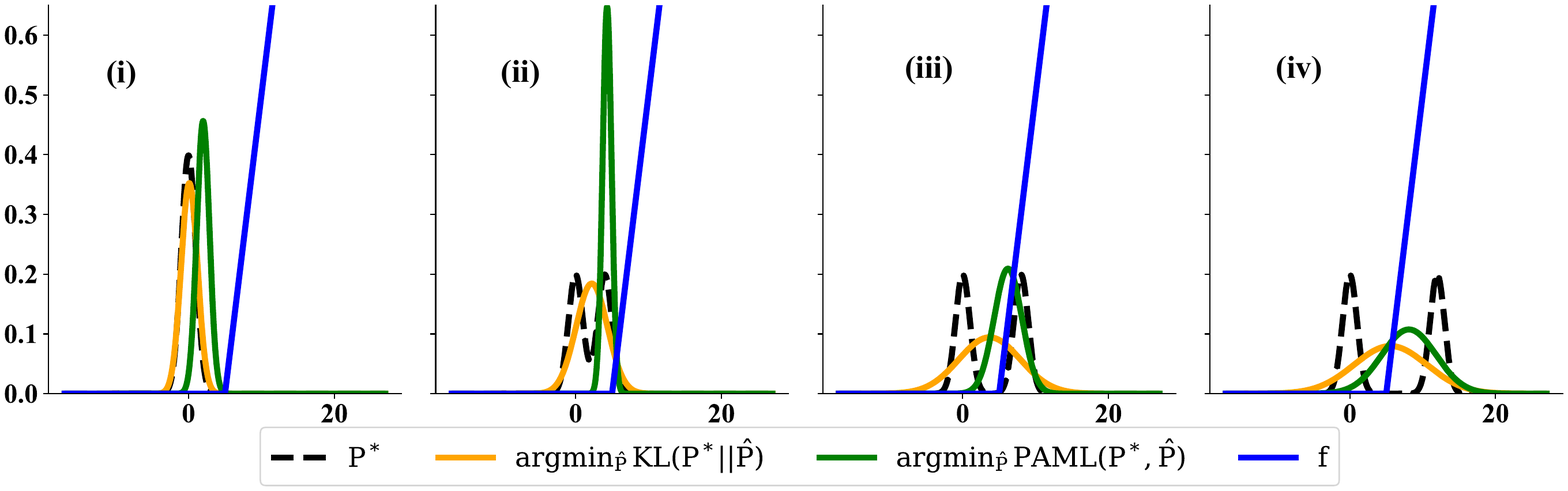}
        \caption{Visualization of minimizers for PAML and MLE. $\PKernelTrue$ is a Gaussian mixture model with 2 modes and the learned model $\PKernelhat$ is a single Gaussian. The loss minimized by PAML for this simple case is: $|\sum_x (\PKernelTrue - \PKernelhat)f(x)|^2$, where f is an arbitrary function.}
        \label{fig:PAML-GMM-visualization}%
    \end{subfigure} 
    \begin{subfigure}[b]{0.5\textwidth}
        \centering
        \includegraphics[width = \linewidth]{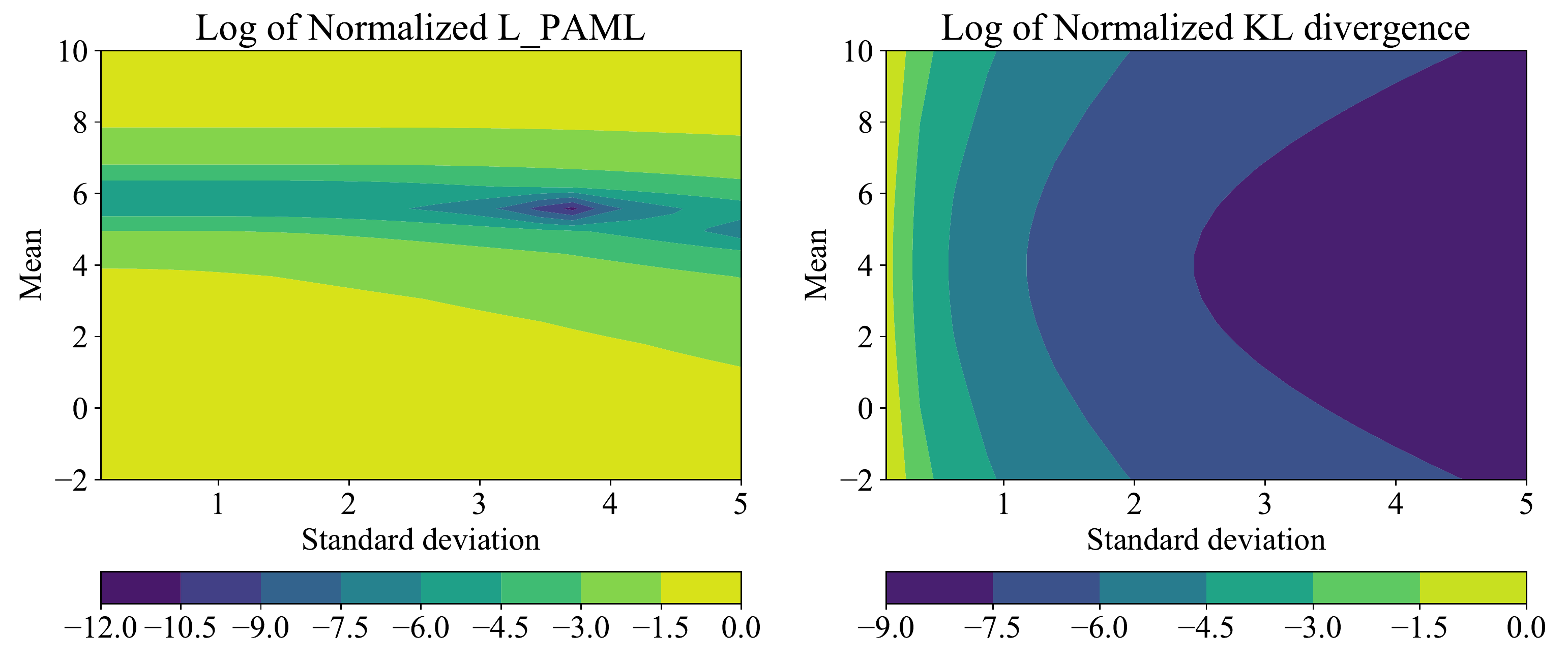}
        \caption{Contours of the two loss surfaces for Figure \ref{fig:PAML-GMM-visualization}(iii). Note that the minimizers of PAML and $\KL$-divergence are at different points. (The losses were log-normalized for better visual contrast in this figure.)}
        \label{fig:PAML-GMM-contour}%
    \end{subfigure}
\end{figure}%

\fi



\section{Empirical Studies}
\label{sec:PAML-Empirical}

We compare the performances of PAML and MLE in the framework of Algorithm~\ref{alg:PAML}. We first present an illustration of PAML and MLE for a finite-state MDP. \ifSupp We then discuss how the loss introduced in Section~\ref{sec:PAML-Algorithm} can be formulated for two PG-based planners, namely REINFORCE~\citep{Williams1992} and DDPG \citep{lillicrap2015continuous}. \else We then discuss how the loss introduced in Section~\ref{sec:PAML-Algorithm} can be formulated for PG-based planner DDPG \citep{lillicrap2015continuous}.  \fi Details for reproducing these results and additional experiments can be found in \ifSupp the Appendix \else the supplementary materials\fi.

We illustrate the difference between PAML and MLE on a finite 3-state MDP\ifSupp and 2-state MDP\else\fi.
In this setting, we can calculate exact PGs with no estimation error, and thus the exact PAML loss and $\KL$-divergence.
\ifSupp The details of the MDPs are provided in Appendix \ref{sec:PAML-Appendix}. \else \fi In these experiments, we use Projected Gradient Descent to update the model parameters and constrain their $L_2$ norm, in order to limit model capacity.
In \ifSupp Figures \ref{fig:finite_losses_comparison_3state} and \ref{fig:finite_losses_comparison_2state} \else Figure \ref{fig:finite_losses_comparison} \fi (Left two), we compare the PAML loss and $\KL$-divergence of models trained to minimize each for a fixed policy. We see that the PAML loss of a model trained to minimize PAML is (expectedly) much lower than that of a model trained to minimize $\KL$. Note that the PAML loss of the $\KL$ minimizer decreases as the constraint on model parameters is relaxed, whereas the PAML minimizer is much less dependent on model capacity.

We also evaluate the performances of policies learned using these models, in a process similar to Algorithm~\ref{alg:PAML}, but with exact values rather than sampled ones. Referring to  \ifSupp Figures \ref{fig:finite_losses_comparison_3state} (Right) and \ref{fig:finite_losses_comparison_2state} (Right) \else Figure  \ref{fig:finite_losses_comparison} (Right)\fi, as the norm of the model parameters becomes smaller, the performance of the $\KL$ agent drops much more than the PAML agent. However, when the constraint is relaxed (i.e. increased), the $\KL$ agent performs similarly to the PAML one. This example provides justification for the use of PAML: when the model space is constrained, such that it does not contain $\PKernelTrue$, PAML is able to learn a model that is more useful for planning. 
\ifSupp
\begin{figure}
    \centering
    \begin{subfigure}{0.55\linewidth}
        \centering
        \includegraphics[width=1.0\textwidth]{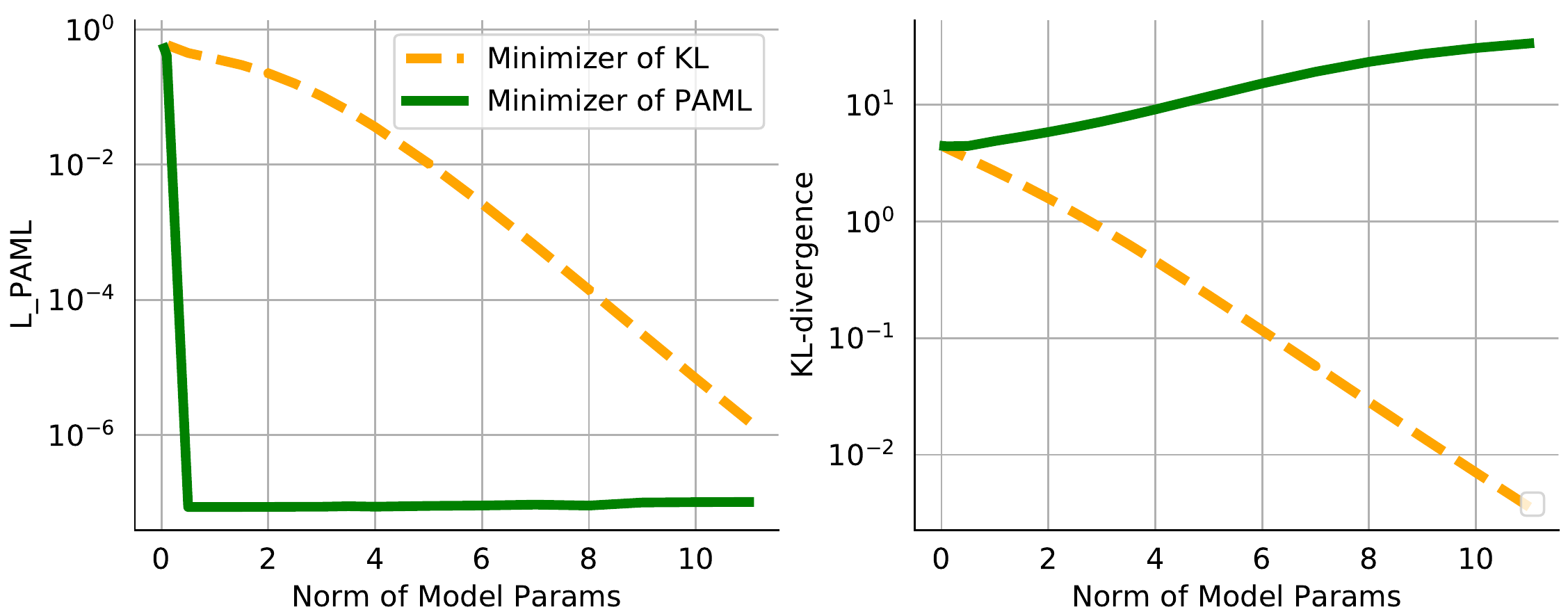}
    \end{subfigure}
    \begin{subfigure}{0.35\linewidth}
        \centering
        \includegraphics[width=1.0\textwidth]{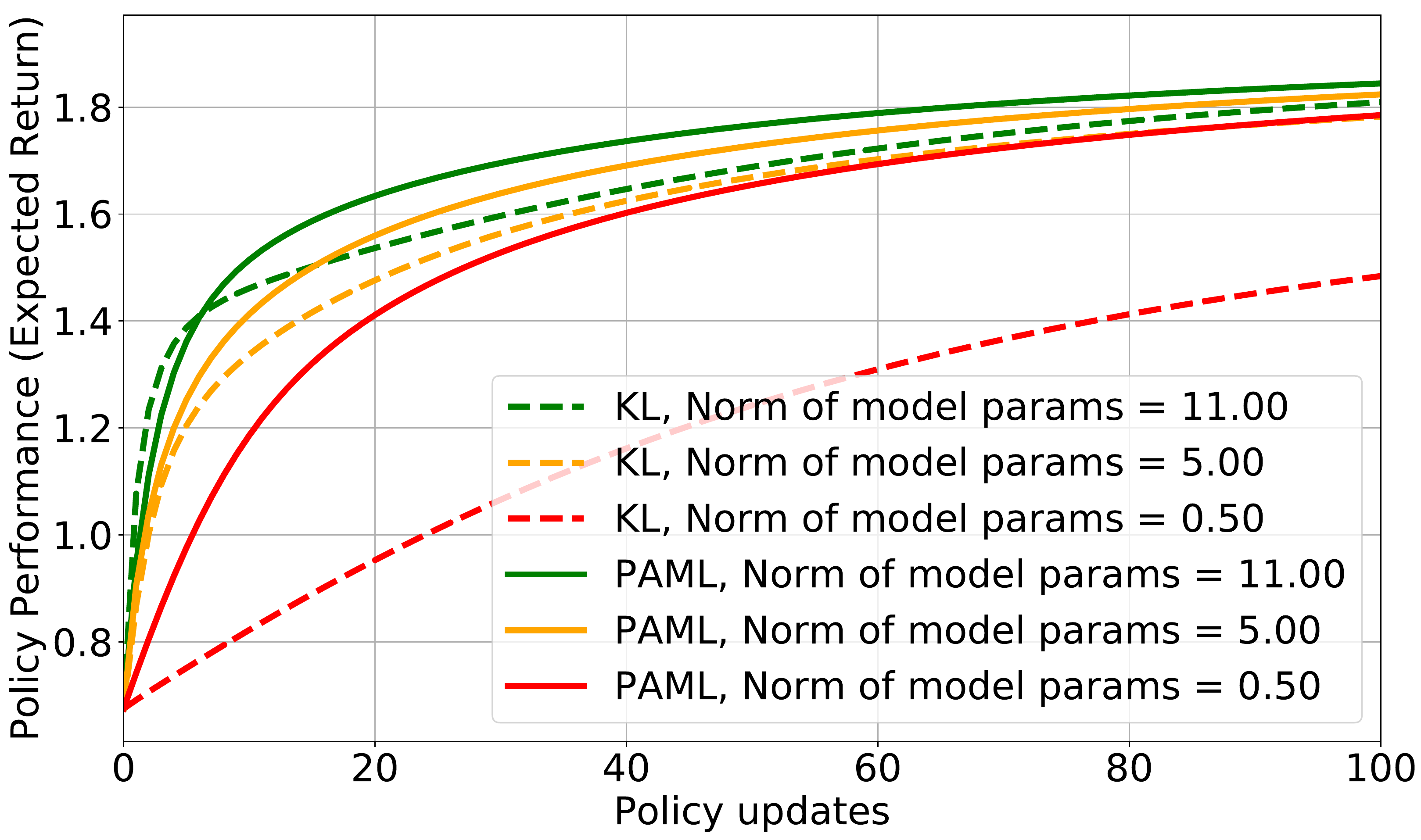}
    \end{subfigure}
\caption{Results for finite 3-state MDP defined in Appendix \ref{sec:PAML-Appendix}. (Top) Comparison of the minimizers of the PAML loss and the $\KL$-divergence as a function of the maximum allowable norm of model parameters. The true model's parameter norm would be close to $11.0$ as measured by minimizing the $\KL$ without constraints. (Bottom) Policy performance as a function of model loss and (maximum allowed) norm of model parameters. Note that there is no estimation error in this setting.}
\label{fig:finite_losses_comparison_3state}
\end{figure}
\begin{figure}
    \centering
    \begin{subfigure}{.49\linewidth}
        \centering
        \includegraphics[width=1.0\textwidth]{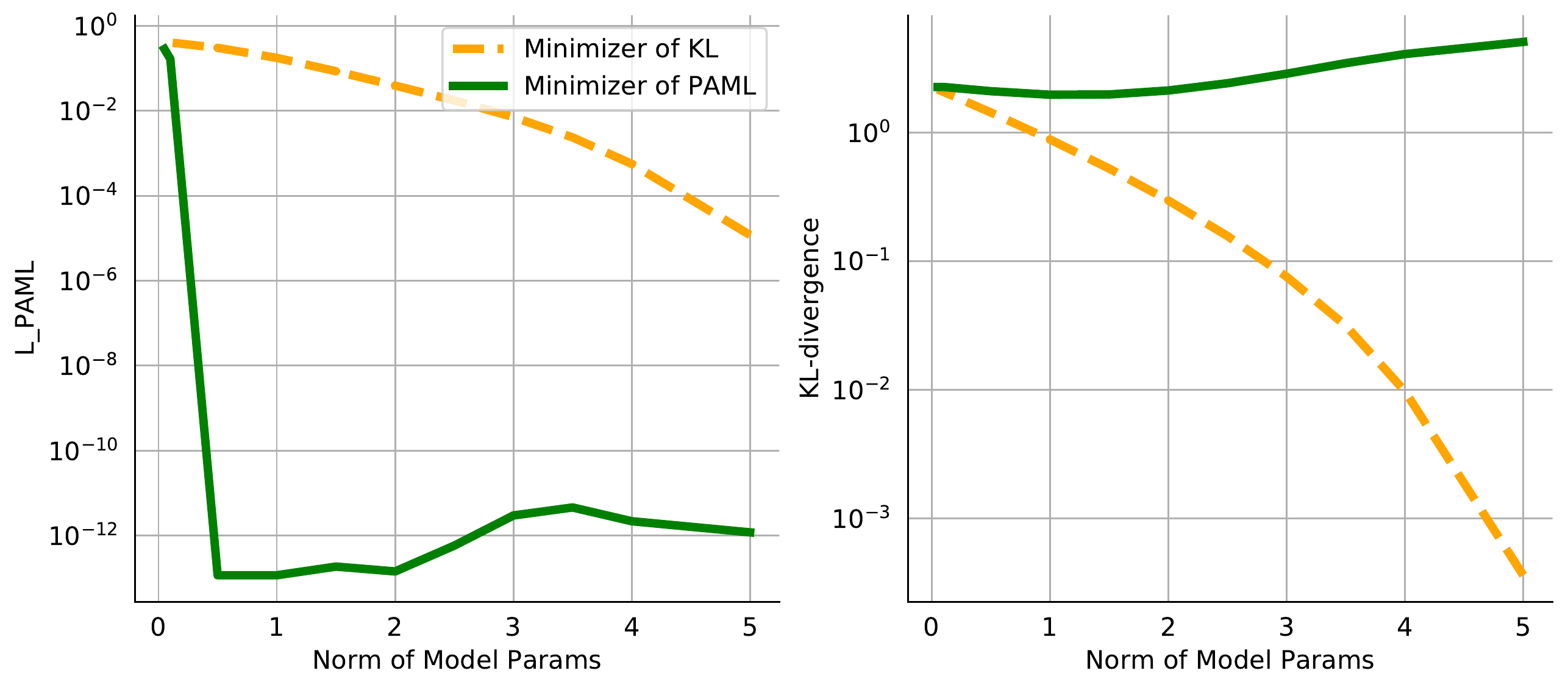}
    \end{subfigure}
    \begin{subfigure}{.49\linewidth}
        \centering
        \includegraphics[width=1.0\textwidth]{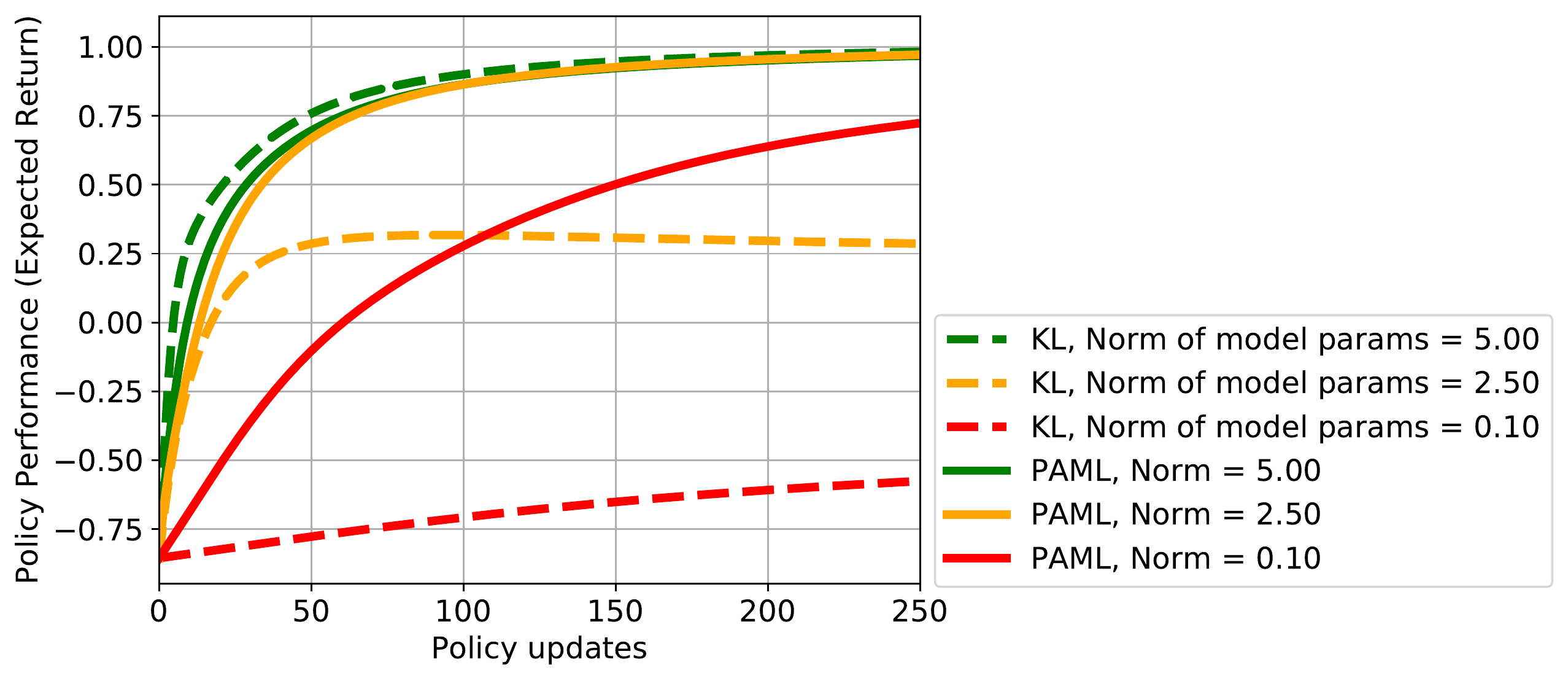}
    \end{subfigure}
\caption{Results for finite 2-state MDP defined in Appendix \ref{sec:PAML-Appendix}. (Top) Comparison of the minimizers of the PAML loss and the $\KL$-divergence as a function of the maximum allowable norm of model parameters. The true model's parameter norm would be close to $5.0$ as measured by minimizing the $\KL$ without constraints. (Bottom) Policy performance as a function of model loss and norm of model parameters. Note that there is no estimation error in this setting.}
\label{fig:finite_losses_comparison_2state}
\end{figure}
\else
\begin{figure}
    \centering
    \begin{subfigure}{0.9\linewidth}
        \centering
        \includegraphics[width=1.0\linewidth]{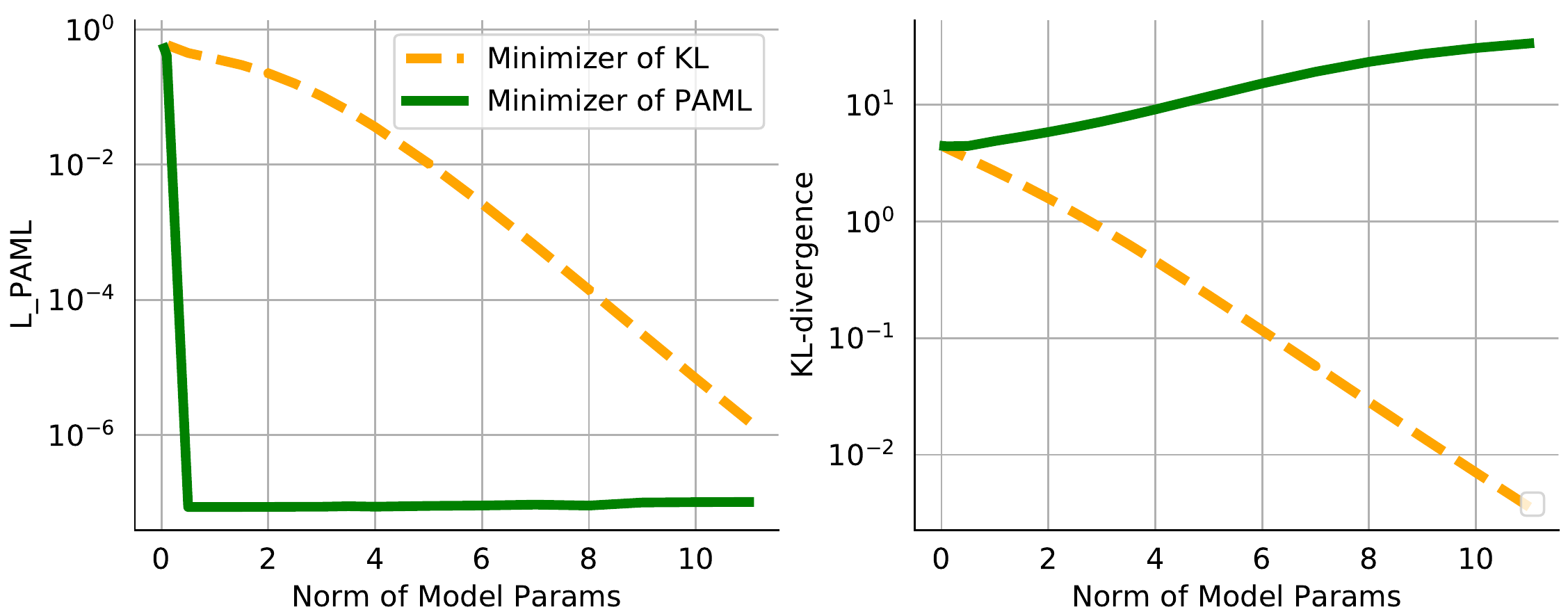}
    \end{subfigure}
    \begin{subfigure}{0.9\linewidth}
        \centering
        \includegraphics[width=1.0\linewidth]{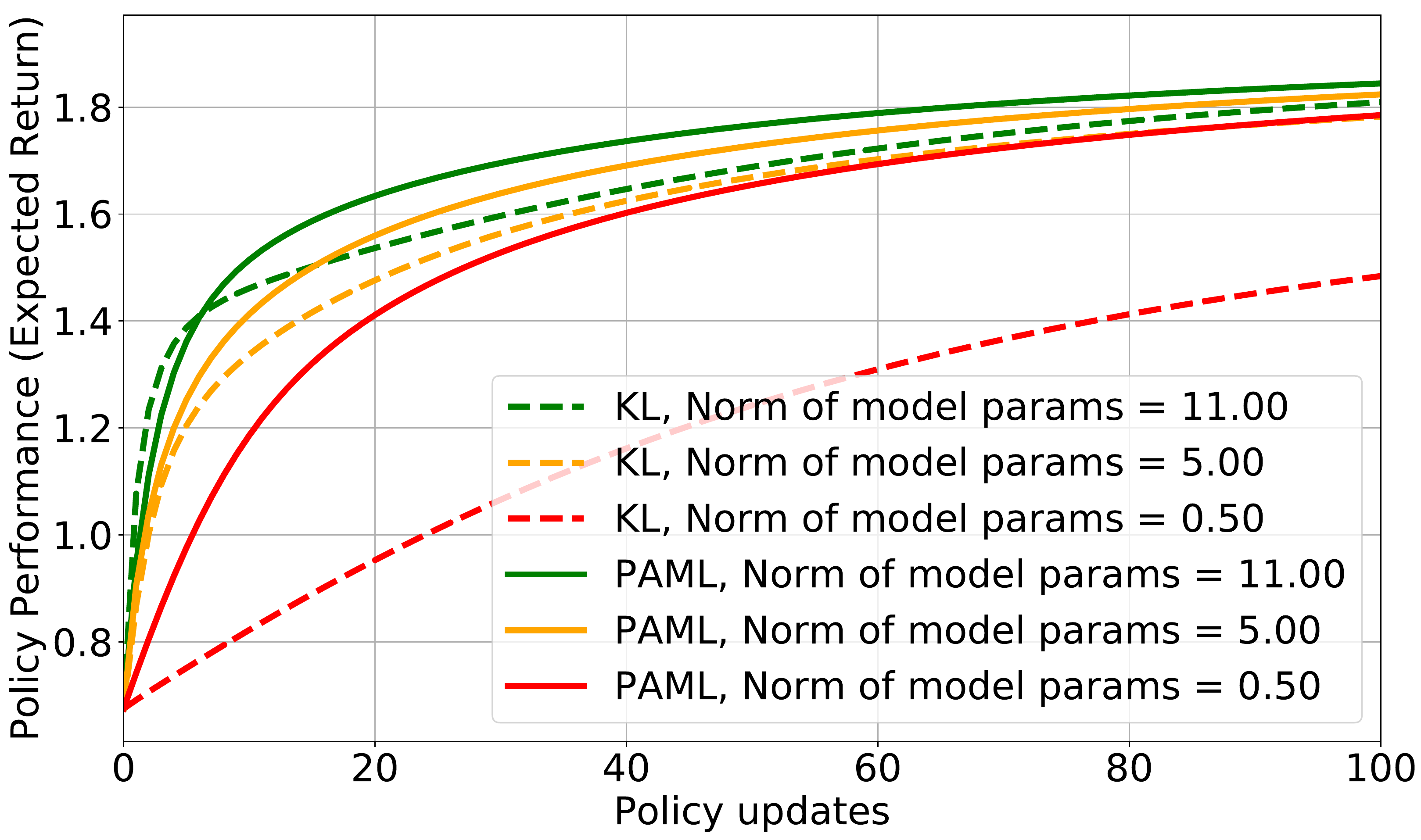}
    \end{subfigure}
\caption{{\small(Top) Comparison of the minimizers of the PAML loss and the $\KL$-divergence as a function of the maximum allowable norm of model parameters. 
Note that the minimizer of PAML does not necessarily correspond to small $\KL$-divergence.
(Bottom) Policy performance as a function of model loss and maximum norm of model parameters. Note that there is no estimation error in this setting. Performance of the PAML agents degrades much less than the performance of the $\KL$ agent when the norm of model parameters is constrained.}}
\vspace{-0.22in}
\label{fig:finite_losses_comparison}
\end{figure}
\fi
%

A question that arises is how updates on the policy affect model error since PAML is policy-aware. 
\ifSupp The Top Figures of  \ref{fig:delta-loss-finite} and \ref{fig:delta-loss-cheetah} show \else  The top of Figure \ref{fig:delta-loss-finite} shows \fi the change in $L_{\text{PAML}}$ as a result of a number of policy updates in an epoch (i.e. an epoch refers to each iteration $k$ in \ref{alg:PAML}), while keeping the model fixed in that epoch. 
\ifSupp For the HalfCheetah experiments (the setup for which is described later), timestep refers to step taken in the environment. \fi
\ifSupp The Bottom Figures of \ref{fig:delta-loss-finite} and \ref{fig:delta-loss-cheetah} show \else The bottom of Figure \ref{fig:delta-loss-finite} shows \fi the policy performance over the same epochs. 
We observe in the 3-State MDP experiments that the change in $L_\text{PAML}$ decreases as the policy performance improves. This is expected as the policy, and therefore model, converge. 
We also observe that for higher numbers of policy updates, the performance of the PAML agent does not always show consistent improvement over the $\KL$ agent, especially at the beginning of training. 
This is also expected as the PAML model is only accurate for policies similar to the policy it was trained on. 
\ifSupp We observe a similar trend for the HalfCheetah experiments. \else We observe a similar trend for the continuous control experiments in the supplementary. \fi
We see that the change in $L_\text{PAML}$ for more virtual episodes is higher.
This is expected as the gradients in this case cannot be exactly computed and so the policy is not necessarily converging to the optimal policy at each timestep, which can also be seen in the performance plots. 
Thus, the optimal number of policy updates should be tuned according to the dynamics. 
Another option is to use an objective closer to $\KL$ at the beginning of training and fine-tune with PAML as the policy improves. 
We leave exploration of this option to future work. 

\ifSupp
\begin{figure}
    \centering
    \begin{subfigure}{0.485\linewidth}
        \centering
        \includegraphics[width=1.0\linewidth]{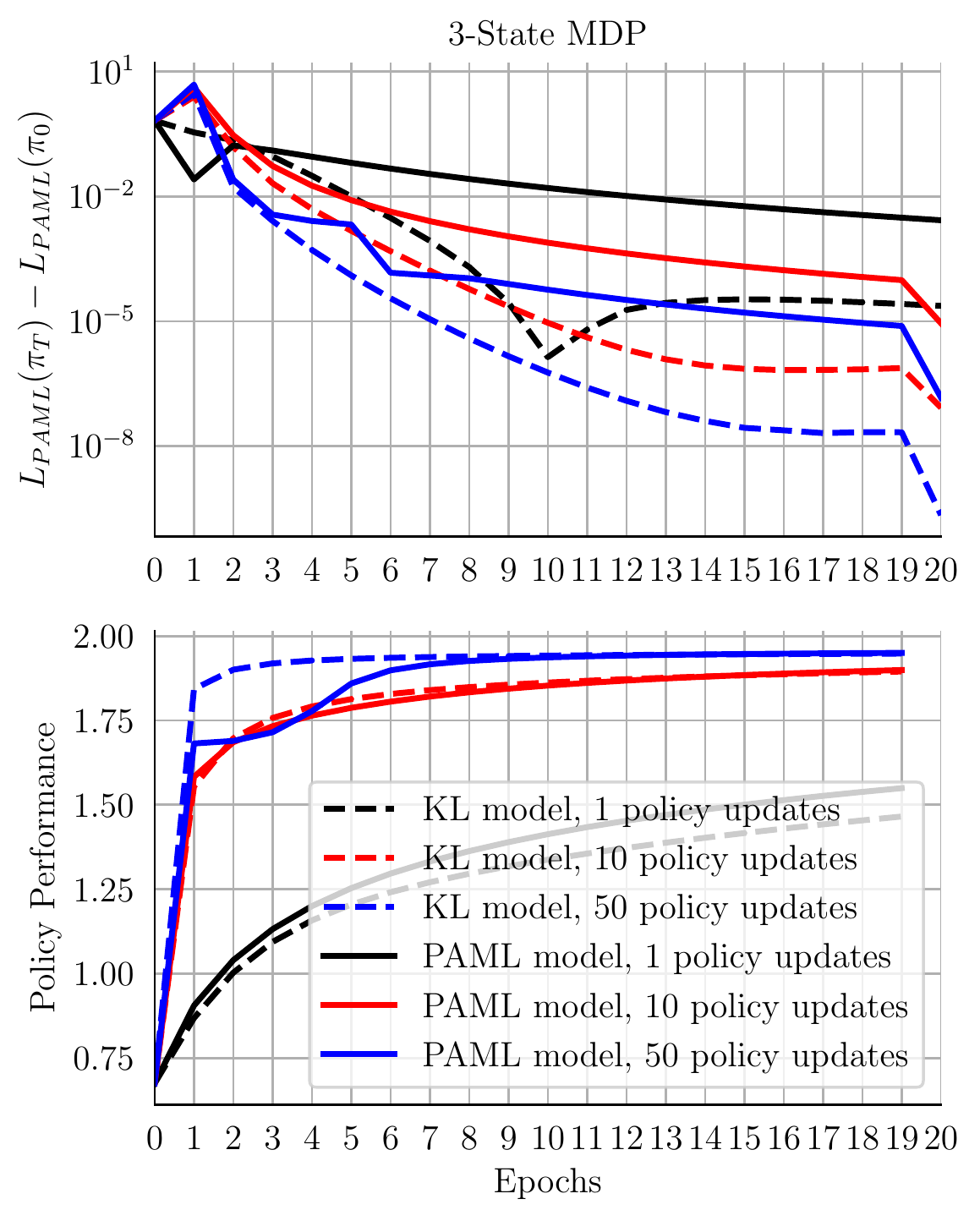}
        \caption{(Top) The model is updated for 200 steps after each set of policy updates (with learning rates described in Appendix \ref{sec:eng-details}). (Bottom) The lines for 1 policy update (black) correspond to the plots in Figure \ref{fig:finite_losses_comparison_3state}. Note that there is no source of randomness in these experiments.}
    \label{fig:delta-loss-finite}
    \end{subfigure}
    \hspace{0.5em}
    \begin{subfigure}{0.485\linewidth}
        \centering
        \includegraphics[width=1.0\linewidth]{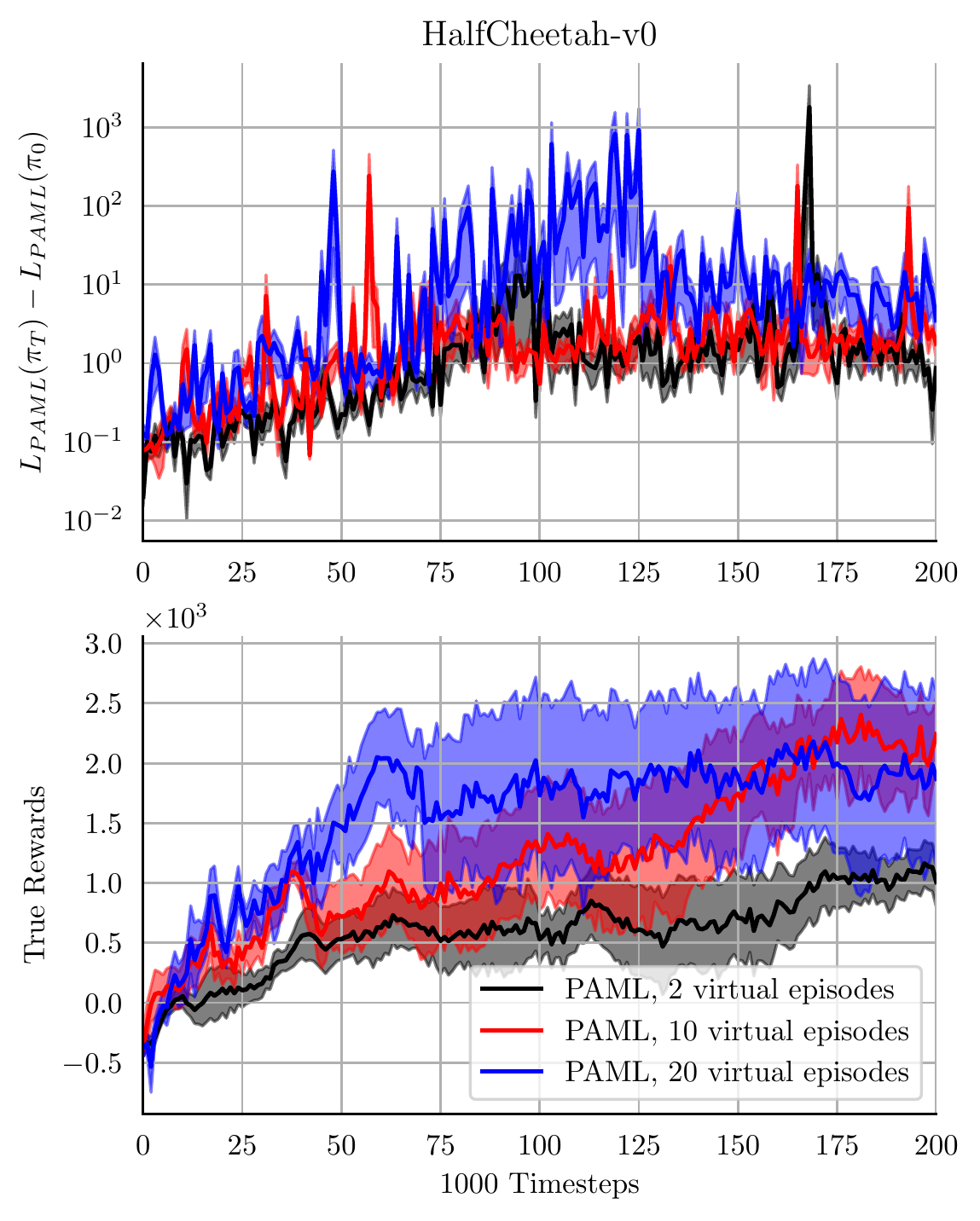}
        \caption{(Top) The number of policy updates in this environment is measured in terms of number of virtual episodes used during planning. (Bottom) The lines for 10 virtual episodes (black) correspond to the HalfCheetah plots in Figure \ref{fig:actor-critic-all-high-dim}. For both figures, the solid lines represent the mean of 5 runs and the shaded areas the standard error.} 
    \label{fig:delta-loss-cheetah}
    \end{subfigure}
\caption{(Top) Difference in PAML loss after a number of policy updates for a fixed model. Note that in this case $\pi_0$ would be the policy the model is trained for, and $\pi_T$ would be the policy with the most number of updates. (Bottom) The true policy performance corresponding to each timestep in the Top diagram.}
\label{fig:delta-loss}
\end{figure}
\else
\begin{figure}[tb]
    \centering
    \includegraphics[width=0.9\linewidth]{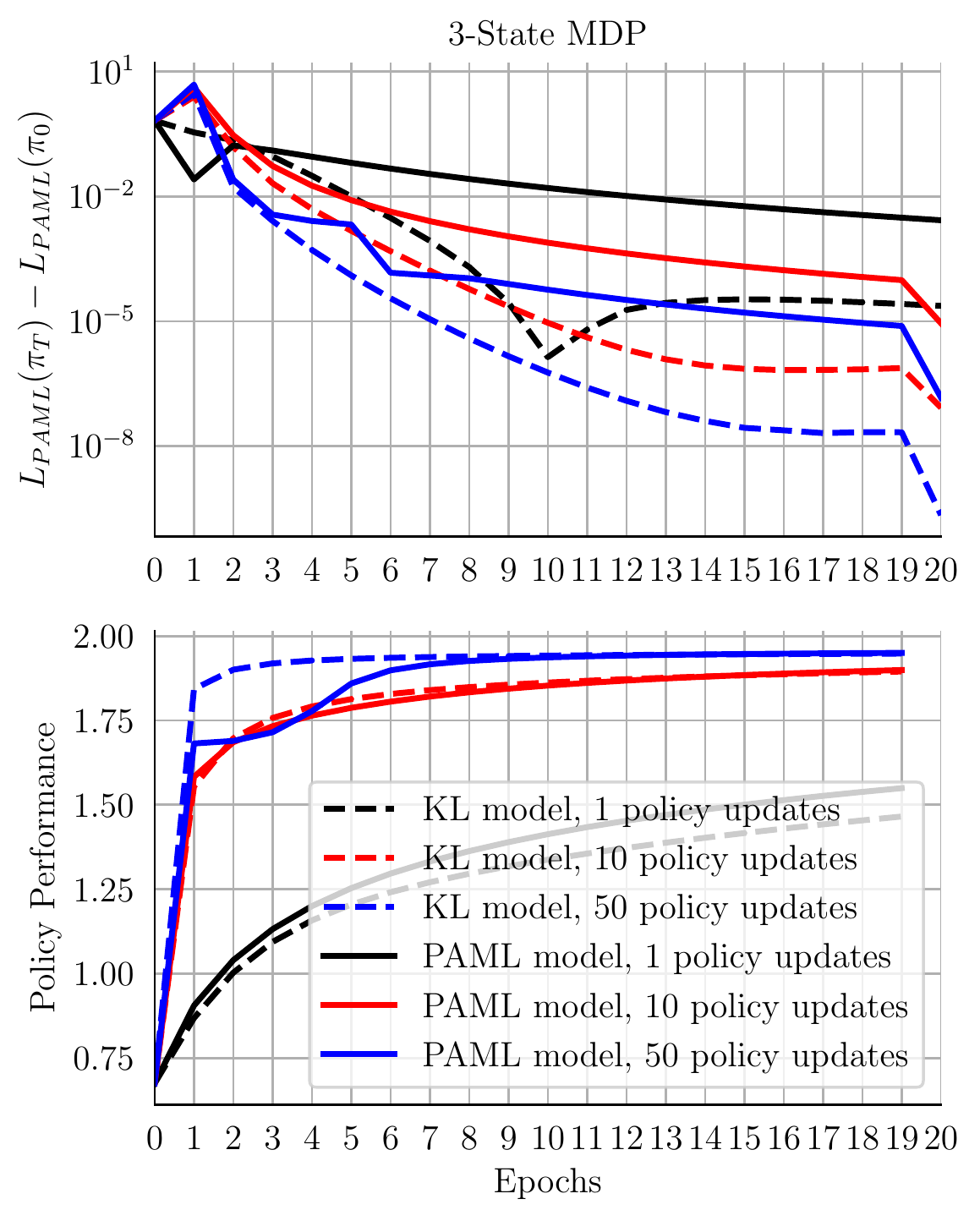}
    \caption{{\small(Top) Difference in PAML loss after a number of policy updates for a fixed model. Note that in this case $\pi_0$ would be the policy the model is trained for, and $\pi_T$ would be the policy with the most number of updates. 
    The model is updated for 200 steps after each set of policy updates (with learning rates in the supplementary). (Bottom) The true policy performance corresponding to each timestep in the Top diagram. The lines for 1 policy update (black) correspond to the plots in Figure \ref{fig:finite_losses_comparison}. Note that there is no source of randomness in these experiments.}}
    \vspace{-0.1in}
\label{fig:delta-loss-finite}
\end{figure}
\fi
We next test PAML on several continuous control environments.
\ifSupp We use the REINFORCE algorithm \citep{Williams1992} as well as the actor-critic Deep Deterministic Policy Gradient (DDPG) \citep{lillicrap2015continuous} as the planner for models learned with PAML and MLE. \else We use the actor-critic algorithm Deep Deterministic Policy Gradient (DDPG) \citep{lillicrap2015continuous} as the planner for models learned with PAML and MLE.\fi
We also evaluate the performance of the model-free method \ifSupp (REINFORCE or DDPG) \fi for reference.
\ifSupp\else Although this is a relatively old algorithm, we defer to it due to its simplicity and reasonably good performance on benchmark environments.\fi
Our goal with these experiments is not to show state-of-the-art results but rather to demonstrate the feasibility of PAML on high-dimensional problems, and show an example of how the loss in \eqref{eq:PAML-DifferenceInGradient} could be formulated. 

\ifSupp
To simulate the effect of having more dimensions in the observations than the underlying state, we concatenate to the state a vector of irrelevant or redundant information. 
For an environment that has underlying state $x_t$ at time $t$, where $x_t \in \mathbb{R}^d$, the agent's observation is given in one of the following ways\fi~\ifSupp:\else  (additional types in the supplementary):\fi
\ifSupp
\begin{enumerate}[itemsep=0mm]
     \item \textbf{Random irrelevant dimensions:} $(x_t,\eta) \in \mathbb{R}^{d+n},$ where $\eta \sim \mathcal{N}_n(0,1)$ \label{noise:rand}
     \item \textbf{Correlated irrelevant dimensions:} $(x_t,\eta_t) \in \mathbb{R}^{d+n},$ where $\eta_t = \eta_0^t \text{ and } \eta_0 \sim \text{Unif}(0,1)$. $n$ can be chosen by the user. We show results for a few cases.\label{noise:correlated}
     \item \textbf{Linear redundant dimensions:} $(x_t, W^T x_t) \in \mathbb{R}^{2d},$ where $W \sim \text{Unif}_{d\times d}(0,1)$ \label{noise:linear}
    \item \textbf{Non-linear redundant dimensions:} $(x_t, \cos(x_t), \sin(x_t))\in\mathbb{R}^{3d}$ \label{noise:nonlinear}
    \item \textbf{Non-linear and linear redundant dimensions:} $(x_t, \cos(x_t), \sin(x_t), W^T x_t) \in \mathbb{R}^{4d},$ where $W \sim \text{Unif}_{d\times d}(0,1)$ \label{noise:nonlinear-linear}
\end{enumerate}
\else
\vspace{-0.05in}
\begin{enumerate}[itemsep=0mm, wide, labelwidth=!, labelindent=0pt, partopsep=0pt,topsep=0pt,parsep=0pt]
     \item \textbf{Random irrel dims:} $(x_t,\eta) \in \mathbb{R}^{d+n},$ $\eta \sim \mathcal{N}_n(0,1)$ \label{noise:rand}
     \item \textbf{Correlated irrel dims:} $(x_t,\eta_t) \in \mathbb{R}^{d+n},$  where $\eta_t = \eta_0^t \text{ and } \eta_0 \sim \text{Unif}(0,1)$. $n$ can be chosen by the user. We show results for a few cases.\label{noise:correlated}
     \item \textbf{Linear redundant dims:} $(x_t, W^T x_t) \in \mathbb{R}^{2d},$ where $W \sim \text{Unif}_{d\times d}(0,1)$ \label{noise:linear}
    \item \textbf{Nonlinear and linear redundant dims:} $(x_t, \cos(x_t), \sin(x_t), W^T x_t) \in \mathbb{R}^{4d},$ $W$ as in type \ref{noise:linear}. \label{noise:nonlinear-linear}
    \vspace{-0.05in}
\end{enumerate}
\fi
In this way, the agent's observation vector is higher-dimensional than the underlying state, and it contains information that would not be useful for a model to learn. In the most general case, this may be replaced by the full-pixel observations, which contain more information than is necessary for solving the problem. 
To illustrate the differences between model learning methods, we choose to forgo evaluations over pixel inputs for the scope of this work. 
Although differentiating between useful state variables and irrelevant variables generated by concatenating noise may be overly simplistic (for example, a certain set of pixels could convey both useful and unnecessary information that the model may not know are unnecessary), it is an approximation that can highlight the weakness of purely predictive model learning. 
\ifSupp 
\else The DDPG planner uses a deterministic policy and explores using correlated noise~\citep{lillicrap2015continuous}.
The PAML formulation we use for this planner corresponds to case~\eqref{eq:gradients-cases-a}: the model PG is calculated using the future-state distribution of $\PKernelpihat$, and the true PG using samples from $\PKernelpiTrue$. 
The critic for both the model PG and true PG is learned using samples collected from $\PKernelpiTrue$.
Although here we demonstrate formulating the PAML loss with an actor-critic algorithm as the planner, other PG-based algorithms can also be used.  \fi

\ifSupp
    To train the model using MLE, we minimize the squared $\ell_2$ distance between predicted and true next states for time-steps $1 \leq t \leq T$, where $T$ is the length of each trajectory. The point-wise loss for time-step $t$ and episode $1 \leq i \leq n$ would then be
    \begin{equation}
        c(X_{t:t+h},\hat{X}_{t:t+h};w) = \frac{1}{N} \sum_{i=1}^N \sum_{h=1}^H \norm{(\hat{X}^{(i)}_{t+h}-\hat{X}^{(i)}_{t+h-1}) - 
        (X^{(i)}_{t+h} - X^{(i)}_{t+h-1})}_2^2,
    \end{equation}
    where $\hat{X}$ are the states predicted by the model, i.e. $\hat{X}_{t+1} \sim \PKernelpihat(\cdot|X_t, A_t)$ and similarly $X$ are the states given by the true environment, $\PKernelTrue$.
    This is a multi-step prediction loss with horizon $H$.
    Our model in all experiments is deterministic and directly predicts $\Delta \hat{X} = \hat{X}^{(i)}_{t+h}-\hat{X}^{(i)}_{t+h-1}$. 
\else
    To train a model using MLE, we minimize the squared $\ell_2$ distance between predicted and true next states for all time-steps in a trajectory. 
    Our model in all experiments is deterministic and directly predicts the change in state from the previous to the current timestep.
\fi
\ifSupp Moreover, for the REINFORCE experiments, we set $H$ to be the length of the trajectory, and for the DDPG ones we set it to $1$. \fi 

\ifSupp
Since PAML is planner-aware, the formulation of the loss changes depending on the planner used. To form the PAML loss compatible with REINFORCE as the planner, the model gradient is obtained according to case \eqref{eq:gradients-cases-b}. Namely, the model returns are calculated by unrolling $\PKernelpihat$ (and the true returns from data collected for every episode). The PG for the model is calculated on states from the real environment, which, since REINFORCE calculates full-episode returns, come only from the first states of each episode. Thus, the model and true PG's are calculated over the starting state distribution $\rho$, whereas the returns are calculated over $\PKernelpihat$ and $\PKernel^{* \pi}$ respectively.
\ifSupp In practice, we find that calculating the $\ell_2$ distance between the true PG and model PG separately for each starting state gives better results than first averaging the PGs over all starting states and then calculating the $\ell_2$ distances.\else\fi
During planning, we use the mean returns as a baseline for reducing the variance of the REINFORCE gradients.
We evaluate this formulation of the algorithm on a simple LQR problem, the details of which can be found in \ifSupp Appendix \ref{sec:PAML-Appendix}\else the supplementary materials\fi.
The extra dimensions used for these experiments were random noise, defined as type \ref{noise:rand} above.

The results for this formulation are shown in Figure~\ref{fig:reinforce} for the LQR problem with trajectories of 200 steps. The performance of agents trained with PAML, MLE, and REINFORCE (model-free) are shown over 200,000 steps. It can be seen that both model-based methods learn more slowly as irrelevant dimensions are added (the model-free method learns slowly for all cases). For no irrelevant dimensions, MLE learns faster than PAML. This is expected as in this case, an MLE model should be able to recover the underlying dynamics easily, making it a good model-learning strategy. However, as the number of irrelevant dimensions are increased, PAML shows better performance than MLE. This is encouraging as it shows that PAML is not as affected by irrelevant information.
\else \fi


\ifSupp
We now describe formulating PAML to use the DDPG algorithm as the planner. This algorithm uses a deterministic policy and explores using correlated noise~\citep{lillicrap2015continuous}. 
It is possible to use PAML with other actor-critic algorithms and we use DDPG due to its simplicity. 
We leave experiments with other actor-critic policies to future work.
In our MBRL loop, after every iteration of data collection from the environment, the critic is trained on true data by minimizing the mean-squared temporal difference error, using a target critic and policy that are soft-updated as shown in \citet{lillicrap2015continuous}. 
\ifSupp In contrast to the REINFORCE formulation, we calculate the $\ell_2$-distance between the model PG and true PG averaged over all states, rather than separately for each starting state. \else\fi
%
%
In addition, for this formulation, we present experiments for no extra dimensions added to the observations, and also for extra dimensions of types \ref{noise:correlated}, \ref{noise:linear}, \ifSupp \ref{noise:nonlinear} \else\fi and \ref{noise:nonlinear-linear} added.\else\fi

\ifSupp The results for no added dimensions are shown in Figure~\ref{fig:actor-critic-all-high-dim}. For one of the environments we also show the effect of added noise dimensions in Figure~\ref{fig:actor-critic-pend}\else The results are shown in  Figure~\ref{fig:actor-critic-all-high-dim}\fi. 
In general, PAML performs similarly to MLE in these domains. It seems that the gains that were observed in the tabular domain do not transfer to these domains. This could be due to several factors. For example, it is not clear how to limit the capacity of neural networks as we did for the experiments in \ifSupp  Figures \ref{fig:finite_losses_comparison_3state} and \ref{fig:finite_losses_comparison_2state}\else Figure \ref{fig:finite_losses_comparison}\fi.
Another reason could be that for the planning horizon used (1), the MLE model performs sufficiently well to hide any differences between the models.
\ifSupp\else Additional experiments are in the supplementary materials.\fi

\ifSupp
\begin{figure}[tb]
    \centering
    \includegraphics[width=0.48\textwidth]{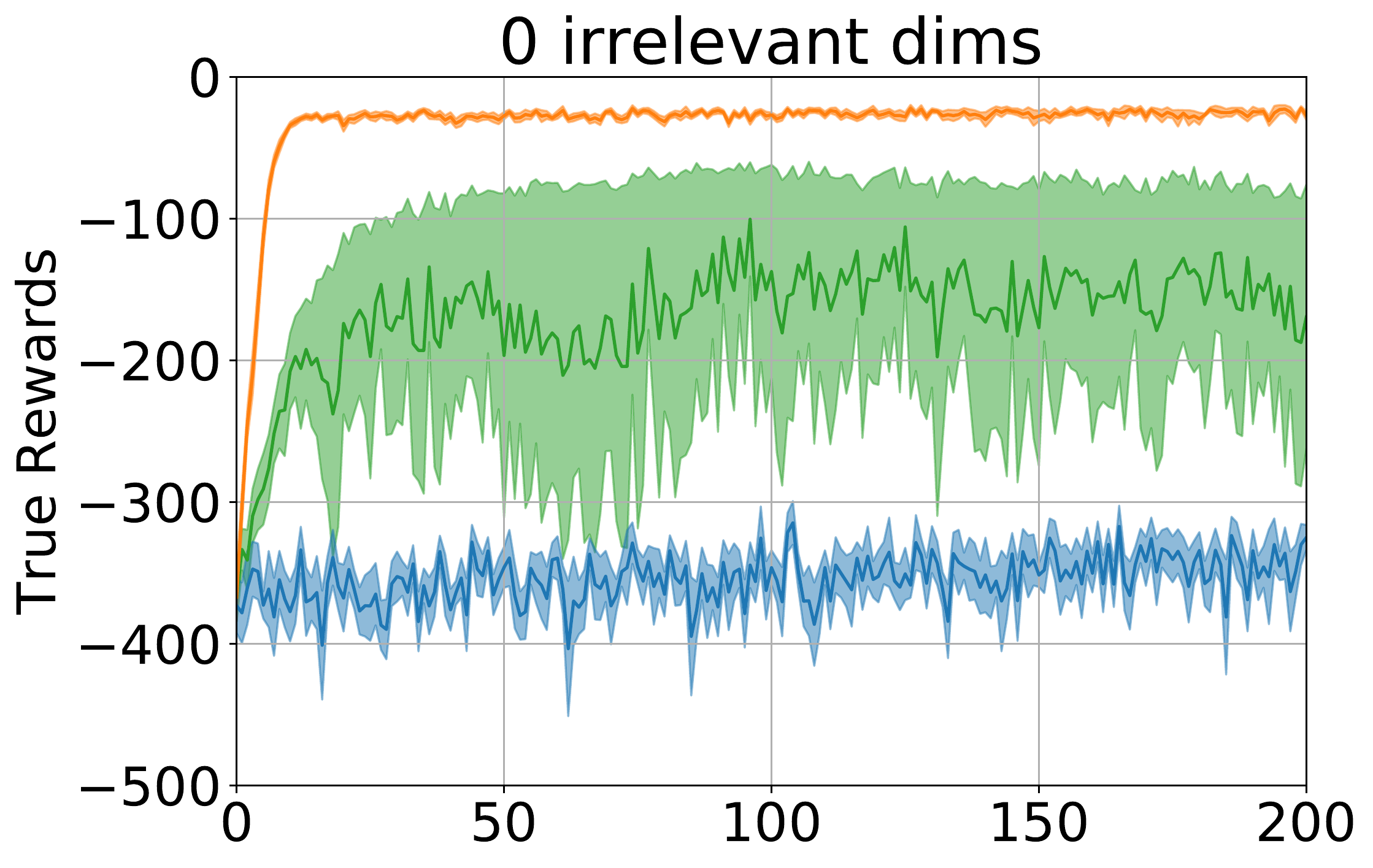}
    \includegraphics[width=0.45\textwidth]{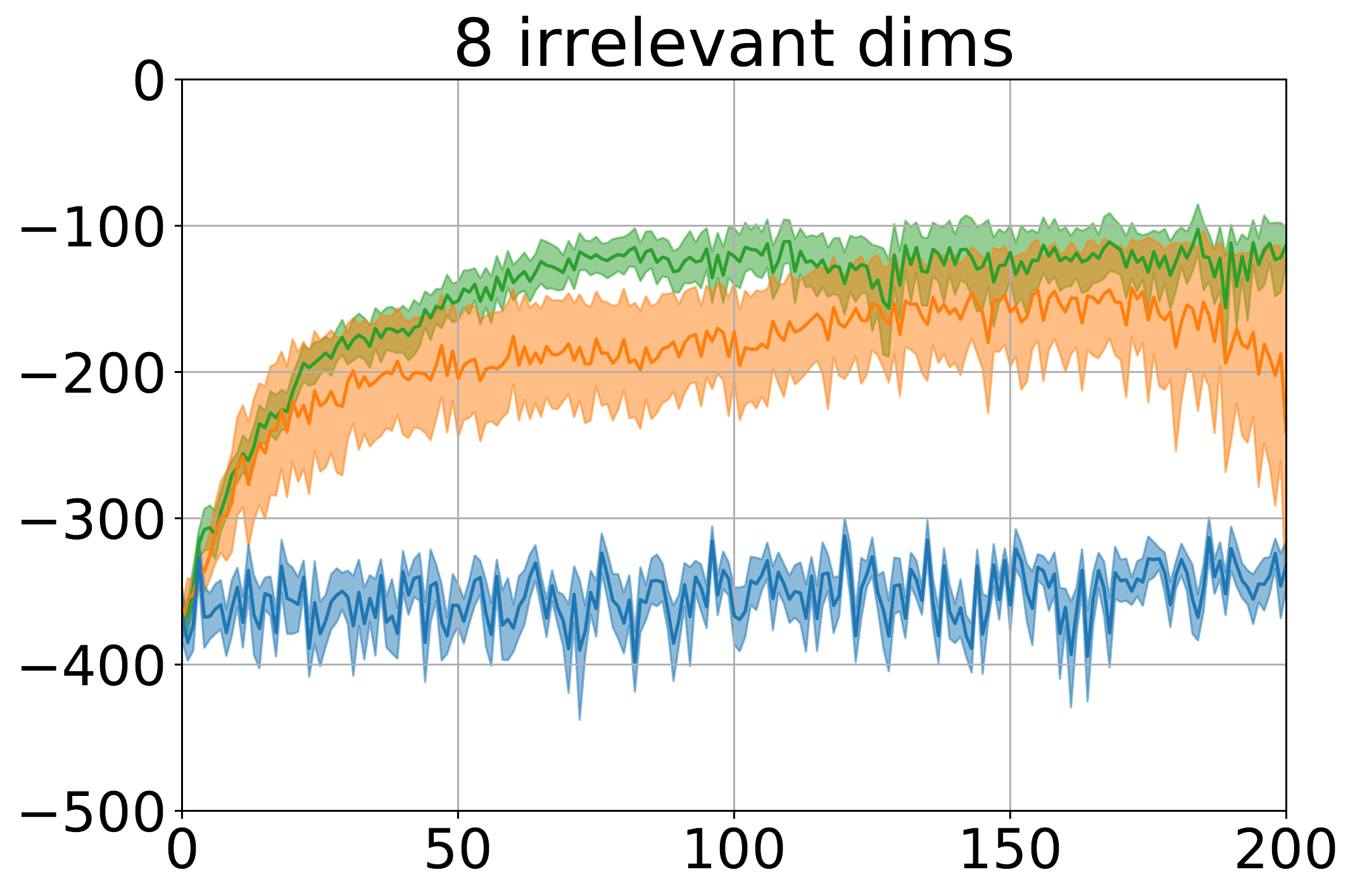} 
    \includegraphics[width=0.48\textwidth]{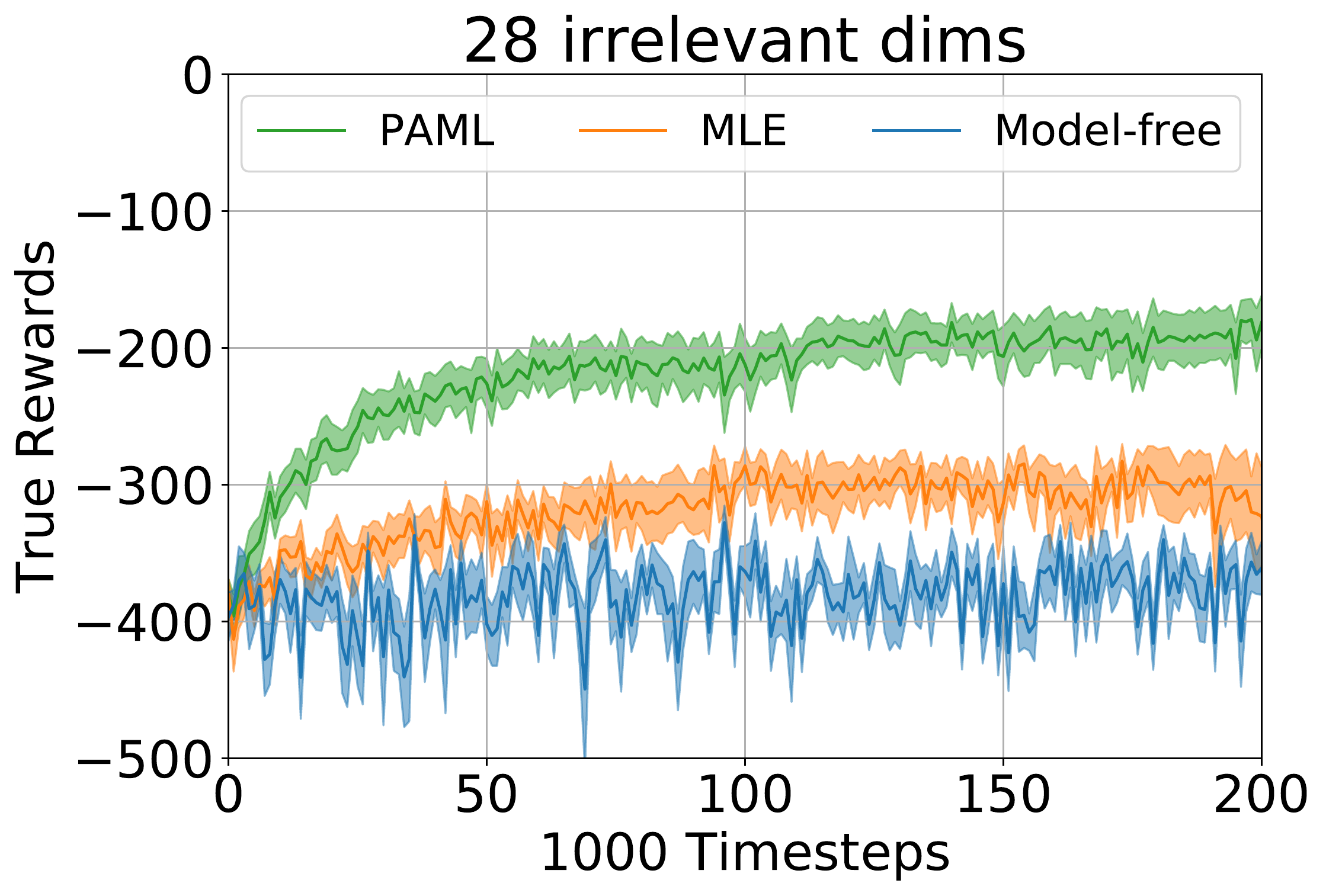}
    \includegraphics[width=0.45\textwidth]{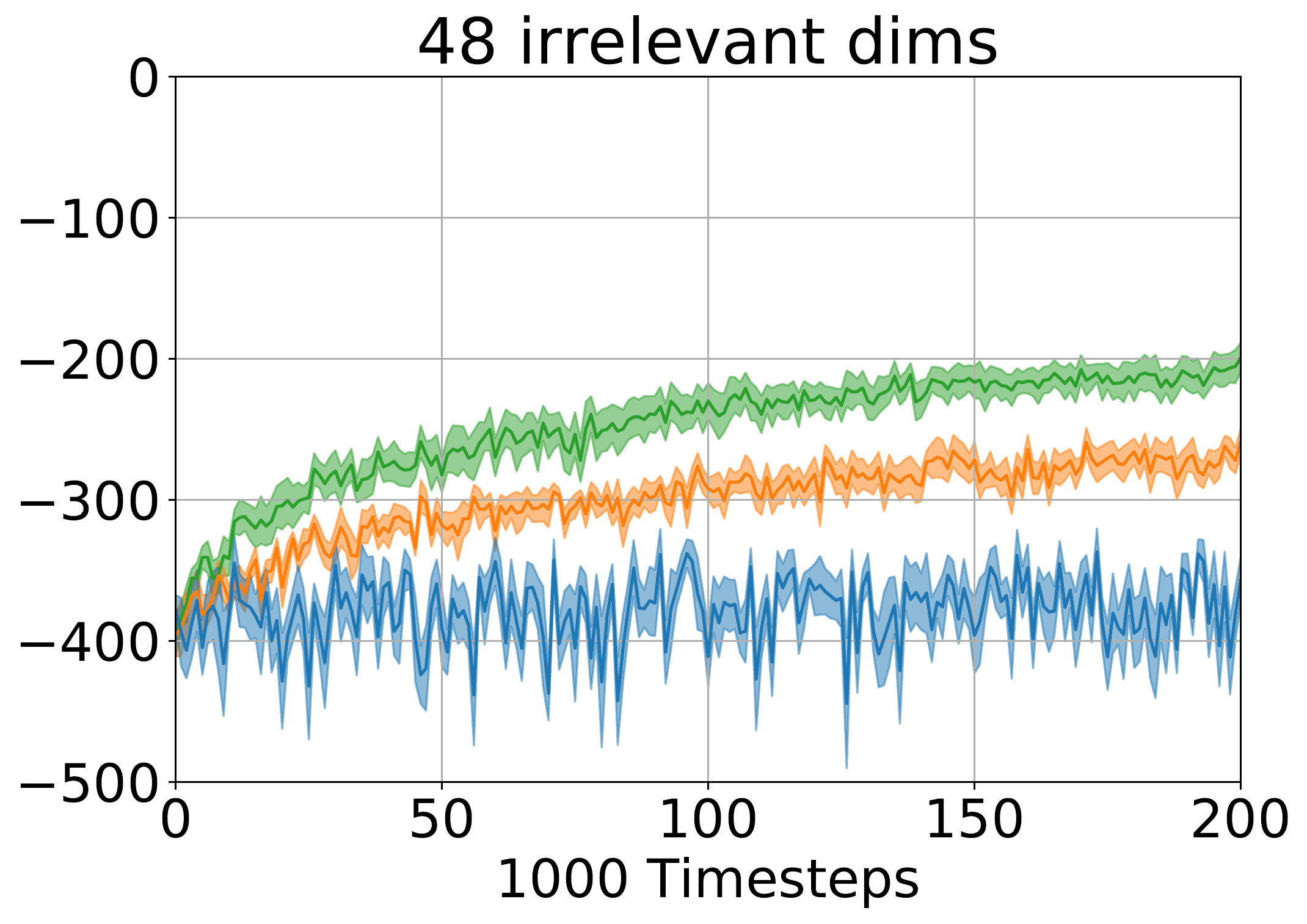}
    \caption{Performance of policies trained model-based with REINFORCE as the planner for different numbers of irrelevant dimensions added to the state observations. The solid lines indicate the mean of 10 runs and the shaded regions depict the standard error.}
    \label{fig:reinforce}
\end{figure}
\begin{figure}[tb]
    \centering
    \includegraphics[width=0.3\textwidth]{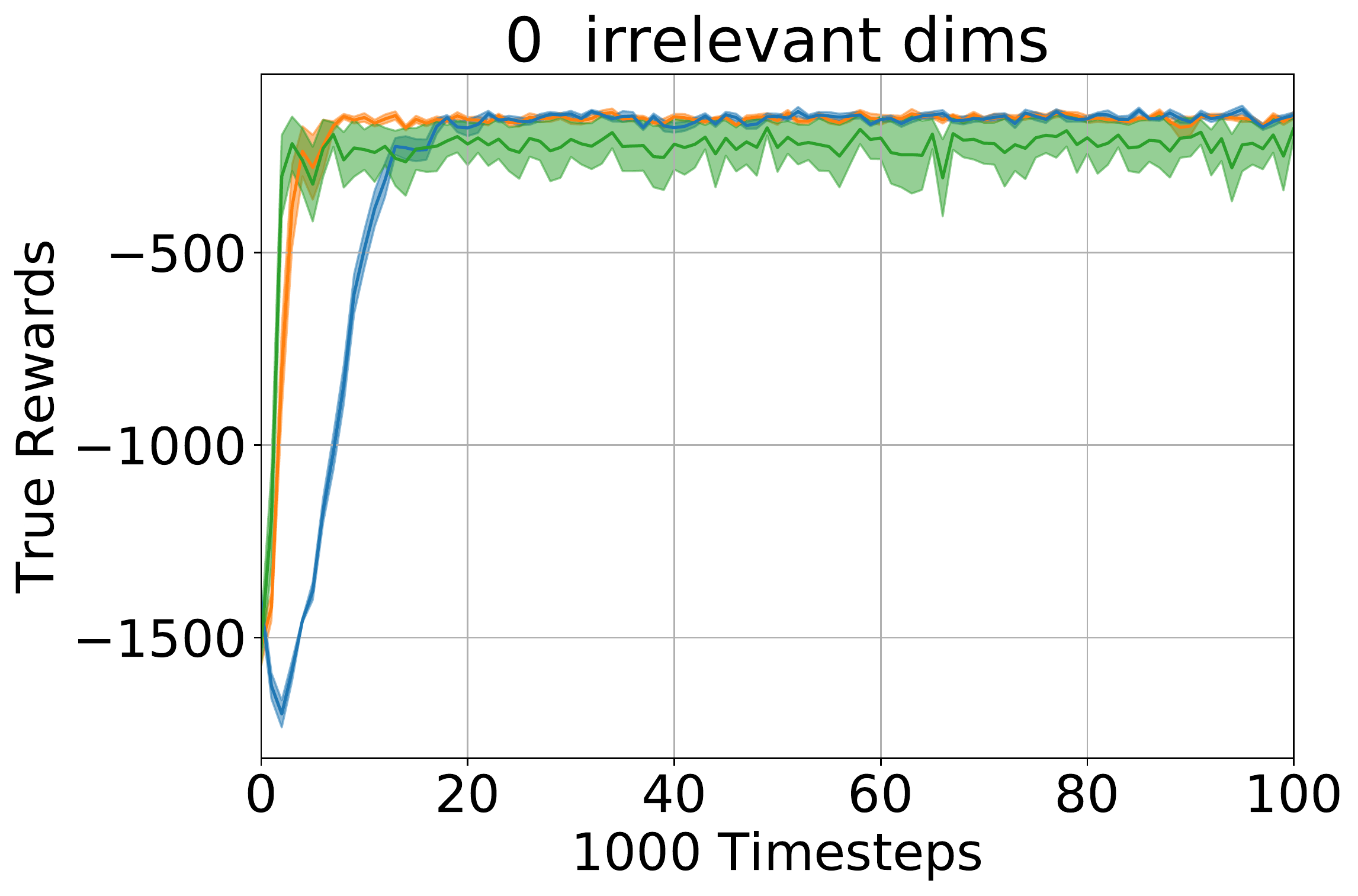}
    \includegraphics[width=0.3\textwidth]{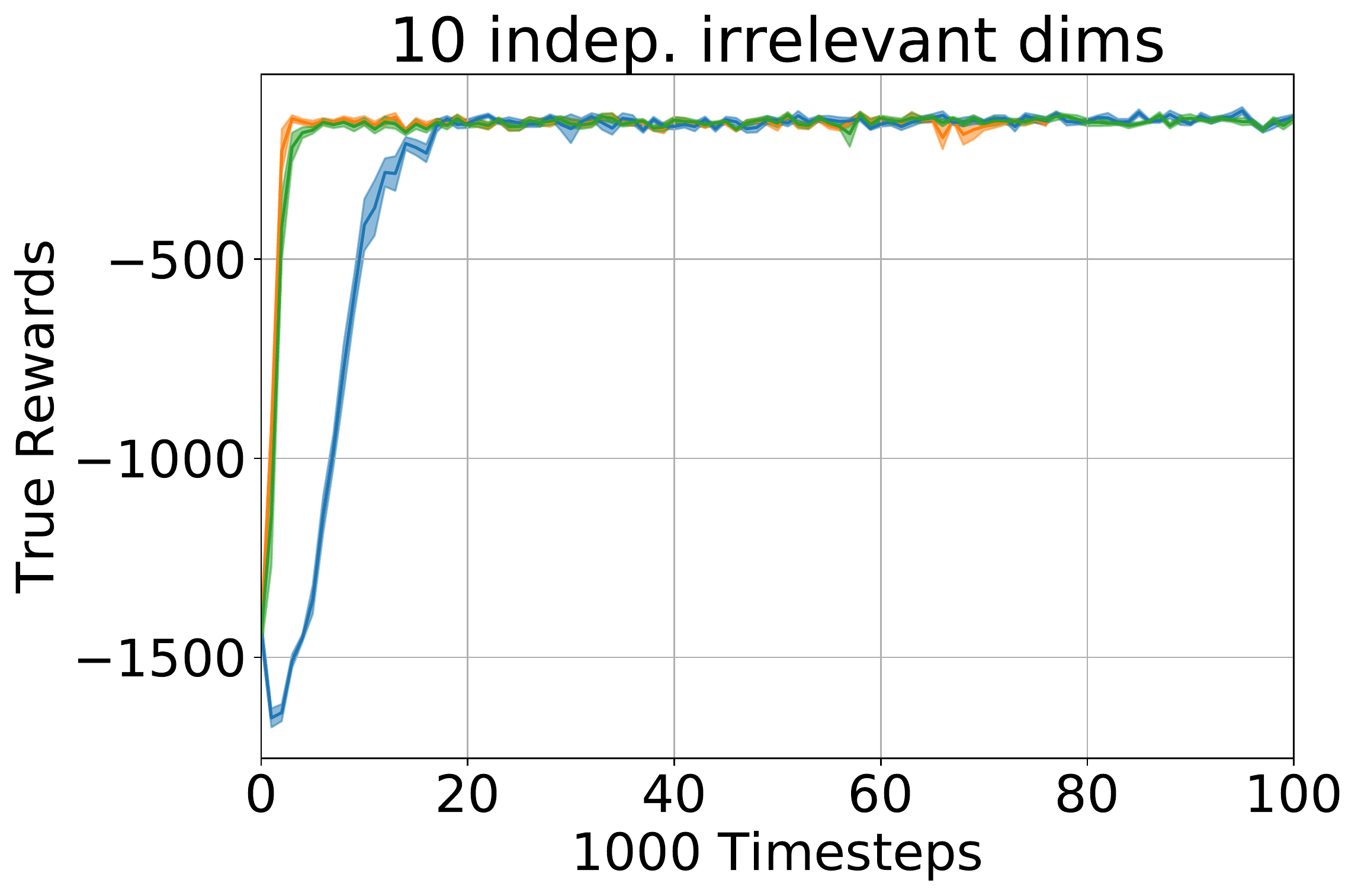}
    \includegraphics[width=0.3\textwidth]{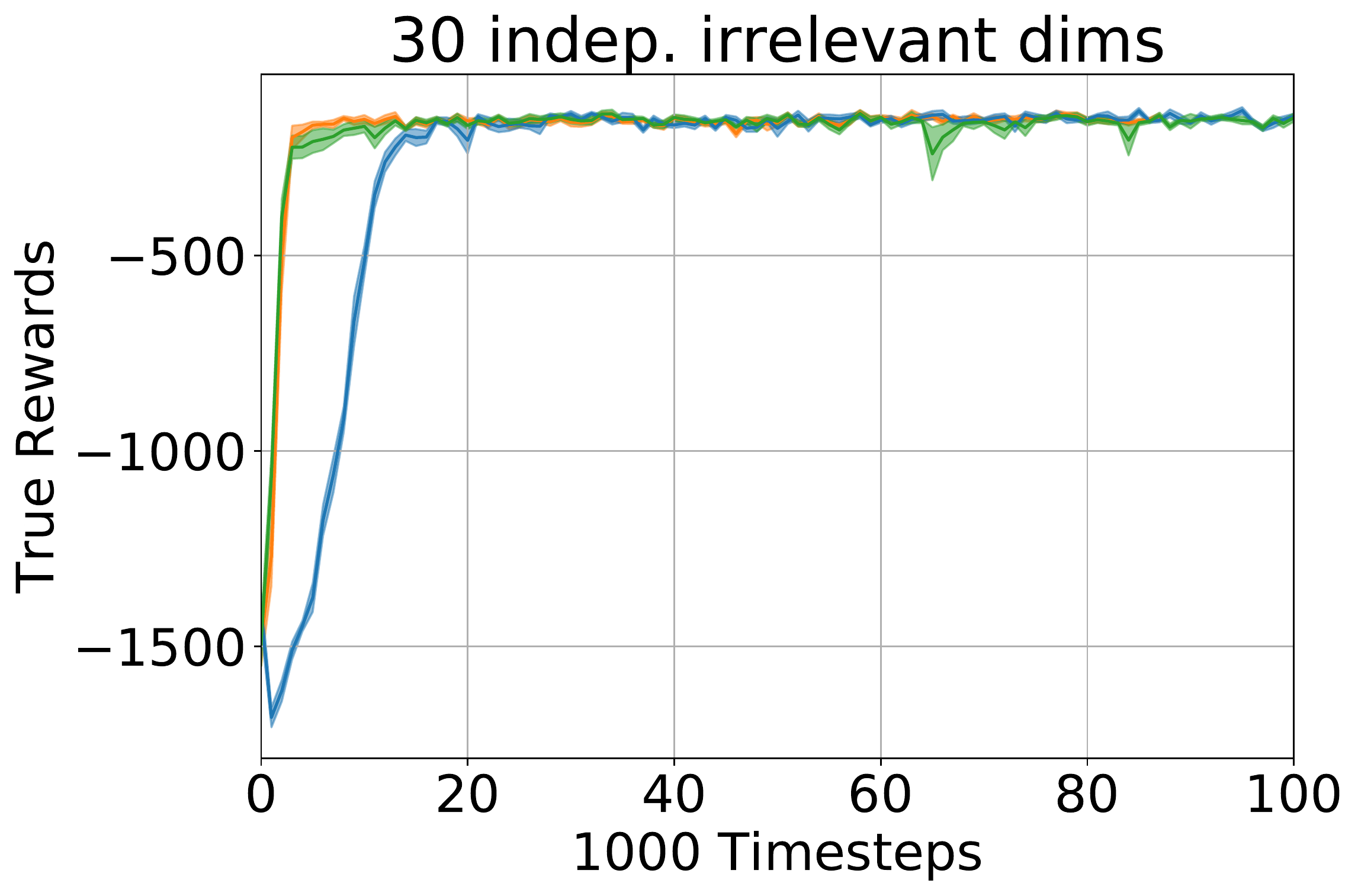}
    \includegraphics[width=0.3\textwidth]{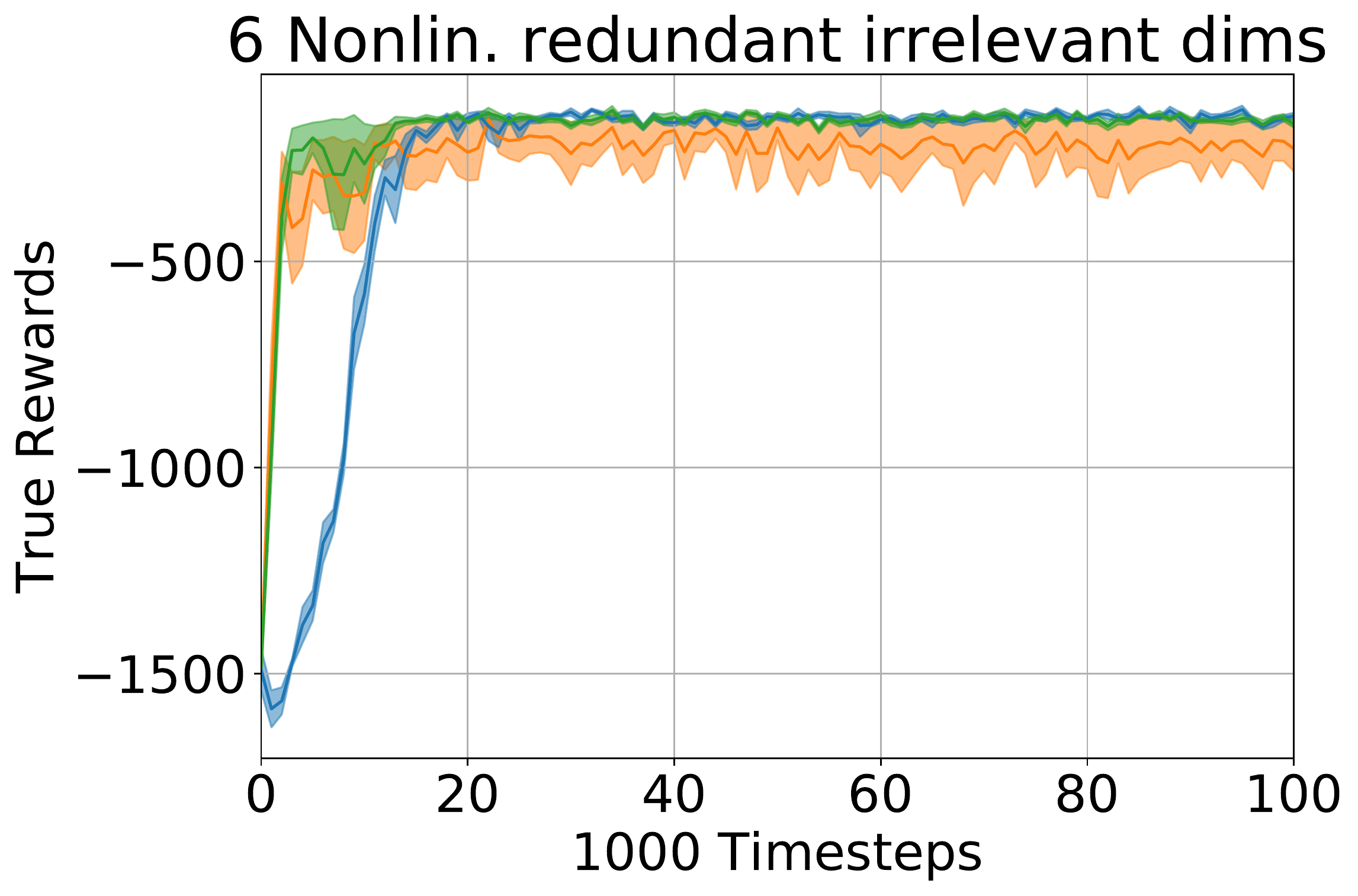}
    \includegraphics[width=0.32\textwidth]{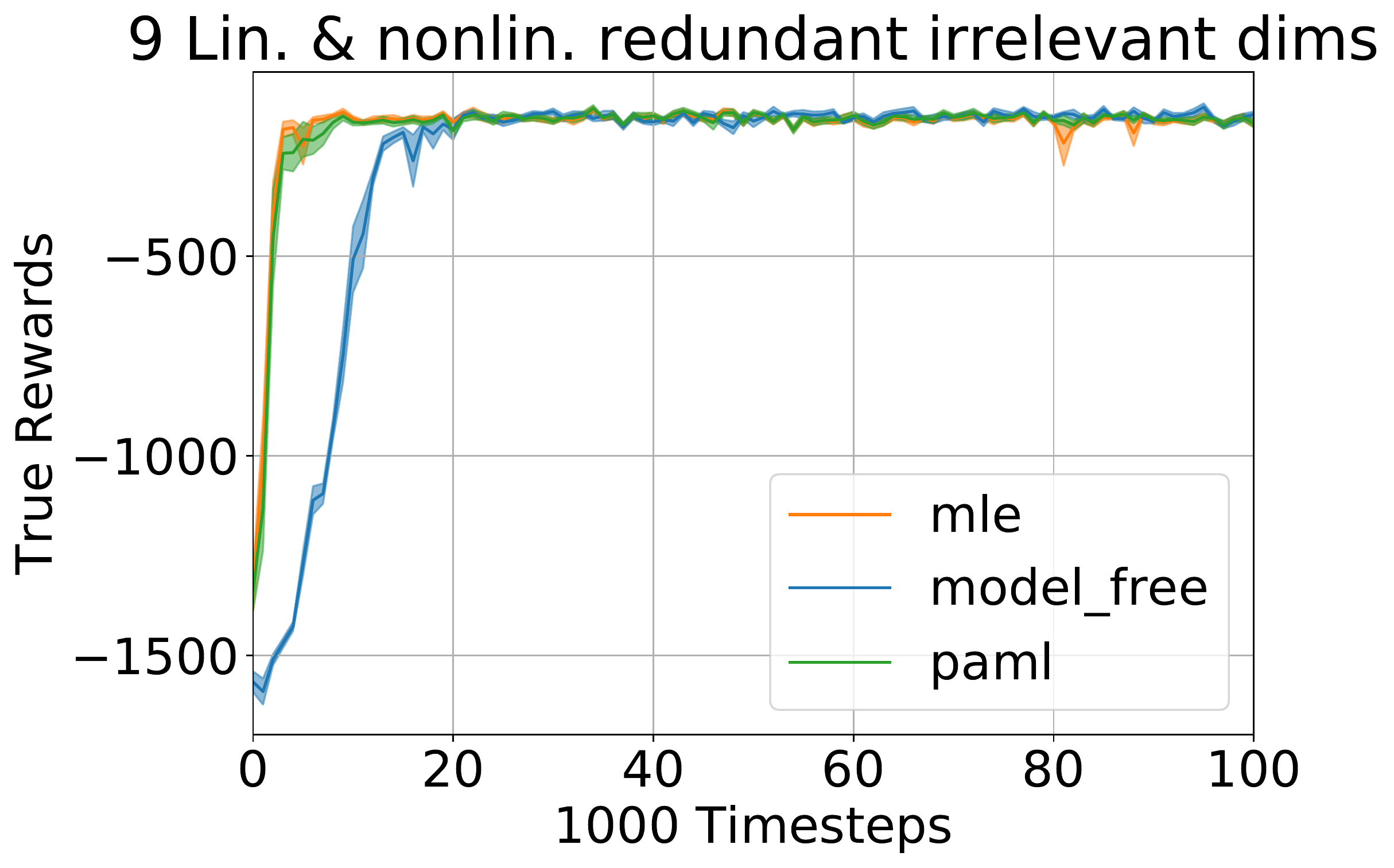}
    \caption{Performance of policies for different numbers of irrelevant and redundant dimensions added to the state observations for the locomotion problem Pendulum-v0, the details of which can be found in the \ifSupp Appendix \else supplementary materials. \fi}
    \label{fig:actor-critic-pend}
\end{figure}
\begin{figure}
\centering
    \includegraphics[width=0.48\linewidth]{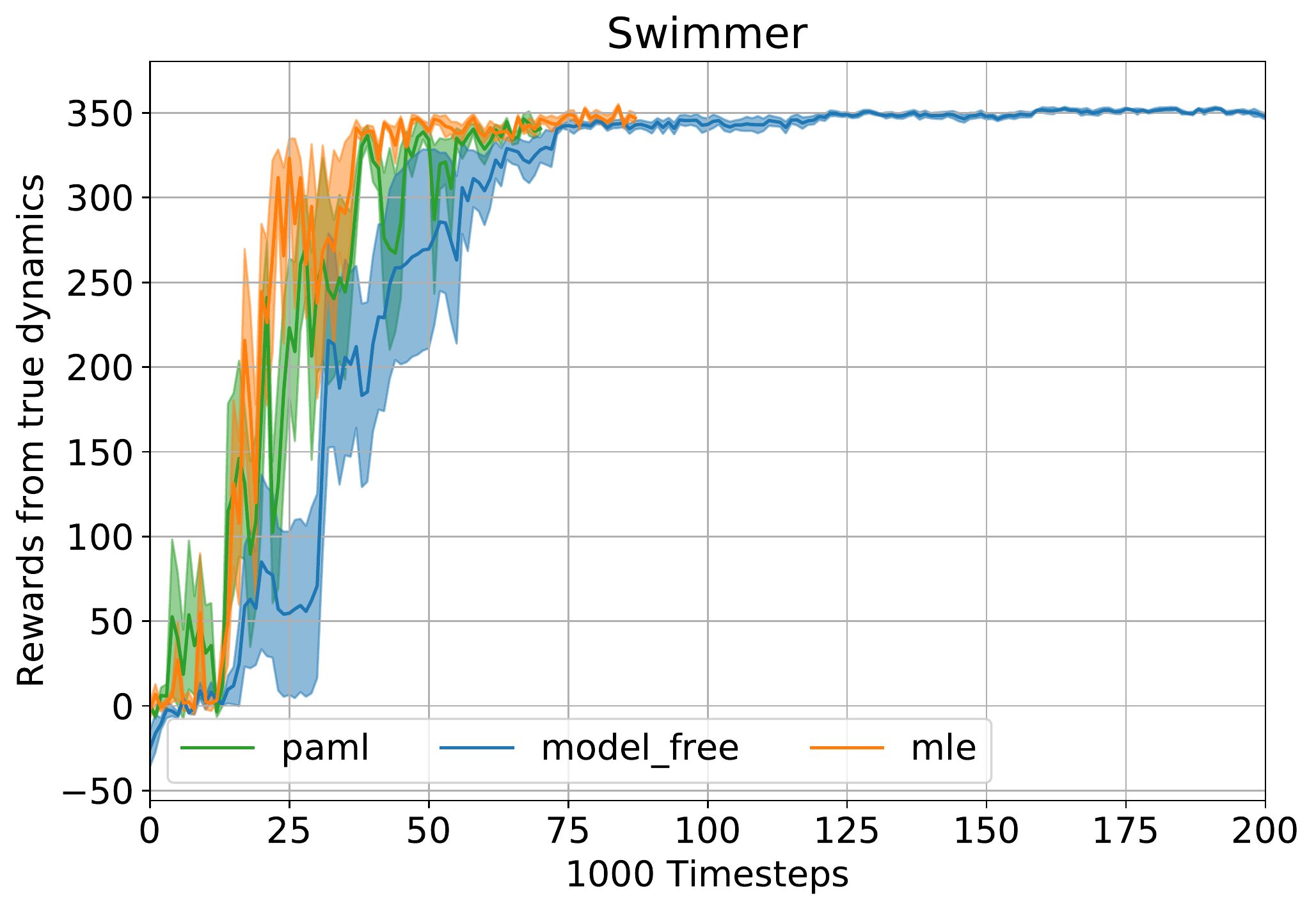}
    \includegraphics[width=0.48\linewidth]{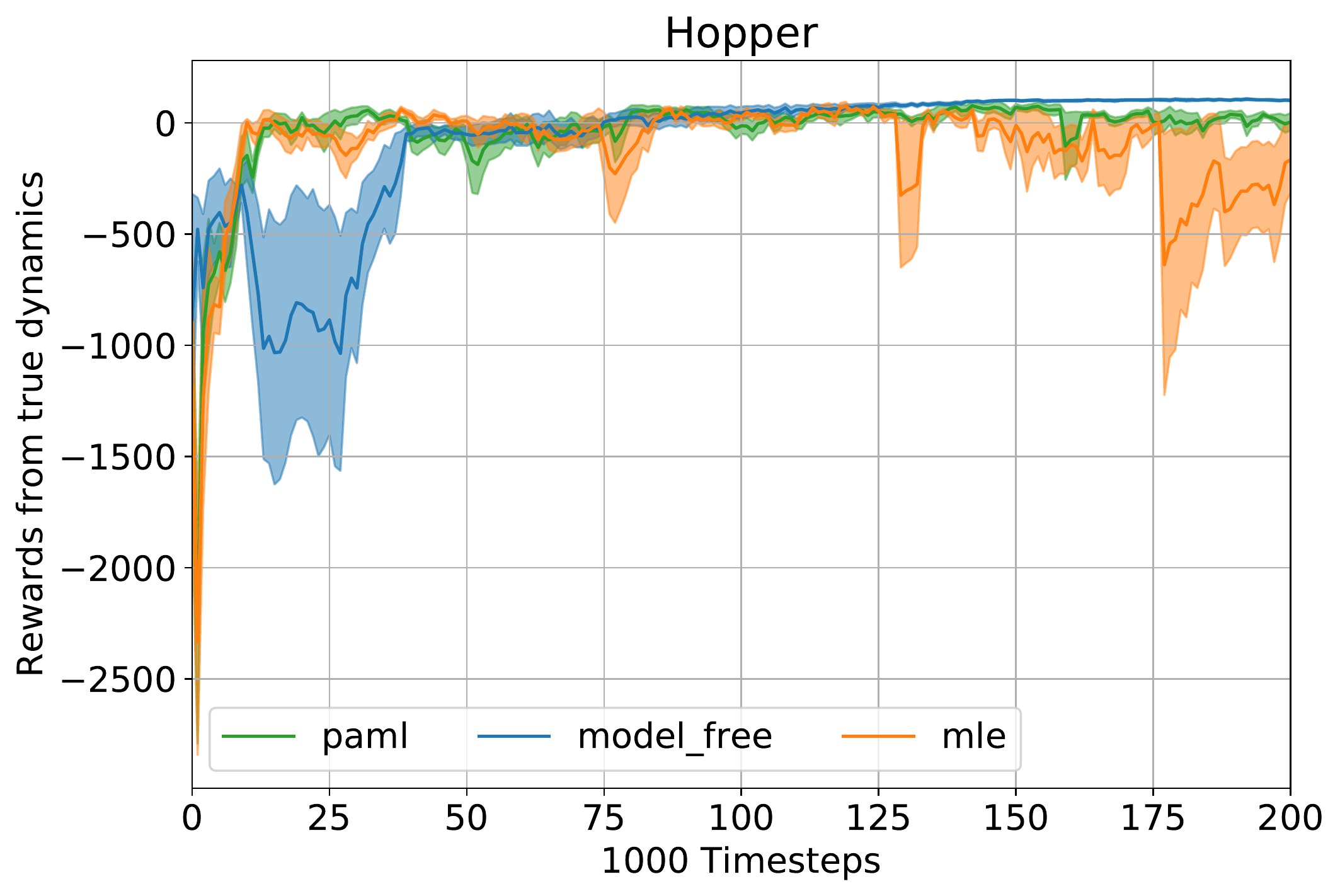}
    \includegraphics[width=0.48\linewidth]{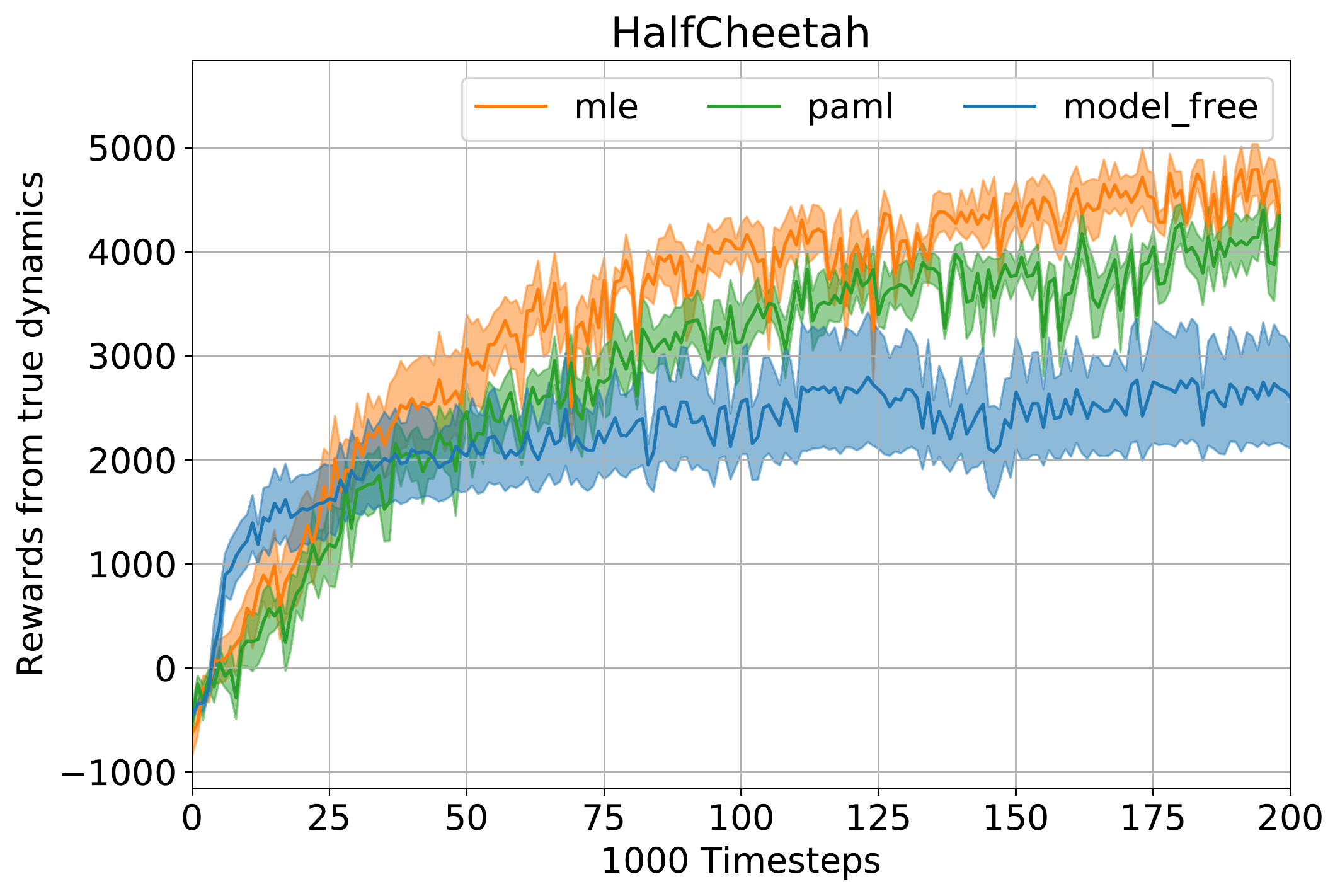}
    \includegraphics[width=0.48\linewidth]{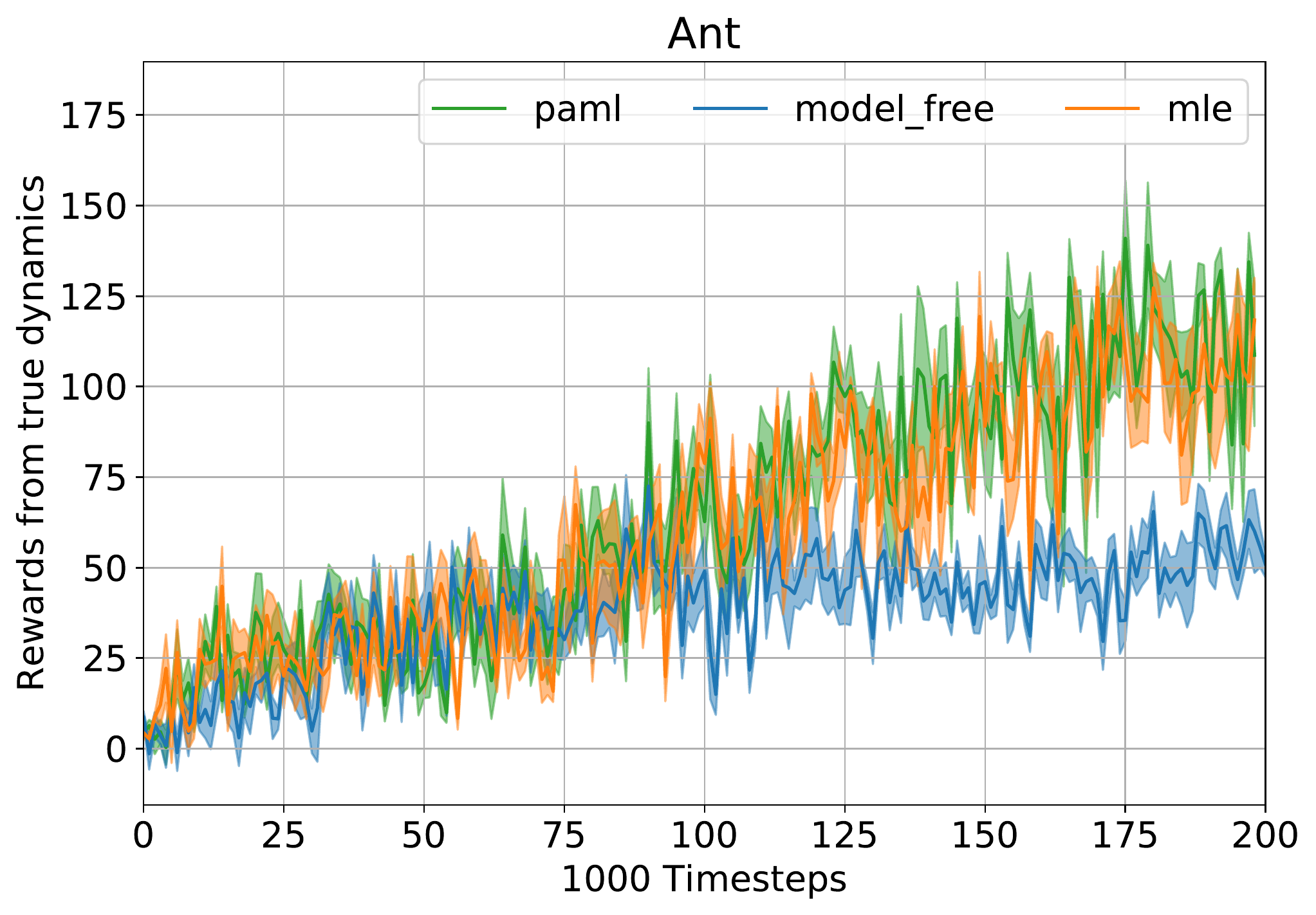}
    \caption{Comparison of policies trained using a model-free method (DDPG, \citealt{lillicrap2015continuous}), or by planning using a model learned by PAML, or MLE. Experimental details in the supplementary. Solid lines indicate mean of 5 runs and shaded regions the standard error.}
    \vspace{-0.1in}
\label{fig:actor-critic-all-high-dim}
\end{figure}
\else
\begin{figure}
\centering
    \includegraphics[width=0.48\linewidth]{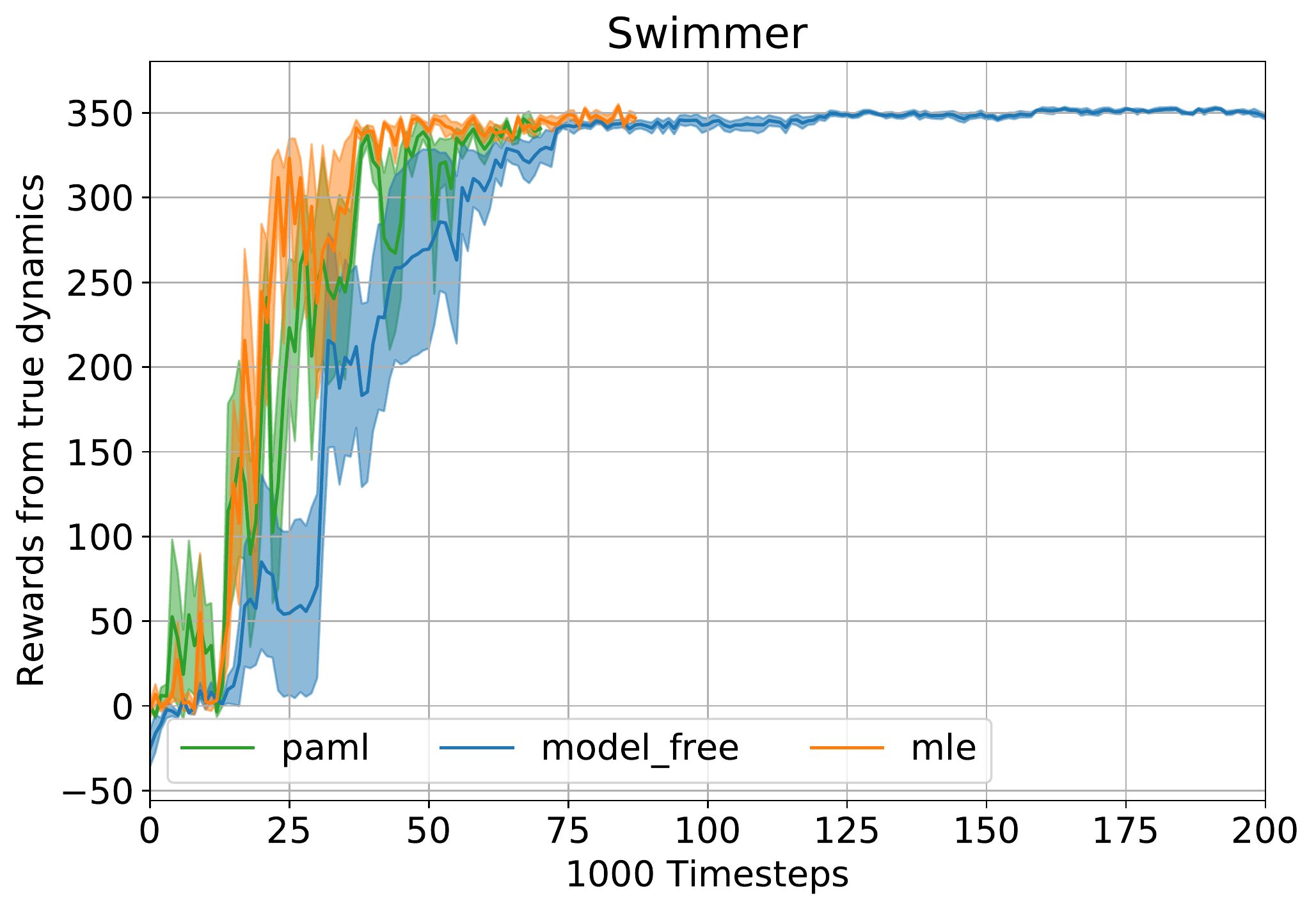}
    \includegraphics[width=0.48\linewidth]{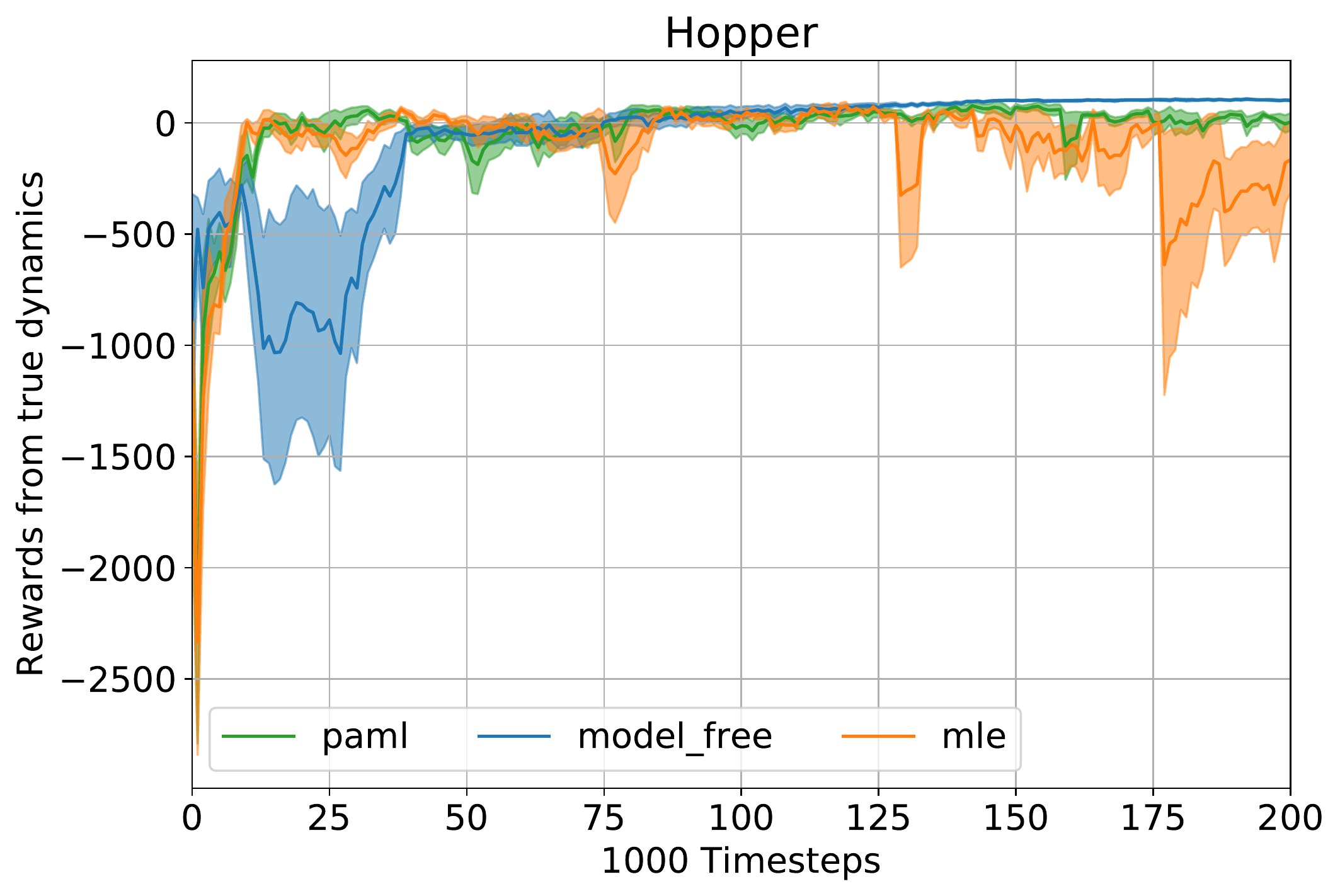}
    \includegraphics[width=0.48\linewidth]{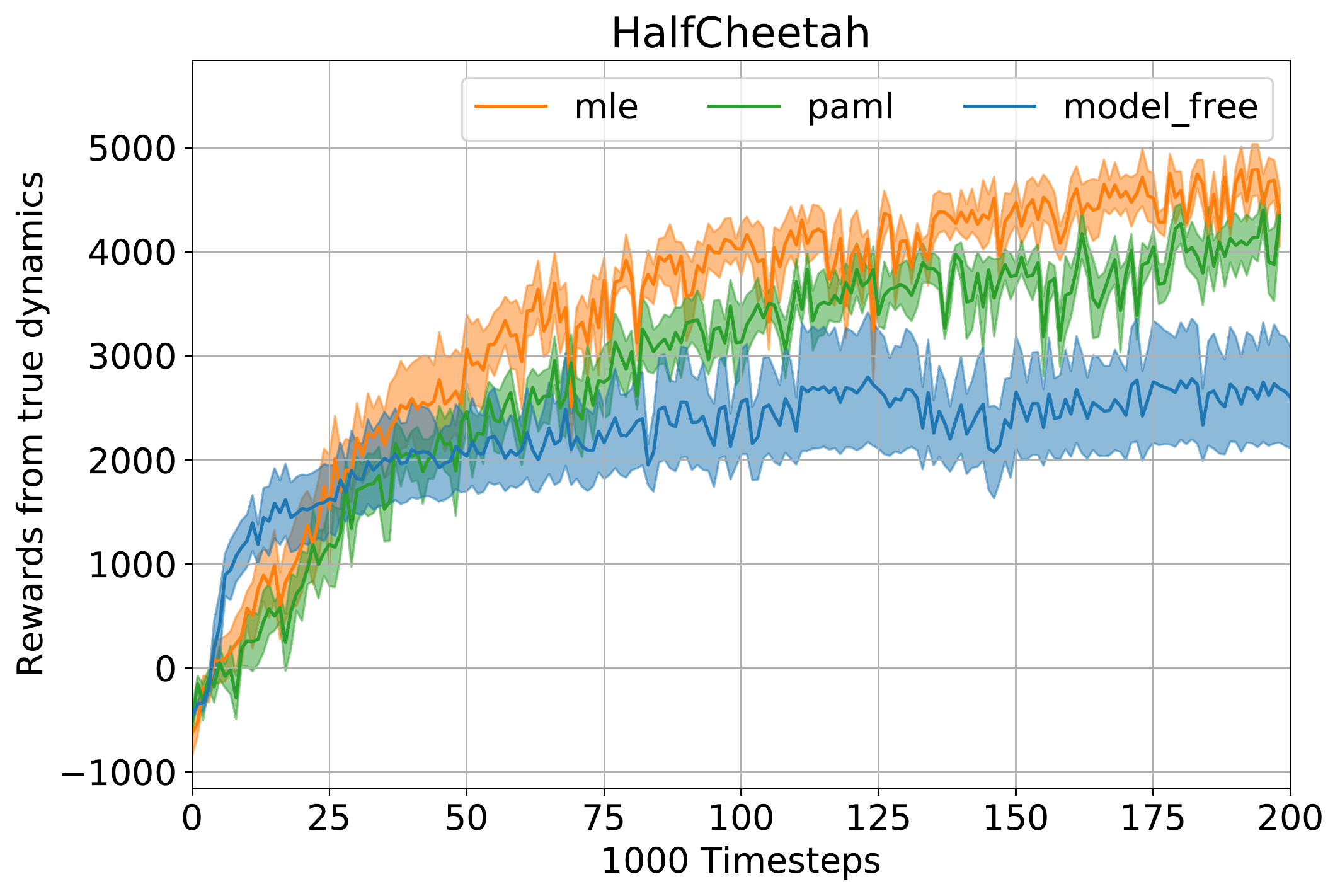}
    \includegraphics[width=0.48\linewidth]{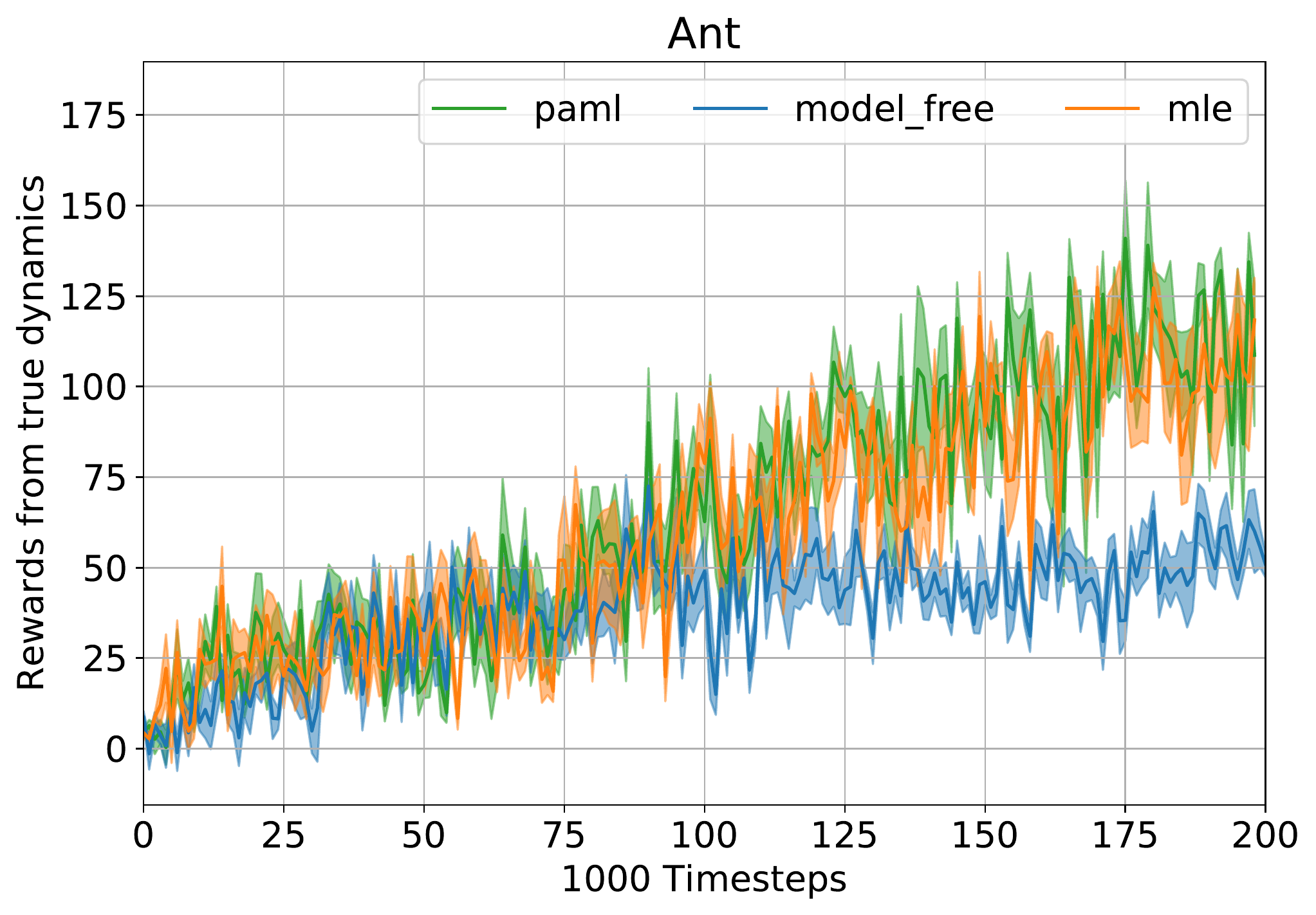}
    \caption{Comparison of policies trained using a model-free method (DDPG, \citealt{lillicrap2015continuous}), or by planning using a model learned by PAML, or MLE. Experimental details in the supplementary. Solid lines indicate mean of 5 runs and shaded regions the standard error.}
    \vspace{-0.1in}
\label{fig:actor-critic-cheetah}
\end{figure}
\fi


\section{Discussion and Future Work}
\label{sec:PAML-Discussion}

We introduced Policy-Aware Model Learning, a decision-aware MBRL framework that incorporates the policy in the way the model is learned. PAML encourages the model to learn about aspects of the environment that are relevant to planning by a PG method, instead of trying to build an accurate predictive model.
We proved a convergence guarantee for a generic model-based PG algorithm, and introduced a new notion of policy approximation error.
We empirically evaluated PAML and compared it with MLE on some benchmark domains. 
A fruitful direction is deriving PAML loss for other PG methods, especially the state of the art ones.

\ifSupp
\appendix

\section{Theoretical Background and Proofs}
\label{sec:PAML-Theory-Appendix}

\todo{Some discussions here? Showing a road map of what they should expect. A few lines suffice. -AMF}
\subsection{Background Results}
\label{sec:PAML-Theory-Background}

We report some background results in this section. These results are quoted from elsewhere, with possibly minor modification, as shall be discussed.

\begin{lemma}[Lemma E.5 of~\citealt{AgarwalKakadeLeeMahajar2019}]
\label{lem:PAML-PolicySmoothness}
Suppose that Assumption~\ref{asm:PolicySmoothness} holds, the action space is finite with $\actionnum$ elements, and
 the action-value functions are all $\Qmax = \frac{\Rmax}{1 - \gamma}$-bounded.
For any $x \in \XX$, we then have
\begin{align*}
	\norm{ \nabla_\theta V^{\pi_{\theta_1}}(x) - \nabla_\theta V^{\pi_{\theta_2}}(x) }_2
	\leq
	\beta \norm{\theta_1 - \theta_2}_2,
\end{align*}
with
\begin{align*}
	\beta =
	\Qmax
	\left[
		\frac{2 \gamma \beta_1^2 \actionnum^2 }{(1-\gamma)^2} + 
		\frac{\beta_2 \actionnum }{1 - \gamma}
	\right].
\end{align*}
\end{lemma}
The difference of this result with the original Lemma E.5 of~\citet{AgarwalKakadeLeeMahajar2019} is that here we assume that the rewards are $\Rmax$-bounded, whereas their paper is based on the assumption that the reward is between $0$ and $1$. As such, their $\Qmax$ is $\frac{1}{1 - \gamma}$, and their $\beta$ is
$
		\frac{2 \gamma \beta_1^2 \actionnum^2 }{(1-\gamma)^3} + 
		\frac{\beta_2 \actionnum }{(1 - \gamma)^2}
$.

The change in the proof of Lemma E.5 stems from the change in Lemma E.2 of~\citet{AgarwalKakadeLeeMahajar2019}.
The upper bound in Lemma E.2 changes from
\[
	\max_{ \norm{u}_2 = 1, \theta + \alpha u \in \Theta}
	\left |
		\frac{\mathrm{d}^2 \tilde{V}(\alpha)}{\mathrm{d}\alpha^2} \Big \vert_{\alpha = 0}
	\right|
	\leq
	\frac{2\gamma C_1^2}{(1 - \gamma)^3} +
	\frac{C_2}{(1 - \gamma)^2}
\]
to
$
	\Qmax
	\left[
		\frac{2\gamma C_1^2}{(1 - \gamma)^2} +
		\frac{C_2}{1 - \gamma}
	\right]
$.

\begin{proposition}[Proposition D.1 of~\citet{AgarwalKakadeLeeMahajar2019}]
\label{prop:PAML-SmallGradientMapping-to-epsStationarity}
Let $\hat{J}_\mu(\pi_\theta) = J_\mu(\pi_\theta, \PKernelhat^{\pi}, \Qhat^{\pi_\theta} )$ be $\beta$-smooth in $\theta$. Define the gradient mapping as
\[
	G^\eta(\theta) = \frac{1}{\eta} \left( \ProjTheta \left [ \theta + \eta \nabla_\theta \hat{J}_\mu \right ] - \theta \right).
\]
Let $\theta' = \theta + \eta G^\eta$ for some $\eta > 0$. We have
\begin{align*}
	\max_{\theta + \delta \in \Theta, \norm{\delta}_2 \leq 1}
	\delta^\top \nabla_\theta \hat{J}_\mu(\pi_{\theta'})
	\leq
	(\eta \beta + 1) \norm{G^\eta(\theta)}_2.
\end{align*}
\end{proposition}

\todo{Add Theorem 10.15 of~\citet{Beck2017}.}

The following lemma is a multivariate form of the mean value theorem. It is not a new result, but for the sake of completeness, we report it here.\footnote{One can find its proof on Wikipedia page on the \href{https://en.wikipedia.org/w/index.php?title=Mean_value_theorem&oldid=916884313}{Mean value theorem}.}
This statement and proof is quoted from the extended version of~\citet{HuangFarahmandKitaniBagnell2015}.

\begin{lemma}
\label{lem:PAML-MultivariateMeanValueTheorem}
Let $f: \Real^m \ra \Real^m$ be a continuously differentiable function and $J: \Real^m \ra \Real^{m \times m}$ be its Jacobian matrix, that is $J_{ij} = \frac{\partial f_i(x)}{\partial x_j}$.
We then have for any $x, \Delta x \in \Real^m$, 
\begin{align*}
	\norm{f(x + \Delta x) - f(x)}_2 \leq
	\sup_{x'} \norm{J(x')}_2 \norm{\Delta x}_2, \\
	\norm{f(x + \Delta x) - f(x)}_1 \leq
	\sup_{x'} \norm{J(x')}_1 \norm{\Delta x}_1.
\end{align*}
\end{lemma}
An $l_1$ and $l_2$ matrix norms in this lemma are vector-induced norms on $\Real^m$, and have the property that for an $m \times m$ matrix $A$, 
$\norm{A}_2 = \sigma_\text{max}(A)$ and
$\norm{A}_1 = \max_j \sum_{i} |A_{ij}|$.

\begin{proof}
Consider a continuously differentiable function $g: \Real \ra \Real$. By the fundamental theorem of calculus,
$
	g(1) - g(0) =
	\int_{0}^{1} g'(t) \mathrm{d}t
$.
%
%
For each component $f_i$ of $f$, define $g_i(u) = f_i(x + u \Delta x)$, so
$f_i(x + \Delta x) - f_i(x) = g_i(1) - g_i(0) = \int_{0}^1 g_i'(t) \mathrm{d}t
=
\int_{0}^1 
	\left[
	\sum_{j=1}^d \frac{ \partial f_i}{\partial x_j}(x + t \Delta x) . \Delta x_j
	\right] \mathrm{d} t.
$
For the vector-valued function $f$, we get
$
f(x + \Delta x) - f(x) = \int_{0}^1 J(x+t \Delta x) \Delta x \, \mathrm{d} t
$,
therefore,
\begin{align*}
	\norm{f(x + \Delta x) - f(x)}_2
	&
	=
	\norm{ \int_{0}^1 J(x+t \Delta x) \Delta x \mathrm{d} t }_2
	\leq
	\int_{0}^1 \norm{J(x+t \Delta x)}_2 \norm{\Delta x}_2 \mathrm{d} t
	\\
	&
	\leq
	\sup_{x'} \norm{J(x')}_2 \norm{\Delta x}_2 \int_{0}^1 \mathrm{d} t.
\end{align*}

The $l_1$-norm result is obtained using the $l_1$-norm instead of the $l_2$-norm in the last step.
\end{proof}

The following is a restatement of Theorem 10.15 part (c) of \citep{Beck2017}, with a slight modification to allow reporting convergence results with an expectation over iterates rather than a $\min$.
\todo{Or maybe make it a proposition? -AMF}

\todo{How is the step size $\eta$ chosen here?! -AMF}

\begin{lemma}[Slight modification of Theorem 10.15 of~\citet{Beck2017} -- Part (c)]
\label{lem:PAML-Beck-Expectation}
Let $F(x)$ be a $\beta$-smooth function for all $x \in \Real^d$, and $G^\eta$ the projected gradient mapping with step-size $\eta \in (\frac{\beta}{2}, \infty)$.
Let $\{x^n\}_{n \geq 0}$ be the sequence generated by the projected gradient method for minimizing $F(x)$, whose optimum point is at $F_\text{opt}$. We then have
\begin{equation*}
    \EEX{n \sim \text{Unif}(0, 1, \ldots, k) }{\smallnorm{G^\eta(x^n)}_2} \leq \sqrt{\frac{F(x^0) - F_{\text{opt}}}{M(k+1)}}.
\end{equation*}
with
\[
	M = \frac{\eta - \frac{\beta}{2}}{\eta^2}.
\]
\end{lemma}
\begin{proof}
By Lemma 10.14 of \citep{Beck2017} we have that 
\begin{equation}
\label{eq:background-inequality}
	F(x^n) - F(x^{n+1}) \geq M\left\Vert G^\eta(x^n) \right\Vert^2. 
\end{equation}
We sum \eqref{eq:background-inequality} over $n = 0, 1, \dotsc, k$ to obtain
\begin{equation*}
    F(x^0) - F(x^{k+1}) \geq M \sum_{n=0}^k \norm{G^\eta(x^n)}^2 \geq M(k+1) \EEX{n \sim \text{Unif}(0, 1, \ldots, k) }{\norm{G_d(x^n)}^2_2}
\end{equation*}
By Jensen's inequality, we have that
\[
	\EEX{n \sim \text{Unif}(0, 1, \dotsc, k) }{\smallnorm{G^\eta(x^n)}_2^2} \geq (\EEX{n \sim \text{Unif}(0, 1, \ldots, k) }{\norm{G^\eta(x^n)}}_2)^2.
\]
Using this and the fact that $F(x^{k+1}) \geq F_{\text{opt}},$ we obtain the desired result.
\end{proof}

\subsection{Auxiliary Results}
\label{sec:PAML-Theory-Auxiliary}

\todo{A description of what result we have should be written here. -AMF}

\begin{lemma}[Exponential Policy -- Boundedness]
\label{lem:PAML-ExponentialFamilyPolicy-Boundedness}
Consider the policy parametrization~\eqref{eq:PAML-PolicyParametrization}. 
%
For $p \in \{2, \infty\}$, let $B_p = \sup_{(x,a) \in \XA} \norm{\phi(a|x)}_p$, and assume that $B_p < \infty$.
It holds that for any $x \in \XX$,
\begin{align*}
	\int \norm{\nabla_\pi \pi_\theta(a|x)}_2 \da =
	\EEX{A \sim \pi_\theta(\cdot|x) }
	{\norm{\nabla_\theta \log \pi_\theta(A|x)}_p} \leq
	\begin{cases}
		B_2,	& p = 2 \\
		2 B_\infty. & p = \infty
	\end{cases}
\end{align*}
\end{lemma}
\begin{proof}
For the policy parameterization~\eqref{eq:PAML-PolicyParametrization}, one can show that
\begin{align*}
	\nabla_\theta \log \pi_\theta(a|x) = \phi(a|x) - \EEX{A \sim \pi_\theta(\cdot|x) }{\phi(A|x)}
	= \phi(a|x) - \bar{\phi}(x),
\end{align*}
with $\bar{\phi}(x) = \bar{\phi}_{\pi_\theta}(x) = \EEX{A \sim \pi_\theta(\cdot|x) }{\phi(A|x)}$ being the expected value of the feature $\phi(A|x)$ w.r.t. $\pi_\theta(\cdot|x)$.

Let us focus on $p = 2$.
\ifconsiderlater
\todo{A bit shorter, but slightly more complicated proof.}
Note that we have
$\EEX{A \sim \pi_\theta(\cdot|x)} { \norm{ \phi(A|x) - \bar{\phi}(x)  }_2 }
	\leq
	\sqrt{ \EEX{A \sim \pi_\theta(\cdot|x)} { \smallnorm{\phi(A|x) - \bar{\phi}(x) }_2^2 } }$.
We study the expected squared norm (and omit the explicit dependence on $\pi_\theta$ in the expectation and covariance symbols):
\begin{align*}
	\EE { \norm{\phi(A|x) - \bar{\phi}(x) }_2^2 }
	& =
	\EE{ (\phi(A|x) - \bar{\phi}(x) ) (\phi(A|x) - \bar{\phi}(x) )^\top }
	\\
	& =
	\Tr \left[ \Cov{\phi(A|x)} \right]
	\\	
	& =
	\Tr \left[ \EE{\phi(A|x) \phi^\top(A|x)}
	-
	\EE{\phi(A|x)} \EE{\phi(A|x)}^\top \right]
	\\
	& \leq
	\Tr \left[ \EE{\phi(A|x) \phi^\top(A|x)} \right]
	\\
	& =
	\EE{ \Tr\left[ \phi(A|x) \phi^\top(A|x) \right] }
	= \EE{  \norm{\phi(A|x)}_2^2 }
	\leq B_2^2,
\end{align*}
where the first inequality is because $\EE{\phi(A|x)} \EE{\phi(A|x)}^\top$ is positive semi-definite, and its trace is non-negative. 
\fi
\todo{This might be a more elementary proof. Which one to keep? (Months later --- 2020 June) I am inclined to keep the first one, as it is shorter and more elegant. -AMF}
We have
\begin{align*}
	\EEX{A \sim \pi_\theta(\cdot|x)} { \norm{ \phi(A|x) - \bar{\phi}(x)  }_2 }
	& \leq
	\sqrt{ \EEX{A \sim \pi_\theta(\cdot|x)} { \norm{\phi(A|x) - \bar{\phi}(x) }_2^2 } }
	\\
	& =
	\sqrt{
	\EEX{A \sim \pi_\theta(\cdot|x)} {\sum_{i=1}^d \left| \phi_i(A|x) - \bar{\phi}_i(x) \right|^2}
	} 
	\\
	& =
	\sqrt{
	\sum_{i=1}^d \VarX{A \sim \pi_\theta(\cdot|x)}{\phi_i(A|x)}
	}
	\\
	& \leq
	\sqrt{
	\sum_{i=1}^d \EEX{A \sim \pi_\theta(\cdot|x)}{\phi_i^2(A|x)}
	}
	\\
	& =
	\sqrt{
	\EEX{A \sim \pi_\theta(\cdot|x)}{\sum_{i=1}^d \phi_i^2(A|x)}
	}
	=
	\sqrt{
	\EEX{A \sim \pi_\theta(\cdot|x)}{ \norm{\phi(A|x)}_2^2}
	}
	\leq B_2.
\end{align*}
That is,
\begin{align}
\label{eq:PAML-Exponential-Policy-Boundedness-Proof-Norm-of-pi-Bound-l2}
	\EEX{A \sim \pi_\theta(\cdot|x)} {\norm{\nabla_\theta \log \pi_\theta(a|x)}_2 }
	\leq B_2.
\end{align}

For $p = \infty$, we have
\begin{align*}
	\EEX{A \sim \pi_\theta(\cdot|x)} { \norm{\phi(A|x) - \bar{\phi}(x) }_\infty }
	\leq
	\EEX{A \sim \pi_\theta(\cdot|x)} { \norm{\phi(A|x)}_\infty + 	\norm{\bar{\phi}(x)}_\infty }
	\leq
	B_\infty + \norm{\bar{\phi}(x)}_\infty.
\end{align*}
We also have
$\norm{\bar{\phi}(x)}_\infty = 
\norm{\EEX{A \sim \pi_\theta(\cdot|x)}{\phi(A|x)}}_\infty
\leq
\EEX{A \sim \pi_\theta(\cdot|x)}{\norm{\phi(A|x)}_\infty } \leq B_\infty
$.
Therefore, 
%
\begin{align}
\label{eq:lem:PAML-Exponential-Policy-Boundedness-Proof-Norm-of-pi-Bound-linf}
	\EEX{A \sim \pi_\theta(\cdot|x)} 
	{\norm{\nabla_\theta \log \pi_\theta(a|x)}_\infty}
	\leq 2 B_\infty.
\end{align}
\end{proof}

\begin{lemma}[Exponential Policy -- Smoothness]
\label{lem:PAML-ExponentialFamilyPolicy-Smoothness}
Consider a policy $\pi_\theta$ with the policy parameterization~\eqref{eq:PAML-PolicyParametrization} and a discrete action space $\Actions$. Assume that
$B_2 = \sup_{(x,a) \in \XA} \norm{\phi(a|x)}_2 < \infty$.
For any $\theta \in \Theta$ and for any $(x,a) \in \XA$, we have that
\begin{align*}
	&
	\norm{\nabla_\theta \pi_\theta(a|x) }_2 \leq 2 B_2,
	\\
	&
	\norm{\frac{\partial^2 \pi_\theta(a|x)}{\partial \theta^2}}_2 \leq 6 B_2^2,
\end{align*}
in which the matrix norm is the $\ell_2$-induced norm.
For any $\theta_1, \theta_2 \in \Theta$ and for any $(x,a) \in \XA$, we also have
\begin{align*}
	&
	\left| \pi_{\theta_2}(a|x) - \pi_{\theta_1}(a|x) \right| \leq 2B_2 \norm{\theta_2 - \theta_1}_2,
	\\
	&
	\norm{ \nabla_\theta \pi_{\theta_1}(a|x) - \nabla_\theta \pi_{\theta_1}(a|x)  }_2
	\leq
	6B_2^2 \norm{\theta_2 - \theta_1}_2.
\end{align*}
\todo{Verify the constants after finish typing the proofs.}
\end{lemma}
\begin{proof}
We use Taylor series expansion of $\pi_\theta$ and the mean value theorem in order to find the Lipschitz and smoothness constants. We start by computing the gradient and the Hessian of the policy. We can fix $(x,a) \in \XA$ in the rest.
For the policy
\[
	\pi_\theta =
	\pi_\theta(a | x) = 
		\frac	{\exp \left( \phi^\top(a|x) \theta \right)}
			{\sum_{a' \in \Actions} \exp \left( \phi^\top(a'|x) \theta \right)},
\]
%
the gradient is
\begin{align}
\nonumber
\label{eq:PAML-Smoothness-Of-ExponentialFamilyPolicy-Proof-Gradient}
	\frac{\partial \pi_\theta}{\partial \theta} & = 
	\pi_\theta(a|x) \left( \phi(a|x) - \EEX{\pi_\theta(\cdot|x)}{\phi(A|x)} \right)\\
	& =
	\pi_\theta(a|x) \left( \phi(a|x) - \bar{\phi}_\theta \right)
	= 
	\pi_\theta(a|x) \Delta \phi_\theta,
\end{align}
where we use $\bar{\phi}_\theta = \bar{\phi}_\theta(x) = \EEX{\pi_\theta(\cdot|x)}{\phi(A|x)}$ and $\Delta \phi_\theta = \phi(a|x) - \bar{\phi}_\theta$ as more compact notations.

Likewise, the Hessian $H_\theta$ of $\pi_\theta$ is
\begin{align*}
	H_\theta = \frac{\partial^2 \pi_\theta(a|x)}{\partial \theta^2} =
	\frac{\partial \pi_\theta}{\partial \theta} \Delta \phi_\theta^\top +
	\pi_\theta \frac{\partial \Delta \phi_\theta^\top}{\partial \theta}
	=
	\pi_\theta \left[
				\Delta \phi_\theta \Delta \phi_\theta^\top +
				\frac{\partial \Delta \phi_\theta^\top}{\partial \theta}	
			\right].
\end{align*}

We compute $\frac{\partial \Delta \phi_\theta^\top}{\partial \theta}$ as follows:
\begin{align*}
	\frac{\partial \Delta \phi_\theta^\top}{\partial \theta} =
	-\frac{\partial  \EEX{\pi_\theta(\cdot|x)}{\phi(A|x)^\top} }{\partial \theta}
	& =
	- \int 	\frac{\partial \pi_\theta}{\partial \theta} \phi(a)^\top \da
	=
	- \int  \pi_\theta(a|x) \left( \phi(a|x) - \bar{\phi}_\theta \right) \phi(a)^\top \da	
	\\
	& =
	\bar{\phi}_\theta(x) \bar{\phi}_\theta^\top(x) - 
	\EE{\phi(A|x) \phi(A|x)^\top}
	= -\CovX{\pi_\theta(\cdot|x)}{\phi(A|x)}.
\end{align*}

Therefore,
\begin{align}
\nonumber
\label{eq:PAML-Smoothness-Of-ExponentialFamilyPolicy-Proof-Hessian}
	H_\theta & =
	\pi_\theta \left[
				\Delta \phi_\theta \Delta \phi_\theta^\top +
				\bar{\phi}_\theta \bar{\phi}_\theta^\top - 
				\EEX{\pi_\theta}{\phi(A|x) \phi(A|x)^\top }
			\right]
	\\
	& =
	\pi_\theta \left[
				\Delta \phi_\theta \Delta \phi_\theta^\top -
				\CovX{\pi_\theta(\cdot|x)}{\phi(A|x)}
			\right].
\end{align}

Consider two points $\theta_1, \theta_2 \in \Theta$. By the mean value theorem, there exists a $\theta'$ on the line segment connecting $\theta_1$ and $\theta_2$ (that is, $\theta' = (1 - \lambda) \theta_1 + \lambda \theta_2$ with $\lambda \in (0,1)$) such that
\[
	\pi_{\theta_2}(a|x) - \pi_{\theta_1}(a|x)
	=
	\nabla_\theta \pi_{\theta'}^\top(a|x) \Big \vert_{\theta_1 \leq \theta' \leq \theta_2} (\theta_2 - \theta_1).
\]
Therefore,
\[
	\left| \pi_{\theta_2}(a|x) - \pi_{\theta_1}(a|x) \right|
	\leq
	\sup_{\theta'} \norm{ 	\nabla_\theta \pi_{\theta'}(a|x) }_2
	\norm{\theta_2 - \theta_1}_2.
\]
By~\eqref{eq:PAML-Smoothness-Of-ExponentialFamilyPolicy-Proof-Gradient}, we have that for any $\theta'$,
\todo{This is only true when $\pi_{\theta'}(a|x) \leq 1$, which is the case when this is the probability of an action, and not the density of an action. So we have to be careful when we move to continuous action spaces. -AMF}
\begin{align}
\label{eq:PAML-Smoothness-Of-ExponentialFamilyPolicy-GradBound}
\nonumber
	\norm{ 	\nabla_\theta \pi_{\theta'}(a|x) }_2 
	& \leq
	\pi_{\theta'}(a|x) \norm{\phi(a|x) - \bar{\phi}_{\theta'}(x)}_2 \leq 
	\pi_{\theta'}(a|x) \left ( \norm{\phi(a|x)}_2 + \EE{ \norm{\phi_{\theta'}(A|x)}_2} \right)
	\\ &
	\leq
	2 B \pi_{\theta'}(a|x) \leq 2B.
\end{align}
Here we used that $\pi_{\theta'}(a|x)$ is a probability of an action, so its value is not larger than $1$.
This shows that 
\[
	\left| \pi_{\theta_2}(a|x) - \pi_{\theta_1}(a|x) \right| \leq 2B \norm{\theta_2 - \theta_1}_2.
\]

Lemma~\ref{lem:PAML-MultivariateMeanValueTheorem} in Appendix~\ref{sec:PAML-Theory-Auxiliary}, which can be thought of as the vector-valued version of the mean value theorem (though it is only an inequality), shows that for any $\theta_1, \theta_2$, we have
\begin{align*}
	\norm{ \nabla_\theta \pi_{\theta_1}(a|x) - \nabla_\theta \pi_{\theta_1}(a|x)  }_2
	\leq
	\sup_{\theta'} \norm{H_{\theta'}}_2 \norm{\theta_1 - \theta_2}_2,
\end{align*}
where $\norm{H_{\theta'}}_2$ is the $\ell_2$-induced matrix norm.
From~\eqref{eq:PAML-Smoothness-Of-ExponentialFamilyPolicy-Proof-Hessian}, we have that for any $\theta$, including $\theta'$,
\begin{align}
\label{eq:PAML-Smoothness-Of-ExponentialFamilyPolicy-Proof-Hessian-Bound}
\nonumber
	\norm{H_\theta}_2
	& \leq
	\pi_\theta \left[
				\norm{ \Delta \phi_\theta \Delta \phi_\theta^\top}_2 +
				\norm{ \bar{\phi}_\theta \bar{\phi}_\theta^\top}_2 + 
				\norm{ \EEX{\pi_\theta}{\phi(A|x) \phi(A|x)^\top } }_2
			\right]
	\\
	&		
	\leq
	\pi_\theta(a|x) \left[
				\norm{ \Delta \phi_\theta}_2^2 +
				\norm{ \bar{\phi}_\theta}_2^2 + 
				\EEX{\pi_\theta}{\norm{ \phi(A|x)}_2^2 }
			\right]
	\leq 6 B^2,
\end{align}
where we used the fact that for a vector $u \in \Real^d$, the $\ell_2$-induced matrix norm of $u u^\top$ is $\norm{u}_2^2$, in addition to the the convexity of norm along with the Jensen inequality.
This shows that 
\begin{align*}
	\norm{ \nabla_\theta \pi_{\theta_1}(a|x) - \nabla_\theta \pi_{\theta_1}(a|x)  }_2
	\leq
	6B^2 \norm{\theta_2 - \theta_1}_2.
\end{align*}
\end{proof}

This result shows that the exponential policy class with an $\ell_2$-bounded features satisfies Assumption~\ref{asm:PolicySmoothness} (Section~\ref{sec:PAML-Theory-PG-Convergence}) with $\beta_1 = 2B_2$ and $\beta_2 = 6B_2^2$.

We remark that the only step of this proof that we used the discreteness of the action space is~\eqref{eq:PAML-Smoothness-Of-ExponentialFamilyPolicy-GradBound}. Other steps would be valid without such a requirement. 
When we have a continuous action space, $\pi_\theta(a|x)$ is not a probability of an action, but is its density. The density is not bounded by $1$. To extend this result for such a space, we need to upper bound the density. This extension is a topic of future work. \todo{Do we really want to do it? What else needs to be changed? -AMF}


\todo{Discussion! What is this result. -AMF}

\begin{proposition}
\label{prop:PAML-Smoothness-of-Performance-with-Inexact-Critic}
Consider any distribution $\mu \in \bar{\MM}(\XX)$ and the space of exponential policies~\eqref{eq:PAML-PolicyParametrization} parameterized by $\theta \in \Theta$ with a finite action space $\actionnum < \infty$.
Assume that $B = \sup_{(x,a) \in \XA} \norm{\phi(a|x)}_2 < \infty$.
Suppose that the inexact critic $\Qhat^{\pi_\theta}$ satisfies Assumption~\ref{asm:Critic-Smoothness} for any $\theta \in \Theta$, and is $\Qmax$-bounded.
Furthermore, assume that the discount factor $0 \leq \gamma < 1$.
The performance~\eqref{eq:PAML-GradientOfPerformance-InexactCritic} is $\beta$-smooth, i.e., for any $\theta_1, \theta_2 \in \Theta$, it satisfies
\begin{align}
	\norm{ 
		\nabla_\theta J_\mu(\pi_{\theta_1}, \PKernelhat^{\pi_1}, \Qhat^{\pi_1}) -
		\nabla_\theta J_\mu(\pi_{\theta_2}, \PKernelhat^{\pi_2}, \Qhat^{\pi_2})
		}_2
		\leq
		\beta \norm{\theta_1 - \theta_2}_2,
\end{align}
with
\begin{align*}
	\beta = 
	\frac{B}{1 - \gamma}
		\left[
			\sqrt{2 \actionnum} L + \frac{\gamma B \Qmax}{1 - \gamma}
		\right].
\end{align*}
\end{proposition}
\begin{proof}
Let $f(x;\theta) = \sum_{a \in \Actions} \nabla_\theta \pi_\theta(a|x) \Qhat^{\pi_\theta}(x,a) =
\sum_{a \in \Actions} \pi_\theta(a|x) \nabla_\theta \log \pi_\theta(a|x) \Qhat^{\pi_\theta}(x,a)
$. 
For any policy $\pi_\theta$, we have (cf.~\eqref{eq:PAML-GradientOfPerformance-InexactCritic})
\[
	\nabla_\theta J_\mu(\pi_{\theta}, \PKernelhat^{\pi_\theta}, \Qhat^{\pi_\theta}) = 
	\frac{1}{1 - \gamma}
	\EEX{X \sim \muDiscounted(\cdot;\PKernel^{\pi_\theta})}{f(X;\theta)}.
\]

We decompose the difference between the $(1-\gamma)$-scaled PGs at $\theta_1, \theta_2$ into two parts as
\begin{align}
\label{eq:PAML-Smoothness-of-Performance-with-Inexact-Critic-Proof-BasicDecomposition}
\nonumber
	&
	(1 - \gamma)
	\left(
	\nabla_\theta J_\mu(\pi_{\theta_1}, \PKernelhat^{\pi_{\theta_1}}, \Qhat^{\pi_{\theta_1}}) -
	\nabla_\theta J_\mu(\pi_{\theta_2}, \PKernelhat^{\pi_{\theta_2}}, \Qhat^{\pi_{\theta_2}})
	\right)
	=
	\\
	&
	\underbrace{
		\EEX{\muDiscounted(\cdot;\PKernelhat^{\pi_{\theta_1}})}{f(X;\theta_1) - f(X;\theta_2)}
	}_{\eqdef \text{(A)} }
	+
	\underbrace{
	\EEX{X \sim \muDiscounted(\cdot;\PKernelhat^{\pi_{\theta_1}})}{f(X;\theta_2)} -
	\EEX{X \sim \muDiscounted(\cdot;\PKernelhat^{\pi_{\theta_2}})}{f(X;\theta_2)}
	}_{\eqdef \text{(B)} } 
\end{align}

We upper bound the $\ell_2$-norm of terms (A) and (B).

For term (A), we first benefit from the convexity of the norm to apply the Jensen's inequality, and then use the Cauchy-Schwarz inequality to get
\begin{align*}
	\norm{ \text{(A)} }_2 & =
	\norm{
			\int \muDiscounted(\dx;\PKernelhat^{\pi_{\theta_1}})
				\sum_{a \in \Actions} \nabla_\theta \pi_{\theta_1} (a|x)
				\left(
					\Qhat^{\pi_{\theta_1}}(x,a) - \Qhat^{\pi_{\theta_2}}(x,a)
				\right)
		}_2
	\\
	& \leq
	\int \sum_{a \in \Actions}  \muDiscounted(\dx;\PKernelhat^{\pi_{\theta_1}}) 
		\norm{\nabla_\theta \pi_{\theta_1} (a|x)}_2
				\left|
					\Qhat^{\pi_{\theta_1}}(x,a) - \Qhat^{\pi_{\theta_2}}(x,a)
				\right|
		\\
	& \leq
	\sqrt{
		\int \sum_{a \in \Actions}  \muDiscounted(\dx;\PKernelhat^{\pi_{\theta_1}})
		\norm{\nabla_\theta \pi_{\theta_1}(a|x)}_2^2
		}
	\; \cdot
	\sqrt{
		\int \sum_{a \in \Actions}  \muDiscounted(\dx;\PKernelhat^{\pi_{\theta_1}})
		\left|
			\Qhat^{\pi_{\theta_1}}(x,a) - \Qhat^{\pi_{\theta_2}}(x,a)
		\right|^2
		}.
\end{align*}

As $\norm{\phi(a|x)}_2$ is bounded by $B$ by assumption, we can evoke Lemmas~\ref{lem:PAML-ExponentialFamilyPolicy-Boundedness} and~\ref{lem:PAML-ExponentialFamilyPolicy-Smoothness} to get that for any $x \in \XX$, 
\[
	\sum_{a \in \Actions} \norm{\nabla_\theta \pi_{\theta_1}(a|x)}_2^2 \leq
	\max_{a' \in \Actions} \norm{\nabla_\theta \pi_{\theta_1}(a'|x)}_2 
	\sum_{a \in \Actions} \norm{\nabla_\theta \pi_{\theta_1}(a|x)}_2
	\leq (2B) B = 2B^2.
\]
So the first term is bounded by $\sqrt{2} B$.
By the $L$-smoothness assumption of the inexact critic, 
\[
	\sqrt{
	\sum_{a \in \Actions}
		\left|
			\Qhat^{\pi_{\theta_1}}(x,a) - \Qhat^{\pi_{\theta_2}}(x,a)
		\right|^2
	}
	=
	\sqrt{\actionnum}
	\norm{\Qhat^{\pi_{\theta_1}}(x,\cdot) - \Qhat^{\pi_{\theta_2}}(x,\cdot)}_2
	\leq L \sqrt{\actionnum} \norm{\theta_1 - \theta_2}_2,
\]
for any $x$. Therefore, the second term in the RHS is bounded by $L \sqrt{ \actionnum } \norm{\theta_1 - \theta_2}_2$.
Therefore,
\begin{align}
\label{eq:PAML-Smoothness-of-Performance-with-Inexact-Critic-Proof-TermA-UpperBound}
	\norm{ \text{(A)} }_2 \leq
	\sqrt{2 \actionnum} B L \norm{\theta_1 - \theta_2}.
\end{align}

Let us turn to term (B). Lemma~\ref{lem:PAML-ExpectedDiscountedError} shows that
\begin{align}
\label{eq:PAML-Smoothness-of-Performance-with-Inexact-Critic-Proof-TermB-First-Step}
	\norm{(B)}_2 \leq
	\frac{\gamma}{1 - \gamma}
	\norm{f(\cdot;\theta_2)}_{2, \infty}
	\norm{\Delta \PKernelhat^{\pi_{\theta_1} \ra \pi_{\theta_2} } }_{1,\infty},
\end{align}
where
\begin{align}
\label{eq:PAML-Smoothness-of-Performance-with-Inexact-Critic-Proof-TermB-ModelError-Bound-Step1}
	\norm{\Delta \PKernelhat^{\pi_{\theta_1} \ra \pi_{\theta_2} } }_{1,\infty} =
	\sup_{x \in \XX}
	\norm{
			\sum_{a \in \Actions} \PKernelhat(\cdot|x,a) \left(\pi_{\theta_1}(a|x) - \pi_{\theta_2}(a|x) \right)
		}_1.
\end{align}

We relate $\pi_{\theta_1}(a|x) - \pi_{\theta_2}(a|x)$ to the difference between $\theta_1$ and $\theta_2$ as follows.
For any $x$, we have that
\begin{align*}
	\pi_{\theta_1}(a|x) - \pi_{\theta_2}(a|x) = 
	\nabla_\theta \pi_{\theta'} (a|x)^\top (\theta_2 - \theta_1)
\end{align*}
for a $\theta'$ on the line segment between $\theta_1$ and $\theta_2$, i.e., $\theta' = \lambda \theta_1 + (1 - \lambda) \theta_2$ for a $\lambda \in (0,1)$.
We use this inequality along with the definition of TV~\eqref{eq:PAML-TV-Distance} and Lemma~\ref{lem:PAML-ExponentialFamilyPolicy-Boundedness} to upper bound the term inside the norm of the RHS of~\eqref{eq:PAML-Smoothness-of-Performance-with-Inexact-Critic-Proof-TermB-ModelError-Bound-Step1} as
\begin{align*}
	\norm{
			\sum_{a \in \Actions} \PKernelhat(\cdot|x,a) \left(\pi_{\theta_1}(a|x) - \pi_{\theta_2}(a|x) \right)
		}_1
	& =
	\sup_{ \norm{g}_\infty \leq 1}
	\left|
	\int \sum_{a \in \Actions}
		\PKernelhat(\dy|x,a) \left(\pi_{\theta_1}(a|x) - \pi_{\theta_2}(a|x) \right) g(y)
	\right|
	\\
	&
	\leq
	\sum_{a \in \Actions} \left|\pi_{\theta_1}(a|x) - \pi_{\theta_2}(a|x) \right|
	\sup_{ \norm{g}_\infty \leq 1} \left| \int \PKernelhat(\dy|x,a) g(y) \right|
	\\
	&
	=
	\sum_{a \in \Actions} \left|\pi_{\theta_1}(a|x) - \pi_{\theta_2}(a|x) \right|
	\\
	&
	=
	\sum_{a \in \Actions} \left| \nabla_\theta \pi_{\theta'} (a|x)^\top (\theta_2 - \theta_1) \right|
	\\
	& \leq
	\norm{\theta_2 - \theta_1}_2
	\sum_{a \in \Actions} \norm{ \nabla_\theta \pi_{\theta'} (a|x) }_2
	\leq B \norm{\theta_2 - \theta_1}_2.
\end{align*}
As this holds for any $x \in \XX$, it shows that~\eqref{eq:PAML-Smoothness-of-Performance-with-Inexact-Critic-Proof-TermB-ModelError-Bound-Step1} can be upper bounded by
\begin{align}
\label{eq:PAML-Smoothness-of-Performance-with-Inexact-Critic-Proof-TermB-ModelError-Bound-Step2}
	\norm{\Delta \PKernelhat^{\pi_{\theta_1} \ra \pi_{\theta_2} } }_{1,\infty}
	\leq
	B \norm{\theta_2 - \theta_1}_2.
\end{align}

To upper bound $\norm{f(\cdot;\theta_2)}_{2, \infty} = \sup_{x \in \XX}  \norm{f(x;\theta_2)}_2$~\eqref{eq:PAML-Norm-f}, we use Lemma~\ref{lem:PAML-ExponentialFamilyPolicy-Boundedness} again to get that
\begin{align}
\nonumber
\label{eq:PAML-Smoothness-of-Performance-with-Inexact-Critic-Proof-TermB-f-bound}
	\norm{f(x;\theta_2)}_2 & =
	\norm{ \sum_{a \in \Actions} \nabla_\theta \pi_\theta(a|x) \Qhat^{\pi_\theta}(x,a) }_2
	\leq
	\sum_{a \in \Actions} \norm{ \nabla_\theta \pi_\theta(a|x) \Qhat^{\pi_\theta}(x,a)}_2
	\\ &
	\leq
	\Qmax \sum_{a \in \Actions} \norm{ \nabla_\theta \pi_\theta(a|x)}_2
	\leq
	\Qmax B.
\end{align}
This holds uniformly over $x$, so it provides an upper bound on $\norm{f(\cdot;\theta_2)}_{2, \infty}$ too.

Plugging~\eqref{eq:PAML-Smoothness-of-Performance-with-Inexact-Critic-Proof-TermB-ModelError-Bound-Step2} and~\eqref{eq:PAML-Smoothness-of-Performance-with-Inexact-Critic-Proof-TermB-f-bound} into~\eqref{eq:PAML-Smoothness-of-Performance-with-Inexact-Critic-Proof-TermB-First-Step} leads to
\begin{align}
\label{eq:PAML-Smoothness-of-Performance-with-Inexact-Critic-Proof-TermB-Final}
	\norm{(B)}_2 \leq
	\frac{\gamma B^2 \Qmax }{1 - \gamma} \norm{\theta_2 - \theta_1}_2.
\end{align}
This latest result along with~\eqref{eq:PAML-Smoothness-of-Performance-with-Inexact-Critic-Proof-TermA-UpperBound} shows that we can upper bound the $\ell_2$-norm of the difference in PG~\eqref{eq:PAML-Smoothness-of-Performance-with-Inexact-Critic-Proof-BasicDecomposition} as
\begin{align*}
	\norm{
		\nabla_\theta J_\mu(\pi_{\theta_1}, \PKernelhat^{\pi_{\theta_1}}, \Qhat^{\pi_{\theta_1}}) -
		\nabla_\theta J_\mu(\pi_{\theta_2}, \PKernelhat^{\pi_{\theta_2}}, \Qhat^{\pi_{\theta_2}})
		}_2
	\leq
	\frac{B}{1 - \gamma}
	\left[
		\sqrt{2 \actionnum} L + \frac{\gamma B \Qmax}{1 - \gamma}
	\right]
	\norm{\theta_1 - \theta_2}_2.
\end{align*}
This concludes the proof.
\end{proof}

\section{Further Detail on VAML and Comparison with PAML}
\label{sec:PAML-VAML}

\todo{Make sure I summarize what is happening here earlier in the paper, so the reader pays attention. -AMF}

This section provides more detail on Value-Aware Model Learning (VAML) and Iterative VAML (IterVAML)~\citep{FarahmandVAMLEWRL2016,FarahmandVAML2017,Farahmand2018}, and complements the discussion in Section~\ref{sec:PAML-Background-VAML}. For more detail and the results on the properties of VAML and IterVAML, refer to the original papers.

Recall that VAML attempts to find $\PKernelhat$ such that applying the Bellman operator $\ToptP{\PKernelhat}$ according to the model $\PKernelhat$ on a value function $Q$ has a similar effect as applying the true Bellman operator $\ToptP{\PKernelTrue}$ on the same function, i.e.,
\[
	\ToptP{\PKernelhat} Q \approx \ToptP{\PKernelTrue} Q.
\]
This ensures that one can replace the true dynamics with the model without (much) affecting the internal mechanism of a Bellman operator-based $\AlgPlan$.
This goal can be realized by defining the loss function as follows: 
Assuming that $V$ (or $Q$) is known,
the pointwise loss between $\PKernelhat$ and $\PKernelTrue$ is
\begin{align}
\label{eq:PAML-VAML-cPhatPx}
	c(\PKernelhat,\PKernelTrue;V)(x,a)  = 
	\left|
		\ip{\PKernelTrue(\cdot|x,a) - \PKernelhat(\cdot|x,a) }{V}
	\right|
	=
	\left|
		\int 
		\left[ \PKernelTrue(\mathrm{d} x' |x,a) - \PKernelhat(\mathrm{d} x'|x,a) \right] V(x')
	\right|,
\end{align}
in which we substituted $\max_a Q(\cdot,a)$ in the definition of the Bellman optimality operator~\eqref{eq:PAML-BellmanOperator} with $V$ to simplify the presentation.

By taking the expectation over state-action space according to the probability distribution $\nu \in \bar{\MM}(\XA)$, which can be the same distribution as the data generating one, VAML defines the expected loss function
\begin{align}
\label{eq:PAMP-VAML-cPhatP}
	c_{2,\nu}^2(\PKernelhat,\PKernelTrue;V) = 
	\int
	\dnu(x,a)
	\left|
		\int 
		\left[ \PKernelTrue(\mathrm{d} x' |x,a) - \PKernelhat(\mathrm{d} x'|x,a) \right] V(x')
	\right|^2.
\end{align}

As the value function $V$ is unknown, we cannot readily minimize this loss function, or its empirical version.
Handling this unknown $V$ differentiates the original formulation introduced by~\citet{FarahmandVAML2017}
with the iterative one~\citep{Farahmand2018}.
Briefly speaking, the original formulation of VAML considers that $\AlgPlan$ represents the value function within a known function space $\FF$, and it then tries to find a model that no matter what value function $V \in \FF$ is selected by the planner, the loss function~\eqref{eq:PAMP-VAML-cPhatP} is still small. This leads to a robust formulation of the loss in the form of
\begin{align}
\label{eq:PAML-VAML-cPhatPsupV}
	c_{2,\nu}^2(\PKernelhat,\PKernelTrue) = 
	\int
	\dnu(x,a)
	\sup_{V \in \FF} 
	\left|
		\int 
		\left[ \PKernelTrue(\mathrm{d} x' |x,a) - \PKernelhat(\mathrm{d} x'|x,a) \right] V(x')
	\right|^2.
\end{align}

Even though taking the supremum over $\FF$ makes this loss function conservative compared to~\eqref{eq:PAMP-VAML-cPhatP}, where the value function $V$ is assumed to be known, it is still a tighter objective to minimize than the $\KL$ divergence.
Consider a fixed $z = (x,a)$, and notice that we have
\begin{align}
\label{eq:PAML-VAML-ErrorUpperBound}
	\sup_{V \in \FF} | \langle \PKernelTrue(\cdot|x,a) - \PKernelhat(\cdot|x,a), V \rangle |
	\leq
	\norm{ \PKernelTrue_z - \PKernelhat_z }_1
	\sup_{V \in \FF} \norm{V}_\infty
	\leq
	\sqrt{2 \KL(\PKernelTrue_z || \PKernelhat_z ) }
	\sup_{V \in \FF} \norm{V}_\infty,
\end{align}
where we used Pinsker's inequality in the second inequality.
MLE is the minimizer of the $\KL$-divergence based on the empirical distribution (i.e., data), so these upper bounds suggest that if we find a good MLE (with a small $\KL$-divergence), we also have a model that has a small total variation error too. This in turn implies the accuracy of the Bellman operator according to the learned model.

Nonetheless, these sequences of upper bounding might be quite loose.
As an extreme, but instructive, example, consider that the value function space consisting of bounded constant functions ($\FF = \cset{x \mapsto c}{|c| < \infty}$). For this function space, $\sup_{V \in \FF} | \langle \PKernelTrue(\cdot|x,a) - \PKernelhat(\cdot|x,a), V \rangle |$ is always zero, irrespective of the the total variation and the $\KL$-divergence of two distributions. MLE does not explicitly benefit from these interaction of the value function and the model.
For more detail and discussion, refer to~\citet{FarahmandVAML2017}.

\todo{Romina: Do we need to talk about iterVAML at all?}
The Iterative VAML (IterVAML) formulation of~\citet{Farahmand2018} exploits some extra knowledge about how $\AlgPlan$ works.
Instead of only assuming that $\AlgPlan$ uses the Bellman optimality operator without assuming any extra knowledge about its inner working (as the original formulation does), IterVAML considers that $\AlgPlan$ is in fact an (Approximate) Value Iteration algorithm.
Recall that (the exact) value iteration (VI) is an iterative procedure that performs
\begin{align*}
	Q_{k+1} \leftarrow \ToptP{\PKernelTrue} Q_k \eqdef r +  \gamma \PKernelTrue V_k,
	\qquad k = 0, 1, \dotsc.
\end{align*}
IterVAML benefits from the fact that if we have a model $\PKernelhat$ such that
\[
	\PKernelhat V_k \approx \PKernelTrue V_k,
\]
the true dynamics $\PKernelTrue$ can be replaced by the learned dynamics $\PKernelhat$ without much affecting the working of VI.
IterVAML learns a new model at each iteration, based on data sampled from $\PKernelTrue$ and the current approximation of the value function $V_k$. The learned model can then be used to perform one iteration of (A)VI.

\todo{These are new paragraphs. Should be revised/verified before publishing. -AMF}

It might be instructive to briefly compare the objective of VAML and PAML.
VAML tries to minimize the error between the Bellman operator w.r.t. the model $\PKernelhat$ and the Bellman operator w.r.t. the true transition model $\PKernelTrue$.
PAML, on the other hand, focuses on minimizing the error between the PG computed according to the model vs. the true transition model.
Furthermore,~\eqref{eq:PAML-VAML-ErrorUpperBound}, Theorem~\ref{thm:PAML-PolicyGradientError}, and~\eqref{eq:PAML-GradientError-KL} show that the TV distance and the $\KL$ divergence provide an upper bound on the loss function of both VAML and PAML.

We may come up with a helpful perspective about these objectives by interpreting them as integral probability metrics (IPM)~\citep{Muller1997}.
Recall that given two probability distributions $\mu_1, \mu_2 \in \bar{\MM}(\XX)$ defined over the set $\XX$ and a space of functions $\GG: \XX \ra \Real$, the IPM distance is defined as
	\begin{align*}
		d_\GG (\mu_1, \mu_2) =
		\sup_{g \in \GG }
			\left|
				\int g(x) \left( \dmu_1(x) - \dmu_2(x) \right)
			\right|.
	\end{align*}
This distance is the maximal difference in expectation of a function $g$ according to $\mu_1$ and $\mu_2$ when the test function $g$ is allowed to be any function in $\GG$.
The TV distance is an IPM with $\GG$ being the set of bounded measurable function, cf.~\eqref{eq:PAML-TV-Distance}.
This set is quite large.
The original formulation of VAML limits the test functions to the space of value functions $\FF$, see~\eqref{eq:PAML-VAML-ErrorUpperBound}.
If we choose $\GG$ to be the space of $1$-Lipschitz functions, we obtain $1$-Wasserstein distance. Therefore, if the space of value functions $\FF$ is the space of Lipschitz functions, VAML minimizes the Wasserstein distance between the true dynamics and the model, as observed by~\citet{AsadiCaterMisraLittman2018}. But the space of value functions often has more structure and regularities than being Lipschitz (e.g., its functions have some kind of higher-order smoothness properties), in which case VAML's loss becomes smaller than Wasserstein distance.
IterVAML further constrains the space of test function by choosing a particular test function $V_k$ at each iteration, that is $\GG = \{ V_k \}$.
The test function for PAML is a single function too, different from IterVAML's, and is defined as
$g(x) = \EEX{A \sim \pi_\theta(\cdot|x)}{\nabla_\theta \log \pi_\theta(A|x) Q^{\pi_\theta}(x,A)}$.
The compared distributions, however, are not $\PKernelTrue$ and $\PKernelhat$, but their discounted future-state distribution $\rhoDiscounted(\cdot; {\PKernelTrue}^{\pi_\theta})$ and $\rhoDiscounted(\cdot; \PKernelhat^{\pi_\theta})$, cf.~\eqref{eq:PAML-J-and-Jhat}.

\todo{Can we compare PAML and VAML's objective for particular choice of policy parameterization and value function space?
For example, if $\nabla \log \pi = \phi$ and we use compatible value functions, we are matching the FIM in the form of $\EE{\phi^\top \phi}$. If we use the same features $\phi$ for VAML, we have to match $\phi$. What's the relation between them?
And then we can also think of exponential family and MaxEnt interpretation. Something to develop/clean up later. -AMF}


\section{Experimental Details}
\label{sec:PAML-Appendix}

\subsection{Environments}
    \begin{enumerate}
        \item \textbf{Finite 2-state MDP.}
        We will use the following convention for this and the 3-state MDP defined next: 
        \begin{align*}
            r(x_i,a_j) &= \boldsymbol{r}[i \times |\AA| + j]\\
            P(x_k|x_i,a_j) &= \boldsymbol{P}[i \times |\AA|+ j][k],
        \end{align*}
        where $\boldsymbol{P}, \boldsymbol{r}$ are given below.
        $$|\AA| = 2, |\XX| = 2, \gamma = 0.9$$
        $$\boldsymbol{P} = [[0.7, 0.3],[0.2, 0.8],[0.99, 0.01],[0.99,0.01]]$$
        $$\boldsymbol{r} = [-0.45, -0.1,0.5,0.5]$$
        \item \textbf{Finite 3-state MDP.}
        Following the convention above:
        $$|\AA| = 2, |\XX| = 3, \gamma = 0.9 $$
        \begin{align*}
        \boldsymbol{P} = &[[[0.6, 0.399999, 0.000001], [0.1, 0.8, 0.1], [0.899999, 0.000001, 0.1]], \\
        &[[0.98, 0.01, 0.01], [0.2, 0.000001, 0.799999], [0.000001, 0.3, 0.699999]]]
        \end{align*}
        $$\boldsymbol{r} = [[0.1, -0.15],
                   [0.1, 0.8],
                   [-0.2, -0.1]
                   ]$$
        \item \textbf{LQR} is defined as follows:
            \begin{equation}
                x' = Ax + Bu; \quad
                r(x,u) = x^Tx + u^Tu; \quad
                u \sim \pi(a|x),
            \end{equation}
            where $x,u \in \mathbb{R}^2$ and $A$ is designed so that the system would be stable over time (i.e. if $u = 0, x \xrightarrow{} 0$). Specifically,
            $$A = 
            \begin{bmatrix}
                0.9 & 0.4\\
                -0.4 & 0.9
            \end{bmatrix}.$$
            The trajectory length is 201 steps (200 actions taken).
        \item \textbf{Pendulum} (OpenAI Gym) state dimensions: 3, action dimensions: 1, trajectory length: 201
        \item \textbf{HalfCheetah} (OpenAI Gym) state dimensions: 17, action dimensions: 6, trajectory length: 1001
        \item \textbf{Ant} (Modified version of OpenAI Gym) state dimensions: 111, action dimensions: 8, trajectory length: 101
        \item \textbf{Swimmer} (Modified version of OpenAI Gym) state dimensions: 9, action dimensions: 2, trajectory length: 1001.
        We used a version of this environment that was modified according to \cite{wang2019benchmarking}, so that it would be solvable with DDPG.
        \item \textbf{Hopper} (Modified version of OpenAI Gym) state dimensions: 11, action dimensions: 3, trajectory length: 101
    \end{enumerate}
\subsection{Engineering details}\label{sec:eng-details}
        For the finite-state MDP experiments, the policy is parameterized by $|\XX| \times |\AA|$ parameters with Softmax applied row-wise (over actions for each state). The model is parameterized by $|\AA| \times |\XX| \times |\XX|$ parameters with Softmax applied to the last dimension (over states for each state, action pair). Calculation of policy gradients is done by solving for the exact value function and taking gradients with respect to the policy parameters using backpropagation. Hyperparameters for performance plots are as follows, model learning rate : 0.001, policy learning rate: 0.1, training iterations for model (per outer loop iteration): 200, training iterations for policy (per outer loop iteration): 1. Since there are no sources of randomness in these experiments (all gradients were calculated exactly), only 1 run is done per experiment. As well, all parameters for the models and policies were initialized at 0. 
        
        For the REINFORCE experiments, the policy is a 2-layer neural network (NN) with hidden size 64 and Rectified Linear Unit (ReLU) activations for the hidden layer. The second layer is separated for predicting the mean and log standard deviation (std) of a Gaussian policy (i.e. the first layer is shared for the mean and std but the second layer has separate weights for each). The output layer activations are Tanh for the mean and Softplus for the log std. The policy is trained with the Adam optimizer \citep{kingma2014adam} with learning rate 0.0001. 
       
        For the actor-critic formulation, the critic network is a 2-layer NN with hidden size 64 and ReLU activations for the hidden layer. The output layer has no nonlinearities. The policy is a 2-layer NN with hidden size 30 and ReLU activations for the hidden layer. The output layer has Tanh activations.
        Both the policy and actor are trained with the Adam optimizer with learning rates 0.0001 and 0.001, respectively. The soft update parameter~\citep{lillicrap2015continuous} for the target networks is 0.001. 

        The model architectures are as follows:
        \begin{itemize}
            \item LQR: Linear connection from input to output with no nonlinearities.
            \item Pendulum: two linear layers with hidden size 2. This is a linear model with a bottleneck as the states of the Pendulum environment are represented with 3 dimensions.
            \item For all other environments: 3-layer NN with ReLU activations. The output layer has no nonlinearities. The hidden sizes are provided in Table \ref{table:hidden-sizes}
            \begin{table}
            \begin{center}
                \begin{tabular}{ | m{5em} | m{5em}|} 
                \hline
                Env & Hidden size\\
                \hline
                Ant & 1024 \\ 
                \hline
                HalfCheetah & 512\\
                \hline
                Swimmer & 128\\ 
                \hline
                Hopper & 512\\ 
                \hline
                \end{tabular}
         \caption{\label{table:hidden-sizes} Hidden sizes for the model NN}
            \end{center}
            \end{table}
        \end{itemize}
    
    For all the approximate experiments, the number of replications, unless stated otherwise in the captions, is 10 and shaded areas indicate standard error. The parameters of the networks in these cases were initialized using Xavier initialization \citep{glorot2010understanding}.
    The models in all experiments are optimized using stochastic gradient descent with no momentum and initial learning rates given in Table \ref{table:eng-details}. The learning rates are reduced by an order of magnitude according to a schedule also given in Table \ref{table:eng-details}. For each iteration of Alg. \ref{alg:PAML}, the model is used to generate virtual samples for planning. Recalling that the model requires a starting state and action in order to predict the next state, the starting states used in our experiments are as follows: a fraction come directly from the replay buffer, of the rest, half are sampled from the starting state distribution of the environment (i.e. $\rho$), and half are uniformly randomly sampled from a neighbourhood of states from the replay buffer. For additional implementation details refer to the code.
    
    The hyperparameters are chosen to have optimal performance for each of the methods separately. We find the optimal hyperparameters for 0 irrelevant dimensions and use those hyperparameters for the rest of the experiments in order to be fair to all methods. 
    Comparing to the performances of other model-based methods in works such as \citep{wang2019benchmarking} and \cite{fujimoto2018addressing}), we have reasonable confidence that our implementations of the MLE and model-free baselines have representative performances for these methods.
    
    Regarding the model learning rate in the stochastic experiments, the search for the PAML and MLE experiments was done over all the orders of magnitude from 1e-8 to 1e-2, with the optimal value being chosen for each method separately. For the policy and critic, we used the learning rates found to be optimal for model-free DDPG, as stated in the original paper [Lillicrap, 2015]. 
    
    For the exact experiments, the search for model learning rates was over 0.1, 0.01, 1e-3, 1e-4 for both MLE and PAML. In the end, it was decided to increase the number of update steps for the model (numbers tried were 1, 50, and 200) and keep the learning rate relatively small to ensure convergence in each iteration (see Table \ref{table:eng-details}).

\begin{table}
\begin{tabularx}{0.9\textwidth}{ 
    | >{\raggedright\arraybackslash}X 
    | >{\centering\arraybackslash}X 
    | >{\centering\arraybackslash}X 
    | >{\raggedleft\arraybackslash}X | }
    \hline
     & LQR (REINFORCE) & Pendulum-v0(DDPG) & HalfCheetah-v0(DDPG) \\
    \hline
    Initial model learning rate & MLE: 1e-5, PAML: 1e-4 & MLE: 1e-4, PAML: 1e-3 & MLE: 1e-5, PAML: 1e-3 \\
    \hline
    LR schedule & [500,1200,1800] (training steps) & \multicolumn{2}{c|}{\parbox{0.5\textwidth}{MLE: [10,1000,20000000], PAML: [10,500,1000] (iterations of Alg. \ref{alg:PAML})}} \\
    \hline
    Number of transitions from real environment per iteration of Alg. \ref{alg:PAML} & MLE: 1000, PAML: 1000 & MLE: 1000, PAML: 1000 & MLE: 200, PAML: 200 \\
    \hline
    Number of virtual samples per iteration of Alg. \ref{alg:PAML} & MLE: 2000 (episodes), PAML: 2000 episodes for irrelevant dimensions, 500 episodes for no irrelevant dimensions & MLE: 500, PAML: 500 & MLE: 20, PAML: 20 \\
    \hline
    Planning horizon & MLE: 200, PAML: 200 & MLE: 10, PAML: 10 & MLE: 10, PAML: 10 \\
    \hline
    Fraction of planning data coming from replay buffer & MLE: 1.0, PAML: 1.0 & MLE: 0.25, PAML: 0.25 & MLE: 0.5, PAML: 0.5 \\
    \hline
\end{tabularx}
\caption{\label{table:eng-details} Experimental details for the results shown in this work.}
\end{table}

\fi 
\clearpage

\bibliography{references,MyBib}
\bibliographystyle{plainnat}

\end{document}